\documentclass[nohyperref]{article}

\usepackage{microtype}
\usepackage{graphicx}
\usepackage{caption}
\usepackage{subcaption}
\usepackage{booktabs} 
\usepackage{bbm}

\usepackage{hyperref}
\usepackage{multirow}

\usepackage[accepted]{icml2023}

\usepackage{amsmath, amsfonts}
\usepackage{amssymb}
\usepackage{mathtools}
\usepackage{amsthm}
\usepackage{bbm}

\usepackage[capitalize,noabbrev]{cleveref}

\theoremstyle{plain}
\newtheorem{theorem}{Theorem}[section]

\newtheorem{lemma}[theorem]{Lemma}
\newtheorem{corollary}[theorem]{Corollary}
\theoremstyle{definition}
\newtheorem{definition}[theorem]{Definition}

\theoremstyle{remark}
\newtheorem{remark}[theorem]{Remark}

\usepackage{amsmath}
\usepackage{amsthm}
\usepackage{xcolor}
\usepackage{algorithm, algorithmic}
\usepackage{makecell}
\usepackage{wrapfig} 
\usepackage{capt-of} 

\newcommand{\ones}{\mathbf 1}

\newcommand{\SF}{$\mathrm{SF}$}

\newcommand{\RFD}{$\mathrm{RFD}$}
\newcommand{\BF}{\textsc{BF}}
\newcommand{\GW}{\textit{GW}}
\newcommand{\FGW}{\textit{FGW}}

\usepackage[textsize=tiny]{todonotes}

\icmltitlerunning{Efficient Graph Field Integrators Meet Point Clouds}

\begin{document}

\twocolumn[
\icmltitle{Efficient Graph Field Integrators Meet Point Clouds}



\icmlsetsymbol{equal}{*}

\begin{icmlauthorlist}
\icmlauthor{Krzysztof Choromanski}{equal,googleresearch,columbia}
\icmlauthor{Arijit Sehanobish}{equal,yale}
\icmlauthor{Han Lin}{equal,columbia}
\icmlauthor{Yunfan Zhao}{equal,columbia}
\icmlauthor{Eli Berger}{xxx} 
\icmlauthor{Tetiana Parshakova}{stanford}
\icmlauthor{Alvin Pan}{columbia}
\icmlauthor{David Watkins}{ai}
\icmlauthor{Tianyi Zhang}{columbia}
\icmlauthor{Valerii Likhosherstov}{cambridge}
\icmlauthor{Somnath Basu Roy Chowdhury}{unc}
\icmlauthor{Avinava Dubey}{googleresearch}
\icmlauthor{Deepali Jain}{googleresearch}
\icmlauthor{Tamas Sarlos}{googleresearch}
\icmlauthor{Snigdha Chaturvedi}{unc}
\icmlauthor{Adrian Weller}{cambridge,alan}
\end{icmlauthorlist}

\icmlaffiliation{columbia}{Columbia University}
\icmlaffiliation{yale}{Independent Researcher}
\icmlaffiliation{cambridge}{University of Cambridge}
\icmlaffiliation{googleresearch}{Google Research}
\icmlaffiliation{alan}{The Alan Turing Institute}
\icmlaffiliation{stanford}{Stanford University}
\icmlaffiliation{unc}{The University of North Carolina at Chapel Hill}
\icmlaffiliation{ai}{The Boston Dynamics AI Institute}
\icmlaffiliation{xxx}{University of Haifa}

\icmlcorrespondingauthor{Krzysztof Choromanski}{kchoro@google.com}

\icmlkeywords{Machine Learning, ICML}

\vskip 0.3in
]



\printAffiliationsAndNotice{\icmlEqualContribution} 

\begin{abstract}
We present two new classes of algorithms for efficient field integration on graphs encoding point clouds. The first class, $\mathrm{SeparatorFactorization}$ (\SF), leverages the bounded genus of point cloud mesh graphs, while the second class, $\mathrm{RFDiffusion}$ (\RFD), uses popular $\epsilon$-nearest-neighbor graph representations for point clouds. 
Both can be viewed as providing the functionality of Fast Multipole Methods (FMMs, \citealp{FMM}), which 
have had a tremendous impact on efficient integration, but for non-Euclidean spaces.
We focus on geometries induced by distributions of walk lengths between points (e.g., shortest-path distance). 
We provide an extensive theoretical analysis of our algorithms, obtaining new results in structural graph theory as a byproduct.
We also perform exhaustive empirical evaluation, including on-surface interpolation for rigid and deformable objects (particularly for mesh-dynamics modeling), Wasserstein distance computations for point clouds, and the Gromov-Wasserstein variant.
\end{abstract}

\section{Introduction \& Related Work}
\label{sec:intro_rel}

Let us consider a weighted undirected graph $\mathrm{G}=(\mathrm{V},\mathrm{E},\mathrm{W})$, where: $\mathrm{V}$ stands for the set of vertices/nodes, $\mathrm{E}$ is the set of edges and $\mathrm{W}: \mathrm{E} \rightarrow \mathbb{R}_{+}$ encodes edge-weights. We assume that a tensor-field $\mathcal{F}:\mathrm{V} \rightarrow \mathbb{R}^{d_{1} \times \ldots \times d_{l}}$ is defined on $\mathrm{V}$, where: $d_{1},\ldots,d_{l}$ stand for tensor dimensions. A kernel (similarity measure) $\mathrm{K}:\mathrm{V} \times \mathrm{V} \rightarrow \mathbb{R}$ on $\mathrm{V}$ is given. 
In this paper, we are interested in efficiently computing the expression $i(v)$ for each $v \in \mathrm{V}$, as defined below:
\begin{equation}
\label{eq:base_integ}
i(v) := \sum_{w \in \mathrm{V}}\mathrm{K}(w,v)\mathcal{F}(w) .
\end{equation}
The expression $i(v)$ can be interpreted as an integration of $\mathcal{F}$ on $\mathrm{G}$ with respect to measure $\mathrm{K}(\cdot, v)$. As such, it can also be thought of as a discrete approximation of the $\mathcal{F}$-field integration in the continuous non-Euclidean space, discretely approximated by $\mathrm{G}$. We refer to the process of computing $i(v)$ for every $v \in \mathrm{V}$ as \textit{graph-field integration} (GFI), see: Fig. \ref{fig:torus}. We write $N=|\mathrm{V}|$ for the size of $\mathrm{V}$. 

\begin{figure}[t!] 
\centering
  \includegraphics[width=0.8\linewidth]{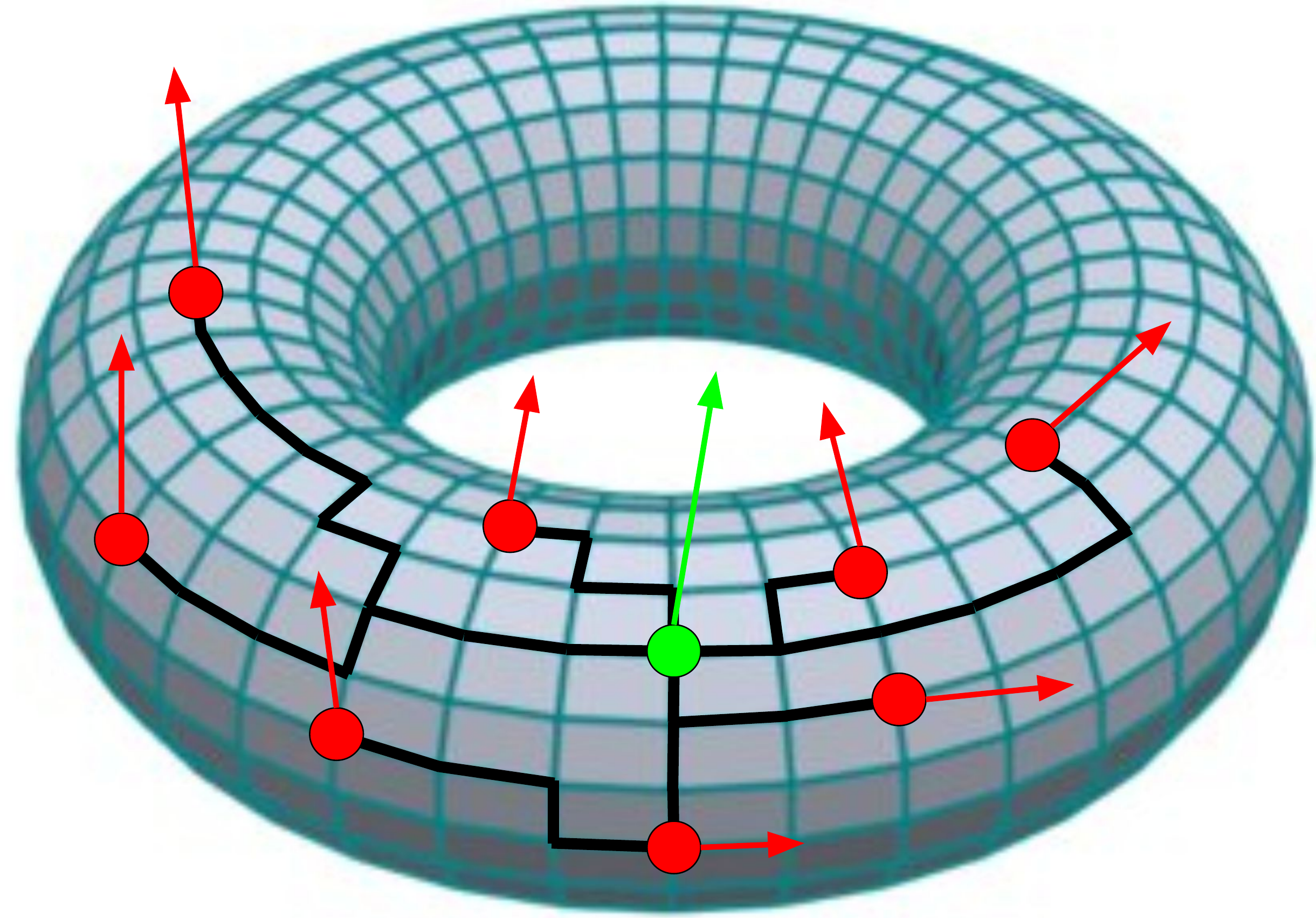}
\vspace{-2mm}  
\caption{\small{Visualization of the problem of integrating vector-field $\mathcal{F}$ on the mesh-graph. Vector $i(v)$ in the green-node $v$ is defined as a weighted sum of the vectors $\mathcal{F}(w)$ in all the nodes (red arrows) with the coefficients $\mathrm{K}(w,v)$ given as $\mathrm{K}(w,v)=f(\mathrm{dist}(w,v))$ for a shortest-path-distance function $\mathrm{dist}$ between nodes of the mesh-graph and some function $f:\mathbb{R} \rightarrow \mathbb{R}$ (with $f(0)=0$). The shortest-path-distance tree from $v$ is marked in black.}}
\label{fig:torus}
\vspace{-5mm}
\end{figure}

A naive, brute-force approach to computing all $i(v)$ is to: \textbf{(1)} calculate a kernel matrix $\mathbf{K}=[\mathrm{K}(w,v)]_{w,v \in \mathrm{V}} \in \mathbb{R}^{N \times N}$ (\textit{\textbf{pre-processing}}), \textbf{(2)} conduct $d_{1} \cdot \ldots \cdot d_{l}$ matrix-vector multiplications: $\mathbf{K}\mathbf{v}^{c_{1},\ldots,c_{l}}$ (\textit{\textbf{inference}}) for $0 \leq c_{i} < d_{i}$, where $\mathbf{v}^{c_{1},\ldots,c_{l}} \overset{\mathrm{def}}{=} \mathcal{F}(\cdot)[c_{1}]\ldots[c_{l}] \in \mathbb{R}^{N}$. Both steps are computationally expensive: inference requires $O(N^{2}d_{1}\cdot\ldots\cdot d_{l})$ time; and pre-processing at least $O(N^{2})$ (in practice, depending on kernel $\mathrm{K}$, often at least $O(N^{3})$). Therefore for large $N$, this approach becomes computationally infeasible.

It is thus natural to ask the following: \textit{Can pre-processing or inference be performed in sub-quadratic time for the number of nodes of a graph?}

It is hopeless to provide a positive answer for an arbitrary graph $\mathrm{G}$ and kernel $\mathrm{K}$; however, methods for certain subclasses have been extensively studied over decades. Probably the most prominent example is the family of \textit{Fast Multipole Methods} (FMMs) \cite{FMM, FMM-2, FMM-3, FMM-4, FMM-5, FMM-6, FMM-7, FMM-8}. FMMs were originally developed for the $N$-body simulation problem \cite{nbody} in force-fields, where $\mathcal{F}$ might encode mass/charge distribution over points, $\mathrm{K}$ defines corresponding potential-field decay over distances and $i(v)$ calculates coordinates of the resultant forces in all $N$ points. Since then, FMMs have been applied in a plethora of applications: (1) molecular/stellar dynamics, (2) interpolation with radial basis functions, and (3) solving differential equations: Poisson/Laplace (fluid dynamics, \citealp{fluid}), Maxwell's (electromagnetism, \citealp{FMM-9}), Helmholtz (acoustic scattering problem, \citealp{FMM-10}).

FMMs were developed for Euclidean spaces corresponding to grid graphs $\mathrm{G}$ embedded in $\mathbb{R}^{d}$ and the specific class of kernels $\mathrm{K}$ defined as functions of the dot-product similarity. However, several applications in machine learning involving point cloud and mesh-graph data require integrating more general graphs (defined on surfaces) approximating non-Euclidean metrics. Examples include: (a) computing Wasserstein distances between probabilistic distributions defined on meshes \cite{solomon2015convolutional}, and (b) interpolation of the velocity fields given on meshes to model the complex dynamics of objects \cite{mesh-net-v2}.

In this paper, we present two new algorithms for the efficient (i.e., \textbf{sub-quadratic} in $N$) graph-field integration for graphs encoding point cloud data (where graph weights correspond to distances betweeen points). The first,  $\mathrm{SeparatorFactorization}$ (\SF), leverages bounded genus of point cloud mesh graphs, while the second, $\mathrm{RFDiffusion}$ (\RFD), uses popular $\epsilon$-nearest-neighbor ($\epsilon$-NN) graph representations for point clouds. Both can be considered to provide the functionality of Fast Multipole Methods but in non-Euclidean spaces. We focus on geometries induced by distributions of walks' lengths between points (e.g., shortest-path distance). We provide an extensive theoretical analysis of the proposed algorithms and, as a byproduct, present new results in structural graph theory. We also perform exhaustive empirical evaluation, including on-surface interpolation for  rigid and deformable objects (e.g., for mesh-dynamics modeling), Wasserstein distance computations for point clouds, including the Gromov-Wasserstein variant as well as point cloud classification. Our code is available at~\url{https://github.com/topographers/efficient_graph_algorithms}.

To summarize, our main contributions are as follows:
\begin{enumerate}
\vspace{-3mm}
\item We propose an $O(N\log^{2}(N))$ time complexity $\mathrm{SF}$ algorithm for approximate graph field integration on mesh-graphs, generalizing methods introduced by \citet{topmasking} (Sec. \ref{sec:sign}, \ref{sec:effiicient-sf}). The algorithm works for kernels of the form: $\mathrm{K}_{f}(w,v)=f(\mathrm{dist}(w,v))$ for the shortest-path-distance function $\mathrm{dist}$ and an \textbf{arbitrary} $f:\mathbb{R} \rightarrow \mathbb{R}$.
\item As a byproduct of methods developed for the $\mathrm{SF}$ algorithm, we construct \textbf{the first} efficient \textbf{exact} graph field integrator of $O(N \log^{2}(N))$ time complexity for unweighted graphs with bounded-length \textit{geodesic-cycles} (e.g., cycles such that some shortest path between any two nodes of the cycle belongs to the cycle). For the important special case: $f_{\lambda}(x)=\exp(-\lambda x)$, we obtain additional computational gains resulting in $O(N\log^{1.383}(N))$ time complexity (Sec. \ref{sec:sign}).
\item We comprehensively compare the $\mathrm{SF}$ algorithm with alternative methods that approximate graph-induced metrics using the powerful technique of \textit{low-distortion trees} \citep{bartal1996probabilistic, charikar, fakcharoenphol2004tight, ld-3, ld-tree, ld-2} (Sec. \ref{sec:exp} and Appendix~\ref{sec-graph-metric-approx}).
\item We propose an $O(N)$ time complexity $\mathrm{RFD}$ algorithm for approximate graph field integration on generalized $\epsilon$-NN graphs. By leveraging 
\textbf{r}andom-\textbf{f}eature-based embeddings, \RFD\ decomposes  $\epsilon$-NN graphs into low-rank dot-product graph space \cite{dot-product-graphs}(Sec. \ref{sec:rfdiff}). \RFD\ works for graph diffusion kernels of the form: $[\mathrm{K}(w,v)]_{w,v \in \mathrm{V}}=\exp(\lambda \mathbf{W}_{\mathrm{G}})$, where $\mathbf{W}_{\mathrm{G}}$ is the weighted adjacency matrix of $\mathrm{G}$, and $\exp$ encodes matrix-exponentiation.
\item We comprehensively compare $\mathrm{RFDiffusion}$ with state-of-the-art algorithms for fast computation of the action of the exponentiated matrix \cite{exp-1, exp-2}. 
\vspace{-2mm}
\end{enumerate}

The $\mathrm{SF}$ algorithm is a combinatorial method leveraging geometries defined by shortest path metrics in the form of the kernel $\mathrm{K}(w,v)=f(\mathrm{dist}(w,v))$. In contrast, the $\mathrm{RFD}$ algorithm is an algebraic approach utilizing geometries defined in terms of the distribution of walks of different lengths between the nodes, not only the shortest paths. The $\mathrm{RFD}$ approach complements the $\mathrm{SF}$ algorithm, acting on the $\epsilon$-NN representation of the point cloud, which is a popular alternative to mesh-graphs used by the $\mathrm{SF}$ algorithm. 

Random feature (RF) map methods are well-known to be an effective strategy to scale up kernel algorithms \citep{rf-1,rf-2, rf-3}. This makes $\mathrm{RFDiffusion}$ the fastest method in our algorithmic portfolio, particularly well-suited for TPU/GPU-powered computations. However, it works for a specific (yet essential) kernel, whereas the $\mathrm{SF}$ algorithm leverages a general class of shortest-path-induced kernels.

\vspace{-3mm}
\section{SeparatorFactorization and RFDiffusion}
\label{sec:alg}
\paragraph{Defining geometries on graphs via walks.}

We start with the following class of kernels $\mathrm{K}^{\Lambda}:\mathrm{V} \times \mathrm{V} \rightarrow \mathbb{R}$ defined on the nodes of the graph $\mathrm{G}$ for the decreasing function $p:\mathbb{N} \rightarrow \mathbb{R}$ and a given hyper-parameter $\Lambda>0$:
\begin{equation}
\label{eq:walks}
\mathrm{K}^{\Lambda}(w,v) = \sum_{k=0}^{\infty} p_{\Lambda}(k) n_{k},
\end{equation}
where $n_{k}$ is the number of walks of length $k$ between $w$ and $v$.
Taking $p_{\Lambda}(k)=\frac{\Lambda^{k}}{k!}$, one reconstructs a version of the so-called \textit{graph diffusion kernel} (GDK).  Using even a simpler formula: $p_{\Lambda}(k)=\Lambda^{k}=\exp(\log(\Lambda)k)$ we obtain another valid kernel related to the \textit{Leontief Inverse matrix}  \cite{leontief, smola_kondor}. 
Note that the sum from the RHS of Eq. \ref{eq:walks} starts de facto from $k=\mathrm{dist}(w,v)$ since $k \geq \mathrm{dist}(w,v)$. 
The first class of kernels considered in this paper is obtained by taking the latter formula for $p(k)$ and its leading $p$-coefficient from the sum in Eq. \ref{eq:walks} (note that longer-walks are discounted as providing weaker ties between vertices).  
Instead of defining $\mathrm{K}(w,v)=\exp(-\lambda \mathrm{dist}(w,v))$ for $\lambda=-\log(\Lambda)$, we consider its generalized version for an \textbf{arbitrary} $f:\mathbb{R} \rightarrow \mathbb{R}$,
\begin{equation}
\mathrm{K}_{f}(w,v) = f(\mathrm{dist}(w,v)).
\end{equation}
Thus the kernel becomes an arbitrary (potentially learnable) function of the shortest path distance. Such kernels are intensively studied for mesh modeling, where $\mathrm{dist}(\cdot, \cdot)$ approximates geodesic distances, see \cite{mory} and \cite{solomon2015convolutional}. In the latter work, that kernel is ultimately replaced by a more computationally feasible variant of the diffusion kernel, see Sec. \ref{sec:wass_dist}.

We also provide an efficient GFI mechanism in the setting where walks of all lengths are considered. Here we decide to work with the aforementioned GDK:
\begin{equation}
\label{eq:diffusion}
[\mathrm{K}^{\Lambda}_{\mathrm{GDK}}(w,v)]_{w,v \in \mathrm{V}}=\exp(\Lambda \cdot \mathrm{W}_{\mathrm{G}}),
\end{equation}
for the weighted adjacency matrix $\mathrm{W}_{\mathrm{G}}$ of $\mathrm{G}$, where $\mathrm{exp}$ denotes matrix exponentiation.
\vspace{-3mm}
\subsection{Tractability and Bounded Genus Graphs}

We start with the tractability concept recently introduced by \citet{topmasking}. 

\begin{definition}[tractable $(\mathcal{G}, f)$-pairs]
\label{def:main}
Let $\mathcal{G}$ be a family of weighted undirected graphs and let $f:\mathbb{R} \rightarrow \mathbb{C}$ be a function. Denote $\mathbf{K}=[\mathrm{K}_{f}(w,v)]_{w,v \in \mathrm{V}}$.
We say that $(\mathcal{G},f)$ is $T$-tractable if for any $\mathbf{x} \in \mathbb{R}^{|\mathrm{V}|}$, matrix-vector multiplication $\mathbf{K}\mathbf{x}$ can be computed in time $O(T)$. If $T=o(|\mathrm{V}|^{2})$, then we say that $(\mathcal{G},f)$ is tractable.
\end{definition}

In Table 1, we summarize previously known results on $T$-tractable $(\mathcal{G}, f)$-pairs, all from \cite{topmasking}. 

\begin{table}[h]
\vspace{-4.0mm}
    \caption{\small{Summary of the known $(\mathcal{G},f)$-tractability results \cite{topmasking}. If not stated otherwise, considered graphs are unweighted. Below, $\mathrm{diam}(\mathrm{G})$ stands for the diameter of $\mathrm{G}$.}} 
    \label{table:speed}
    \scalebox{0.80}{
     \begin{tabular}{ c c c} 
    \toprule
       $\mathcal{G}$ & $f(z)$ & $T$ \\ [0.5ex] 
    \toprule
       weighted trees & $\exp(az+b)$ for given $a,b \in \mathbb{C}$ & $|\mathrm{V}|$ \\ [0.5ex] 
       unweighted trees & arbitrary & $|\mathrm{V}|\log^{2}(\mathrm{|V|})$  \\ [0.5ex] 
       unweighted trees & arbitrary & $|\mathrm{V}|\mathrm{diam(\mathrm{G})}$ \\ [0.5ex] 
       $d$-dimensional grids & arbitrary & $|\mathrm{V}|\log(\mathrm{|V|})$  \\ [0.5ex] 
    \bottomrule
    \end{tabular}}
\end{table}
\vspace{-2mm}
We immediately realize that $(\mathcal{G},f)$-tractability implies efficient GFI for kernels $\mathrm{K}_{f}$. The results for trees from Table 1 can be elegantly extended to trigonometric functions $f$ (still on trees) by using complex field $\mathbb{C}$, see Appendix, Sec. \ref{sec:app_warmup}. 

In this section, we target mesh-graph representations of point clouds that are not trees. But they are not completely random. If the surface where the mesh graph lives does not have too many ``holes", the mesh graph is very structured. This property can be precisely quantified as a \textit{bounded genus}. The \textit{genus} of a surface is the topologically invariant property defined as the largest number of non-intersecting simple closed curves that can be drawn on the surface without separating it~\cite{genus}. 

\begin{theorem}[\citealp{genus-2}]
\label{thm:genus}
The set of vertices $\mathrm{V}$ of graphs of genus $ \leq g$ (i.e., embeddable with no edge-crossings on the surface of genus $g$) can be efficiently (in time $O(|\mathrm{V}|+g)$) partitioned into three subsets: $\mathrm{V}_{1}$, $\mathrm{V}_{2}$ and $\mathcal{S}$ such that:
$|\mathrm{V}_{1}|,|\mathrm{V}_{2}| \geq \frac{|\mathrm{V}|}{3}$, $|S| = O(\sqrt{(g+1)|\mathrm{V}|})$ and furthermore there are no edges between $\mathrm{V}_{1}$ and $\mathrm{V}_{2}$ . 
\end{theorem}

We call set $\mathcal{S}$ a \textit{balanced separator} since it splits the vertices of $\mathrm{V}$ into two ``large" subsets (each of size $\geq c |\mathrm{V}|$ for some universal constant $c$; in our case $c=\frac{1}{3}$). Balanced separators are useful since they often enable using divide-and-conquer strategies to solve problems on graphs. As we show soon, this holds for the GFI problem.

\subsection{Towards SeparatorFactorization: BCTW Graphs}
\label{sec:sign}

Let us consider first an extreme case where bounded-size balanced separators exist. A prominent class of graphs with this property is a family of  \textit{bounded connected treewidth} (BCTW) graphs. We next introduce the concept of \textit{treewidth} ($\mathrm{tw}$), one of the most important graph parameters of modern structural graph theory.

\begin{definition}[tree-decomposition \& treewidth]
A tree-decomposition of a given undirected graph $\mathrm{G}=(\mathrm{V},\mathrm{E})$ is a tree $T$ with  nodes corresponding to subsets (bags) $X_{1},\ldots,X_{L}$ of $\mathrm{V}$ satisfying the following:
\vspace{-3mm}
\begin{itemize}
\item $\bigcup_{i=1}^{L} X_{i} = \mathrm{V}$,
\item for every edge $\{u,w\} \in E$ there exists a bag $X_{i}$ such that $u,w \in X_{i}$,
\item for any two $X_{i},X_{j}$, the subset $X_{i} \cap X_{j}$ is 
contained in all nodes on the (unique) path from $X_{i}$ to $X_{j}$. 
\end{itemize}
\vspace{-3mm}
The \textit{treewidth} of $\mathrm{G}$ is the minimum over different tree-decompositions of $\mathrm{G}$ of 
$\max_{i=1,\ldots,L} |X_{i}|-1$. 
\end{definition}

We say that a family $\mathcal{G}$ of undirected and unweighted graphs has \textit{bounded connected treewidth}, if each $\mathrm{G} \in \mathcal{G}$ has a tree-decomposition, where all the bags are connected graphs of bounded size. BCTW-graphs are extensions of trees. Every tree is a BCTW graph, but BCTW graphs can contain cycles (while trees cannot).
It turns out that bags from the tree-decomposition are themselves separators. 

Next, we show that sparse BCTW graphs admit fast GFI:
\begin{theorem}
\label{theorem:bct_main}
If $\mathcal{G}$ is a family of bounded connected treewidth sparse graphs (i.e., with $|\mathrm{E}|=O(|\mathrm{V}|)$) then $(\mathcal{G},f)$ is $|\mathrm{V}|\log^{2}(|\mathrm{V}|)$-tractable for any $f:\mathbb{R} \rightarrow \mathbb{C}$.
\end{theorem}

Interestingly, a family $\mathcal{G}$ has bounded connected treewidth iff all the geodesic cycles of graphs $\mathrm{G} \in \mathcal{G}$ have bounded length (see \citealp{diestel}). Thus, as a corollary, we immediately get the following result:

\begin{corollary}
If $\mathcal{G}$ is a family of sparse graphs with geodesic cycles of bounded length, then $(\mathcal{G},f)$ is $|\mathrm{V}|\log^{2}(|\mathrm{V}|)$-tractable for any function $f:\mathbb{R} \rightarrow \mathbb{C}$.
\end{corollary}

Below we provide a sketch of the proof of Theorem \ref{theorem:bct_main}, which also serves as pseudocode with a detailed explanation of each step. The full proof is given in Appendix Sec. \ref{sec:app_thmbct}. This sketch will be sufficient to develop a ``practical" version of the method that can be applied to bounded genus mesh graphs.
We introduce extra notation for $\mathcal{Y},\mathcal{Z} \subseteq \mathrm{V}$: 
\begin{align}
\begin{split}
i^{\mathrm{G}}_{\mathcal{Z}}(v) \overset{\mathrm{def}}{=} \sum_{w \in \mathcal{Z}} \mathrm{K}(w,v)\mathcal{F}(w), \\
i^{\mathrm{G}}_{\mathcal{Z}}(\mathcal{Y}) \overset{\mathrm{def}}{=} \{i^{\mathrm{G}}_{\mathcal{Z}}(y):y \in \mathcal{Y}\}.
\end{split}
\end{align}

\textit{Proof sketch of Theorem \ref{theorem:bct_main}:}

\begin{figure}[!htb]
  \includegraphics[width=0.99\linewidth]{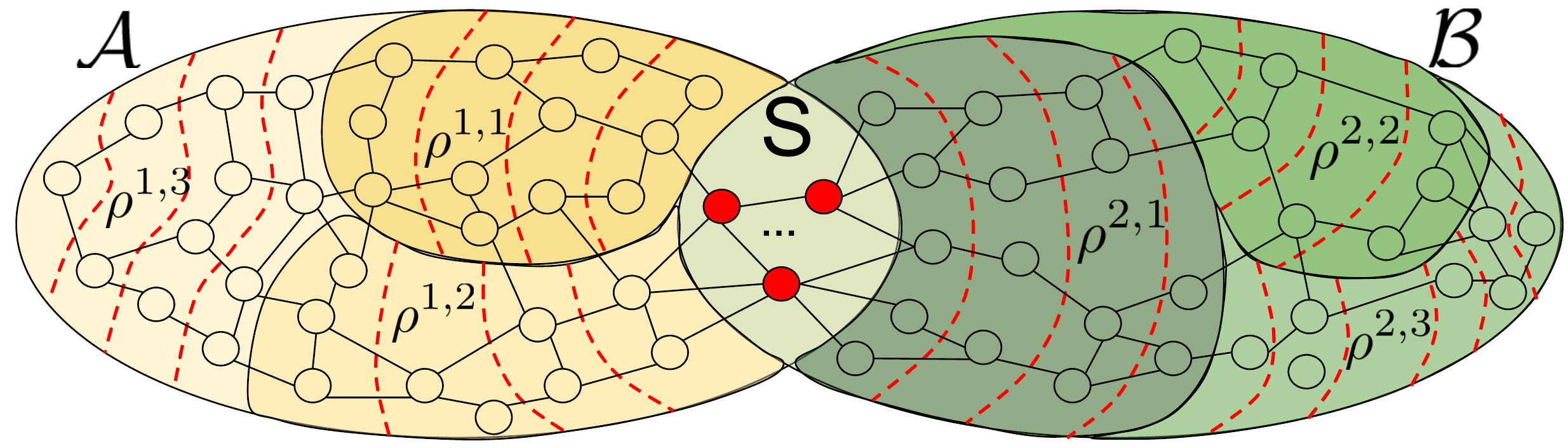}
\vspace{-3mm}  
\caption{\small{Visualization of the proof-sketch of Theorem \ref{theorem:bct_main}. Subsets $\mathcal{A}$ and $\mathcal{B}$ are sliced based on the signature vectors $\rho^{x,y}$ (different shadows of yellow and green). Slices are further partitioned based on the distance from the separator $\mathcal{S}$ (dotted red lines).}}
\label{fig:bct-vis}
\vspace{-2.4mm}
\end{figure}

\paragraph{Step 1: Balanced separation \& initial integration.} We start by finding a small (constant size) balanced separator $\mathcal{S}=\{s_{1},\ldots,s_{|\mathcal{S}|}\}$ splitting $\mathrm{V}$ into two ``large" subsets: $\mathcal{A}$ and $\mathcal{B}$ ($|\mathcal{A}|,|\mathcal{B}| > cN$ for some universal constant $c>0$). It turns out that this can be done in time $O(N)$ (see: Sec. \ref{sec:app_thmbct}). We compute $i^{\mathrm{G}}_{\mathrm{V}}(\mathcal{S})$ using Dijkstra's algorithm ($O(N\log(N))$ time complexity) or its improved variant ($O(N\log \log(N))$ time complexity; \citealp{thorup}). 

It suffices to compute $i^{\mathrm{G}}_{\mathrm{V}}(\mathcal{A} \cup \mathcal{B})$.
Note that for every $v$:
\begin{equation}
i(v) = i^{\mathrm{G}}_{\mathcal{A}}(v) + i^{\mathrm{G}}_{\mathcal{B}}(v) + i^{\mathrm{G}}_{\mathcal{S}}(v).
\end{equation}

\paragraph{Step 2: Computing $i^{\mathrm{G}}_{\mathcal{S}}(\mathcal{A} \cup \mathcal{B})$.} As before, this can be done in time $O(N\log(N))$ (or even $O(N\log\log(N))$).

\paragraph{Step 3: Computing $i^{\mathrm{G}}_{\mathcal{A}}(\mathcal{A})$ and $i^{\mathrm{G}}_{\mathcal{B}}(\mathcal{B})$.}
This can be done by running the algorithm recursively on the sub-graphs $\mathrm{G}[\mathcal{A}]$, $\mathrm{G}[\mathcal{B}]$ of $\mathrm{G}$ induced by $\mathcal{A}$ and $\mathcal{B}$ respectively (for a rigorous proof, we actually need to run it for extended versions of $\mathrm{G}[\mathcal{A}]$ and $\mathrm{G}[\mathcal{B}]$ since the shortest path between two vertices in $\mathcal{A}/\mathcal{B}$ can potentially use vertices $\notin \mathcal{A}/\mathcal{B}$; crucially, as we show in Sec. \ref{sec:app_thmbct}, those extended versions are obtained by adding only a \textbf{constant} number of extra vertices; for the practical variant we apply the simplified version though). 


\paragraph{Step 4: Computing $i^{\mathrm{G}}_{\mathcal{A}}(\mathcal{B})$ and $i^{\mathrm{G}}_{\mathcal{B}}(\mathcal{A})$.}

We will show how to compute the latter. The former can be calculated in a completely analogous way.

\paragraph{Substep 4.1: $\mathcal{A},\mathcal{B}$-slicing based on signature vectors.}
For every vertex $v \in\mathcal{A} \cup \mathcal{B}$, we define $\chi_{v} \in \mathbb{R}^{|\mathcal{S}|}$ as $\chi_{v}[k]=\mathrm{dist}(v,s_{k})$ for $k=1,\ldots,|\mathcal{S}|$. Write: $\chi_{v}=\tau_{v} + \rho_{v}$, where $\tau_{v}[i]=\min_{k \in \mathcal{S}} \mathrm{dist}(v,k),  \forall i$. We call vector $\rho_{v} \in \mathbb{R}^{|\mathcal{S}|}$ the \textit{signature vector} ($\mathrm{sg}$-$\mathrm{vect}$).  The critical observation is that not only does this vector have bounded dimensionality (since $\mathcal{S}$ is of constant size), but a bounded number of different possible values of different dimensions, i.e., for every $i=1,\ldots,|\mathcal{S}|$:
\begin{equation}
0 \leq \rho_{v}[i] \leq |\mathcal{S}| - 1.
\end{equation}
This is an immediate consequence of the fact that the separator is connected. This implies that there is only a finite (yet super-exponentially large in $|\mathcal{S}|$ !) number of different signature vectors $\rho_{v}$. We partition $\mathcal{A}$ into subsets corresponding to different signature vectors, called: $\rho^{1,1},\rho^{1,2},\ldots$ and similarly, partition $\mathcal{B}$ into subsets corresponding to different signature vectors: $\rho^{2,1},\rho^{2,2},\ldots$ (see Fig. \ref{fig:bct-vis}).

\paragraph{Substep 4.2: Partitioning slices.}

Fix a subset $A_{\rho^{1,l}} \subseteq \mathcal{A}$ corresponding to some $\rho^{1,l}$ and a subset $B_{\rho^{2,t}} \subseteq \mathcal{B}$ corresponding to some $\rho^{2,t}$. Note that for every $v \in A_{\rho^{1,l}}$ and $w \in B_{\rho^{2,t}}$ the following holds:
\begin{equation}
\label{eq:neat_formula}
\mathrm{dist}(w,v) = \tau_{w}[1] + \tau_{v}[1] + \min_{k \in \mathcal{S}} (\rho^{2,t}[k] + \rho^{1,l}[k]).
\end{equation}
Furthermore, the last element of the RHS above does not depend on $w$ and $v$, and $\mathrm{dist}(w,v)$ depends only on the distances of $w$ and $v$ from $\mathcal{S}$. We thus partition subsets $A_{\rho^{1,l}}$ and $B_{\rho^{2,t}}$ based on the value of $\tau_{v}[1]$ and $\tau_{w}[1]$ accordingly. To compute $i^{\mathrm{G}}_{\mathcal{B}^{2,t}}(A_{\rho^{1,l}})$, it is thus sufficient to compute a sequence: $\{i^{\mathrm{G}}_{\mathcal{B}^{2,t}}(v_{i})\}$ for $i=0,1,\ldots$ and where $v_{i}$ is an arbitrary vertex of $A_{\rho^{1,l}}$ with $\mathrm{dist}(v_{i},\mathcal{S})=i$. This can be done via a linear transformation encoded by a Hankel matrix $\mathbf{W}$ as in the proof of Lemma 6.1 from \cite{topmasking} using Fast Fourier Transform (FFT) in time $O(N\log(N))$. Furthermore, if $f_{\lambda}(x)=\exp(-\lambda x)$, multiplication with $\mathbf{W}$ can be done in time $O(N)$. Putting together computations for all (constant) number of pairs: $(A_{\rho^{1,l}}, B_{\rho^{2,t}})$, we conclude that $i_{\mathcal{B}}^{\mathcal{G}}(\mathcal{A})$ can be computed in time $O(N\log(N))$ (or even $O(N\log\log(N))$ for $f_{\lambda}(x)=\exp(-\lambda x)$). Solving the corresponding time complexity recursion, we obtain total \textbf{pre-processing time} $O(N\log(N))$ and \textbf{inference time} $O(N \log^{2}(N))$ (or even $O(N\log^{1.38}(N))$ for $f_{\lambda}(x)=\exp(-\lambda x)$). This completes the proof sketch.

\vspace{-3mm}
\subsection{SeparatorFactorization}
\label{sec:effiicient-sf}
The $\mathrm{SeparatorFactorization}$ method is a straightforward relaxation of the
the algorithm presented above.
The relaxation to make the approach practical (for approximate GFI) is based on the following pillars:
\begin{enumerate}
    \vspace{-3mm}
    \item \textbf{Separator truncation.} Replacing small $\mathcal{S}$ from Theorem \ref{thm:genus} (not necessarily of constant size) with its sub-sampled constant-size subset $\mathcal{S}^{\prime}$ (the other vertices of $\mathcal{S}$ are distributed across $\mathcal{A}$ and $\mathcal{B}$ randomly). Balanced separation is computed via the algorithmic version of Theorem \ref{thm:genus} (see Fig. \ref{fig:sphinx} for an example).
    \item \textbf{Clustering signature vectors.} Instead of partitioning sets $\mathcal{A}$ and $\mathcal{B}$ based on $\mathrm{sg}$-$\mathrm{vect}$s (their number is finite yet super-exponentially large in $|\mathcal{S}^{\prime}|$), the partitioning is based on their hashed versions. We use a constant number of hashes. Every hashing mechanism that can be computed in time $O(N\log(N))$ is acceptable, and thus LSH methods can be applied \cite{lsh}. We found that in practice, initial partitioning based on $\mathrm{sg}$-$\mathrm{vect}$ (substep 4.1) can be avoided. Good quality approximate GFI can be obtained by only one-level partitioning of $\mathcal{A}$ and $\mathcal{B}$ (based on the distance from $\mathcal{S}^{\prime}$). 
\vspace{-3mm}
\end{enumerate}
As in Sec. \ref{sec:sign}, this approach leads to $O(N\log(N))$ \textbf{pre-processing time} and $O(N\log^{2}(N))$ \textbf{inference time}. Furthermore, for weighted graphs, all the distances are effectively quantized (meaning natural numbers can approximate them). Finally, we stop the recursive unroll when the subsets $\mathcal{A}$ and $\mathcal{B}$ are small enough (when brute-force matrix-vector multiplication for GFI is fast enough).
\vspace{-2mm}
\subsection{RFDiffusion}
\label{sec:rfdiff}
In contrast to $\mathrm{SeparatorFactorization}$ above, the $\mathrm{RFDiffusion}$ algorithm leverages an $\epsilon$-NN (Nearest Neighbor) representation of point clouds (see Fig. \ref{fig:sphinx}). This representation is particularly convenient for the graph diffusion kernel from Eq. \ref{eq:diffusion} used by $\mathrm{RFDiffusion}$.

$\mathrm{RFDiffusion}$ starts by producing a low-rank decomposition of the weighted adjacency matrix $\mathbf{W}_{\mathrm{G}}$ of $\mathrm{G}$, defined via a set of vectors $\{n_{i} : i \in \mathrm{V}\} \subseteq \mathbb{R}^{d}$ and $f:\mathbb{R}^{d} \rightarrow \mathbb{R}$:
\vspace{-1mm}
\begin{equation}
\label{eq:w}
\mathbf{W}_{\mathrm{G}}(i,j)=f(\mathbf{n}_{i}-\mathbf{n}_{j}) .
\vspace{-2mm}
\end{equation}
The \textit{generalized $\epsilon$-NN graphs} are special instantiations of such graphs, where $f$ is defined as  
$f(\mathbf{z})=h(\|\mathbf{z}\|)$ for  non-increasing $h$ with compact support, where vectors $\mathbf{n}_{i}$ are the points themselves
and $\|\cdot\|$ is some norm (e.g. $f(\mathbf{z}) = \mathbbm{1}[\|\mathbf{z}\| \leq \epsilon]$, as for the regular $\epsilon$-NN graph).
\begin{figure}\vspace{-2mm}
\centering
  \includegraphics[width=0.75\linewidth]{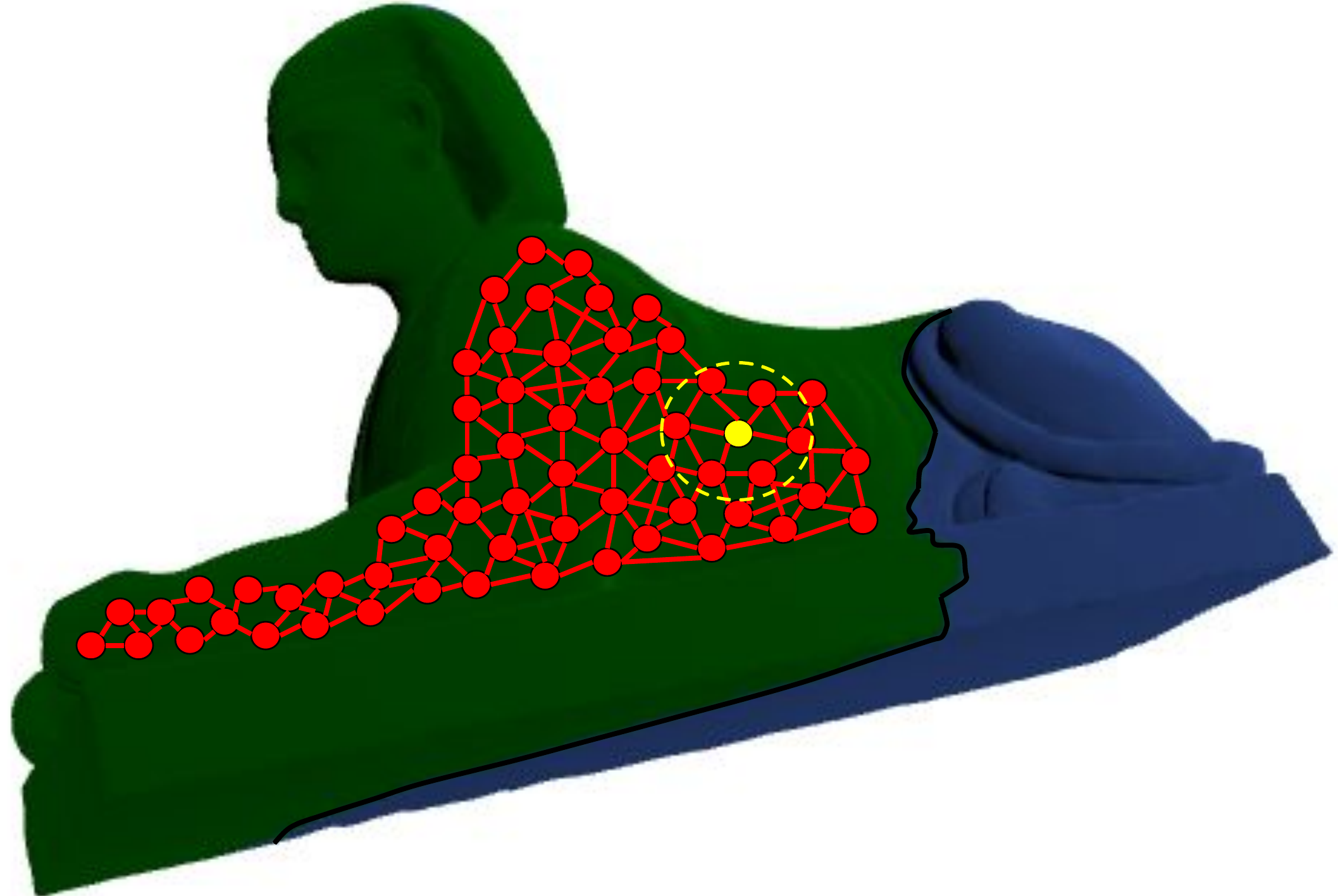}
\vspace{-3mm}  
\caption{\small{A sphinx mesh with $\mathbf{1.17M}$ faces, its first-level balanced separation obtained via $\mathrm{SeparatorFactorization}$} (with $\mathbf{685K}$ faces entirely in one class and $\mathbf{486K}$ entirely in the second one) and a visualization of the $\epsilon$-NN graph leveraged by $\mathrm{RFDiffusion}$. The graph is different from the mesh graph; particularly in this picture, it is not planar even though the mesh graph is.}
\label{fig:sphinx}
\vspace{-4mm}
\end{figure}

Our goal is to rewrite: $\mathbf{W}_{\mathrm{G}}(i,j) \approx \phi(\mathbf{n}_{i})^{\top}\psi(\mathbf{n}_{j})$ for maps $\phi,\psi:\mathbb{R}^{d} \rightarrow \mathbb{C}^{m}$ and $m \ll N$.
Note that (for $\mathbf{i}^{2}=-1$):
{ \small
\begin{align}
\begin{split}
f(\mathbf{z}) &= \int_{\mathbb{R}^{d}} \exp(2 \pi \mathbf{i} \omega^{\top}\mathbf{z})\tau(\omega)d\omega\\
&= \int_{\mathbb{R}^{d}} \exp(2 \pi \mathbf{i} \omega^{\top}\mathbf{z})\frac{\tau(\omega)}{p(\omega)}p(\omega) d\omega,    
\end{split}
\end{align}
}
where $\tau: \mathbb{R}^{d} \rightarrow \mathbb{C}$ is the Fourier Transform (FT) of $f$ and $p$ is a pdf function corresponding to some probability distribution $P \in \operatorname{Prob}(\mathbb{R}^{d})$. Take 
$\omega_{1},\ldots,\omega_{m} \overset{\mathrm{iid}}{\sim} P$. For $\mathbf{v} \in \mathbb{R}^{d}$, define $\rho_{j}=2 \pi \mathbf{i} \omega_{j}^{\top}\mathbf{v}$ and $\nu_{i}=\sqrt{\frac{\tau(\omega_{i})}{p(\omega_{i})}}$. Then, using Monte Carlo approximation, we can estimate: $\mathbf{W}_{\mathrm{G}}(i,j) \approx \phi(\mathbf{n}_{i})^{\top}\psi(\mathbf{n}_{j})$
for $\phi=\sigma_{1}$, $\psi=\sigma_{-1}$ and
$\sigma_{c}(\mathbf{v}) = \frac{1}{\sqrt{m}} 
\left(\exp(2 \pi c \mathbf{i} \omega_{1}^{\top}\mathbf{v})\nu_{1},\ldots,\exp(2 \pi c \mathbf{i} \omega_{m}^{\top}\mathbf{v})\nu_{m}\right)^{\top} \vspace{-1mm}
$.

\textbf{Note.} Distribution $P$ should 1) provide efficient sampling, 2) have easy-to-compute pdf, and 3) (ideally) provide low estimation variance. Here we use (truncated) Gaussian.

We conclude that we can decompose $\mathbf{W}_{\mathrm{G}}(i,j)$ as $\mathbf{W}_{\mathrm{G}}(i,j) = \mathbf{A}\mathbf{B}^{T}$, where the rows of $\mathbf{A} \in \mathbb{R}^{N \times m}$ and $\mathbf{B} \in \mathbb{R}^{N \times m}$ are given as:
$\{\sigma_{1}(\mathbf{n}_{j}):j \in \mathrm{V}\}$ and $\{\sigma_{-1}(\mathbf{n}_{j}):j \in \mathrm{V}\}$ respectively. Now note that we have the following:
\vspace{-2mm}
{\small
\begin{align}
\begin{split}
\exp(\Lambda \cdot \mathbf{A}\mathbf{B}^{\top}) &= \sum_{i=0}^{\infty} \frac{1}{i!}(\Lambda \mathbf{A}\mathbf{B}^{\top})^{i} \\
=
\mathbf{I}& + \sum_{i=0}^{\infty}\frac{1}{(i+1)!}\mathbf{A}(\Lambda\mathbf{B}^{\top}\mathbf{A})^{i+1}\mathbf{A}^{-1} \\
= 
\mathbf{I}& + \mathbf{A}[\exp(\Lambda \mathbf{B}^{\top}\mathbf{A}) - \mathbf{I}](\mathbf{B^{\top}\mathbf{A}})^{-1}\mathbf{B}^{\top}
\end{split}
\end{align}
}

\begin{figure*}[!htb]
\begin{center}
  \includegraphics[width=0.8\linewidth]{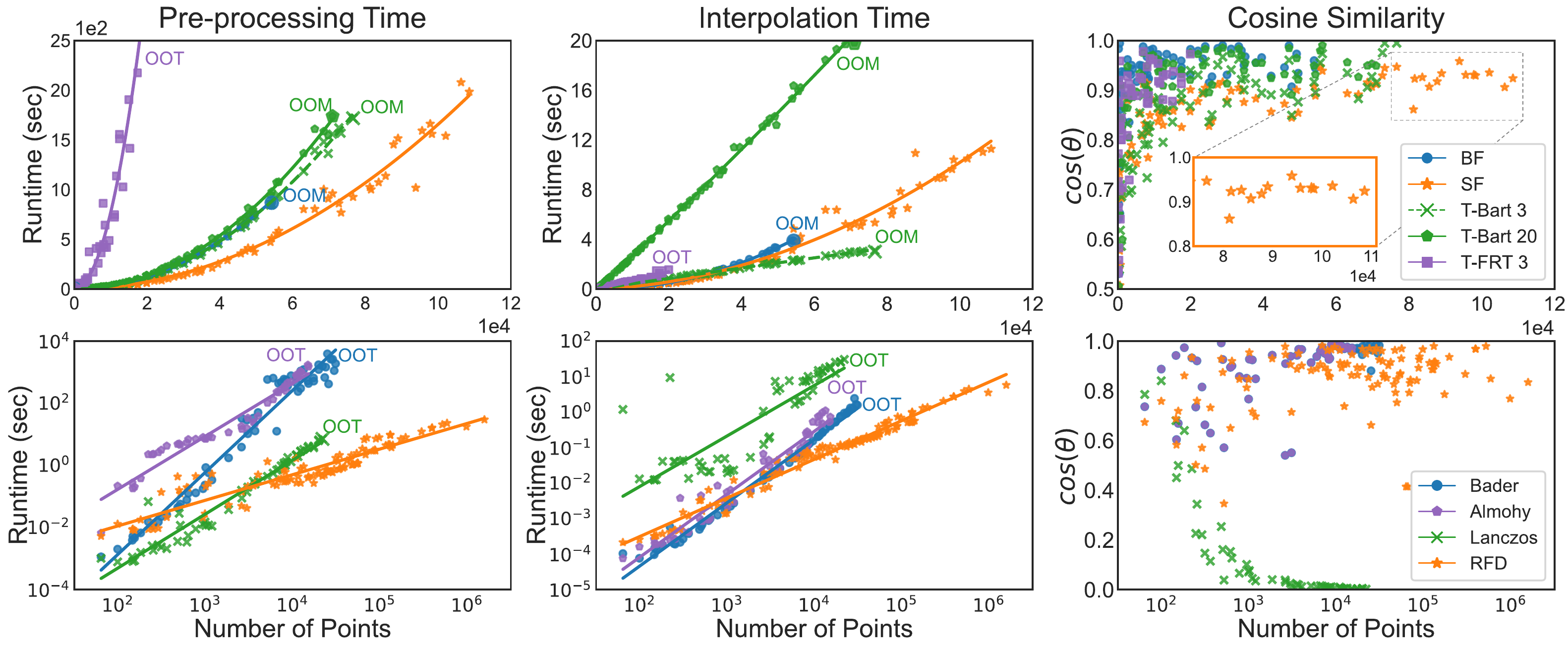}
\vspace{-5mm}  
\caption{\small{\textbf{First Row}: Vertex normal prediction with SF and comparison with relevant low-distortion trees methods. \textbf{Second Row}: The same task, but with RFD and the corresponding methods for multiplications with matrix exponentials. All methods except \SF \ and \RFD \ either went out of memory (OOM) or ran out of time (OOT). The two algorithms maintain high accuracy even on large meshes.}}
\label{fig:vertex_normal}
\end{center}
\end{figure*}
\vspace{-2mm}
We can thus approximate: $\exp(\Lambda \cdot \mathbf{W}_{\mathrm{G}})\mathbf{x}$ for any vector $\mathbf{x} \in \mathbb{R}^{N}$ (which leads to the GFI algorithm) as:
{\small
\begin{equation}
\mathbf{x} + 
\mathbf{A}([\exp(\Lambda \mathbf{B}^{\top}\mathbf{A}) - \mathbf{I}]((\mathbf{B^{\top}\mathbf{A}})^{-1}(\mathbf{B}^{\top}
\mathbf{x}))),
\end{equation}
}
where brackets indicate the order of computations. We see that this algorithm has: (a) \textbf{pre-processing} time linear in $N$ and cubic in the number of random features $m$, and (b) \textbf{inference} stage linear in $N$ and quadratic in $m$. Algorithm $\mathrm{RFDiffusion}$ can be thought of as approximating the graph given by Eq. \ref{eq:w} via random feature map-based smoothing (in our applications with $f$ given as a threshold function and $d=3$) and has an excellent property - its running time is \textbf{independent} (Fig. \ref{fig:epsilon_ablates} in Appendix~\ref{sec:ablates}) of the number of edges of the graph (that is never explicitly materialized). 

It remains to compute function $\tau$. Fortunately, this is easy for several threshold functions $f$ defining $\epsilon$-NN graphs.  For instance, for $f(\mathbf{z}) = \mathbbm{1}[\|\mathbf{z}\|_{1} \leq \epsilon]$ and $\xi \in \mathbb{R}^{d}$, we have: 
\vspace{-3mm}
{\small
\begin{equation}
\tau(\mathbf{\xi}) = \prod_{i=1}^{d} \frac{\mathrm{sin}(2\epsilon \xi_{i})}{\xi_{i}},
\end{equation}
}
and for $f(\mathbf{z}) = \mathbbm{1}[\|\mathbf{z}\|_{2} \leq \epsilon]$, $\tau$ is the $d$-th order \textit{Bessel function} \cite{bessel}.

We quantify the quality of the estimation of the original $\epsilon$-NN graph with $\mathrm{L}_{1}$-norm via $\mathrm{RFDiffusion}$ (proof in Sec \ref{sec:lemma_rfd}). Analogous results can be derived from other norms.

\begin{lemma}
\label{lemma:rfd}
Take the $\epsilon$-NN point cloud graph in $\mathbb{R}^{3}$ with respect to the $\mathrm{L}_{1}$-norm. For two given vertices $v$ and $w$, denote by $\mathrm{MSE}(\widehat{\mathbf{W}}(v,w))$ the mean squared error of the $\mathrm{RFDiffusion}$-based estimation of the true weight $\mathbf{W}(v,w)$ between $v$ and $w$ (defined as: $1$ if $\mathrm{dist}_{\mathrm{L}_{1}}(v,w) \leq \epsilon$ and $0$ otherwise). Let $\mathrm{P}$ be a Gaussian distribution truncated to the $L_{1}$-ball $\mathcal{B}(R)$ of radius $R$, used by RFD. Assume that $v$ and $w$ are encoded by $\mathbf{n}_{v}$ and $\mathbf{n}_{w}$. Then:
{\small
\begin{align}
\begin{split}
\mathrm{MSE}(\widehat{\mathbf{W}}(v,w)) 
\leq \frac{1}{m}\left((2\pi)^{\frac{3}{2}}C(\Gamma_{\epsilon}(R))^{d}-\theta_{1}^{2}\right)+\theta_{2}, 
\vspace{-2mm}
\end{split}
\end{align}
}
\end{lemma}
for $\theta_{1}=(f(\mathbf{z})+ \gamma)$, $\theta_{2}=\gamma(2f(\mathbf{z})+\gamma)$, $\mathbf{z}=\mathbf{n}_{v}-\mathbf{n}_{w}$, $\mathrm{C} = \int_{\mathcal{B}(R)} (2\pi)^{-\frac{3}{2}}\exp(-\frac{\|\mathbf{r}\|^{2}}{2})d\mathbf{r}$,
$\Gamma_{\epsilon}(R)= \int_{-R}^{R}\frac{\sin^{2}(\epsilon x)}{x^{2}} dx$
and $\gamma = -\int_{\mathbb{R}^{d} \backslash \mathcal{B}(R)}
\cos(2\pi \omega^{\top}\mathbf{z})\prod_{i=1}^{d}\frac{\sin(2 \epsilon \omega_{i})}{\omega_{i}} d\omega$.

\section{Experiments}
\label{sec:exp}
To evaluate \SF\ and \RFD, we choose two broad applications: a) interpolation on meshes and b) Wasserstein distances and barycenters computation on point clouds. The kernel matrices chosen for \SF\ and \RFD\ are ~$\mathrm{K}(i,j):=\exp(-\lambda \mathrm{dist}(i,j))$ and~$\mathrm{K}:= \exp(\lambda \mathbf{W}_{\mathrm{G}})$ respectively. Moreover we demonstrate the effectiveness of \RFD\ kernel in various point cloud and graph classification tasks. 
 
\subsection{Interpolation on Meshes}\label{sec:interpolation_meshes}
In this section, we use our methods to predict the masked properties of meshes. In particular, we compare the computational efficiency of \SF\ and \RFD\ against baselines in predicting vertex normals and nodes' velocities in meshes.

\textbf{Vertex normal prediction.}\label{sec:vertex_normal_prediction}
In this setup, we predict the field of normals in vertices from its masked variant. We are given a set of nodes with vertex locations $\mathbf{x}_i\in\mathbb{R}^3$ and vertex normals $\mathbf{F}_i\in\mathbb{R}^3$ in a mesh $\mathrm{G}$ with vertex-set $\mathrm{V}$. In each mesh, we randomly select a subset $\mathrm{V}' \subseteq \mathrm{V}$ with $|\mathrm{V}'|=0.8|\mathrm{V}|$ and mask out their vertex normals (set as zero vectors). Our task is to predict the vertex normals of each masked node $i\in \mathrm{V}'$ computed as:
{\small
\begin{equation*} \vspace{-2mm}
\mathbf{F}_i=\sum_{j\in \mathrm{V}\setminus \mathrm{V}'}\mathrm{K}(i,j)\mathbf{F}_j,  
 \vspace{-1mm}
\end{equation*}
}
We perform a grid search on $\lambda$ and other algorithm-specific hyper-parameters for each mesh. We report the result with the highest cosine similarity between predicted and ground truth vertex normals, averaged over all the nodes.


We run tests on \textbf{120 meshes} for 3D-printed objects with a wide range of sizes from the Thingi10k~\cite{Thingi10K} dataset (see Sec. ~\ref{sec:mesh_interpolation} for details).
Fig. \ref{fig:vertex_normal} reports the pre-processing time, interpolation time, and cosine similarity for algorithms on meshes with different sizes $|\mathrm{V}|$. In the first row, we compare \SF\ with brute-force  ($\mathrm{BF}$) (explicit kernel-matrix materialization followed by matrix-vector multiplications) and low-distortion tree-based algorithms such as Bartal trees (T-Bart-n; $n$ is the number of trees, \citealp{bartal1996probabilistic}) and FRT trees (T-FRT, \citealp{fakcharoenphol2004tight}) (see Appendix~\ref{sec-graph-metric-approx} for details). \SF\ is the fastest in pre-processing and 
accurately interpolates on large meshes while BF, T-FRT, and T-Bart gradually run out of memory or time (OOM/OOT).

In the second row of Fig. \ref{fig:vertex_normal}, we compare \RFD\ with three other algorithms for multiplications with matrix exponentials, including Bader's algorithm \cite{bader2019computing}, Al-Mohy's algorithm \cite{al2010new}, and Lanczos method \cite{orecchia}. As the mesh size increases, the pre-processing time of Bader and Al-Mohy grows quickly. The performance of the Lanczos algorithm is positively correlated with hyper-parameter $m$, which controls the number of Arnoldi iterations. Even though we chose $m$ to be relatively large (which affects the interpolation time), its performance still drops quickly as mesh size grows. In contrast, \RFD\ scales well to large meshes. For example, on the mesh with \textbf{1.5M} nodes, \RFD\ needs only \textbf{29.7} seconds for the pre-processing and \textbf{5.7} seconds for interpolation. Detailed ablation studies are given in Sec.~\ref{sec:ablates_vertex_normal_prediction}.

\textbf{Velocity prediction.}
We further evaluate our algorithms on the deformable \texttt{flag\_simple} dataset from \cite{pfaff2020learning}. The largest mesh sizes from that dataset are of order $~\sim1.5$k nodes; thus, one can, in principle, apply brute force methods. Therefore this dataset was used only to provide a vision-based validation of the techniques. Fig. \ref{fig:acceleration_prediction} shows four sample snapshots of the mesh. The vertex location $\mathbf{x}_i\in\mathbb{R}^3$ and velocity $\mathbf{F}_i\in\mathbb{R}^3$ from each node $\mathbf{n}_i$ in the snapshot are used for interpolation. We randomly mask out 5\% of the nodes in each mesh and do a similar interpolation for vertex normals. In the supplementary material, we provide videos representing the dynamics of the deformable meshes and their corresponding fields (ground truth and predicted).

\vspace{0mm}  
\begin{figure}
\vspace{-2mm}
\centering
  \includegraphics[width=0.82\linewidth]{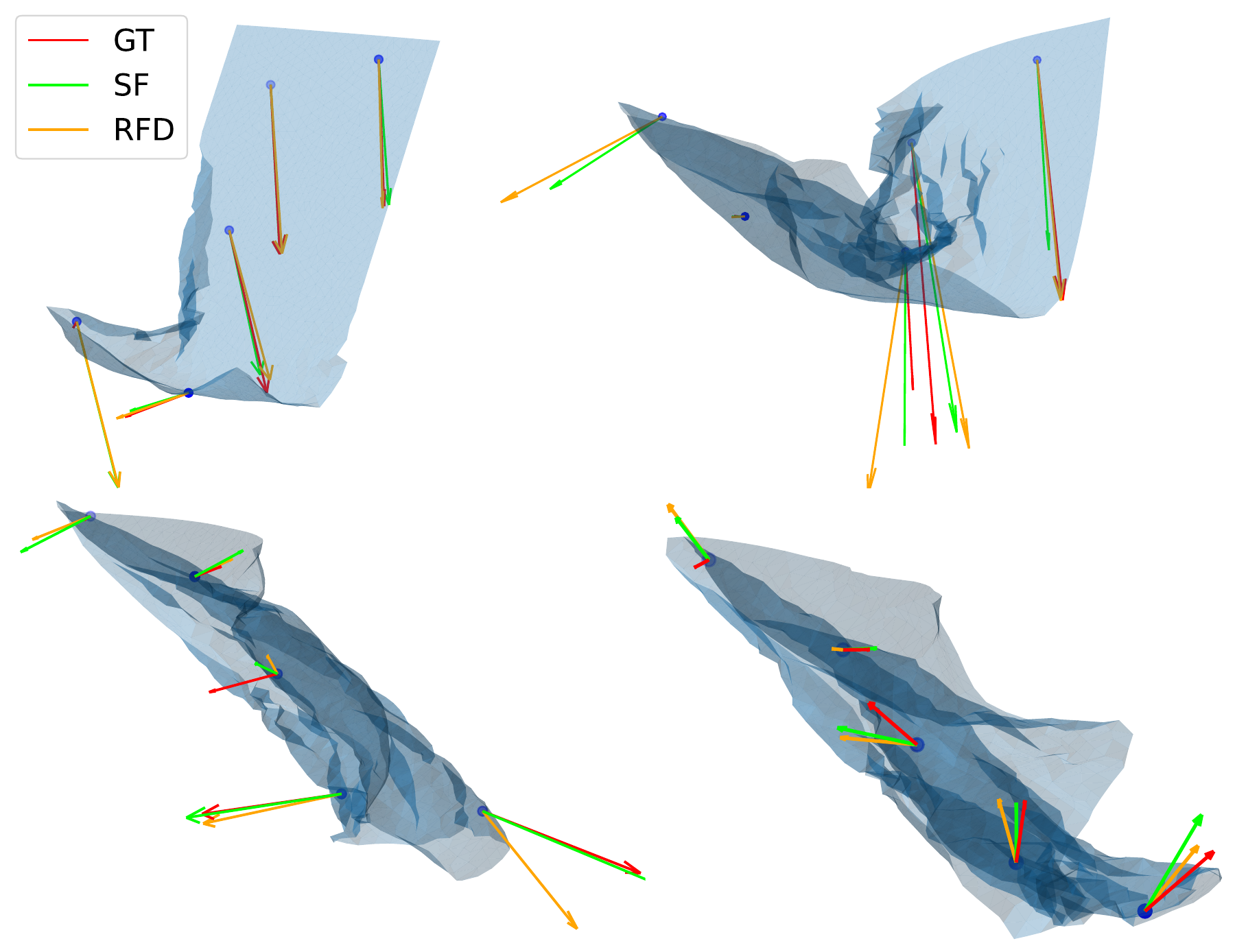}
\vspace{-5mm}  
\caption{\small{Snapshots of meshes for the velocity prediction task comparing results of our GFI methods with ground truth (\textbf{GT}). In several cases, predictions and ground truth are close enough that the velocity vectors appear on top of each other.}}
\label{fig:acceleration_prediction}
\vspace{-3mm}
\end{figure}

\begin{table}[hbt]
    \begin{center}
            \caption{\small{Comparison of the total runtime and mean-squared error (MSE) across several meshes for diffusion-based integration. Runtimes are reported in seconds. The lowest time for each mesh is shown in bold. MSE is calculated w.r.t. the output of \BF.}}
    \small{
        {\begin{tabular}{@{}lc c cc c r@{}}
        \toprule
        \multirow{2}{*}{\textbf{Mesh}} & \multirow{2}{*}{\textbf{$|\mathrm{V}|$}} & & \multicolumn{2}{c}{\textbf{Total Runtime}} & &\multirow{2}{*}{\textbf{MSE}} \\
        \cmidrule{4-5} 
        & & & \BF & \RFD & & \\
        \midrule
        {{Alien}} & $5212$ & & 8.06 & \textbf{0.39} & & 0.041 \\
        {{Duck}} & $9862$ & & 45.36 & \textbf{1.10} & & 0.002\\
        {{Land}} & $14738$ & & 147.64 & \textbf{2.17} & & 0.017\\
        {{Octocat}} & $18944$ & & 302.84 & \textbf{3.36} & & 0.027\\
        \bottomrule
        \end{tabular}}
        \vspace{-2mm}
\label{table:wass_bary_diffusion}} 
    \end{center}
\end{table}

\begin{table}[hbt]
    \begin{center}
    \caption{Setup as in Table~\ref{table:wass_bary_diffusion}, but for the $\mathrm{SF}$ algorithm.}
    \small{
    \begin{tabular}{@{}lc c cc@{} c c@{}}
        \toprule
        \multirow{2}{*}{\textbf{Mesh}} & \multirow{2}{*}{\textbf{$|\mathrm{V}|$}} & & \multicolumn{2}{c}{\textbf{Total Runtime}} & & \multirow{2}{*}{\textbf{MSE}} \\
        \cmidrule{4-5} 
        & & & \multicolumn{1}{c}{\BF} & \SF & &  \\
        \midrule
        {Dice} & $4468$ & & 6.8 & \textbf{4.9} & &  0.063 \\
        {Duck} & $9862$ & & 39.2 & \textbf{19.4} & &  0.002\\
        {Land} & $14738$ & & 90.7 & \textbf{38.9} & & 0.015\\
        {bubblepot2} & $18633$ & & 113.2 & \textbf{48.3} & &  0.081\\
        \bottomrule
        \end{tabular}
        \vspace{-2mm}
   \label{table:wass_bary_separation}
        }
    \end{center}
\end{table}


\begin{figure*}\vspace{-3mm}
     \centering
     \begin{subfigure}[b]{0.16\textwidth}
         \centering
         \includegraphics[trim={1cm 1cm 1cm 1cm},clip,width=\textwidth]{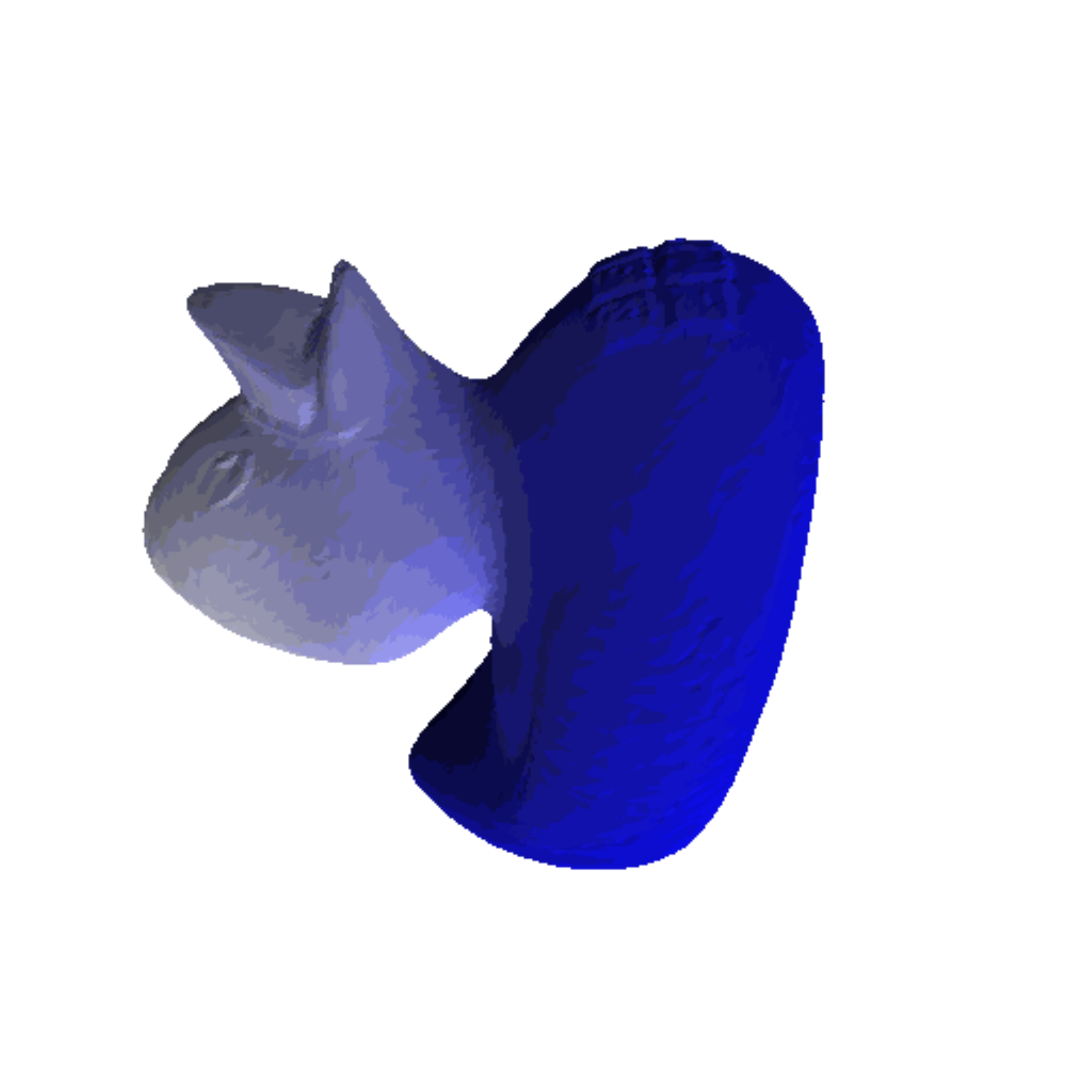}\vspace{-7mm}
         \caption{}\vspace{-1mm}
     \end{subfigure}
     \begin{subfigure}[b]{0.16\textwidth}
         \centering
         \includegraphics[trim={1cm 1cm 1cm 1cm},clip,width=\textwidth]{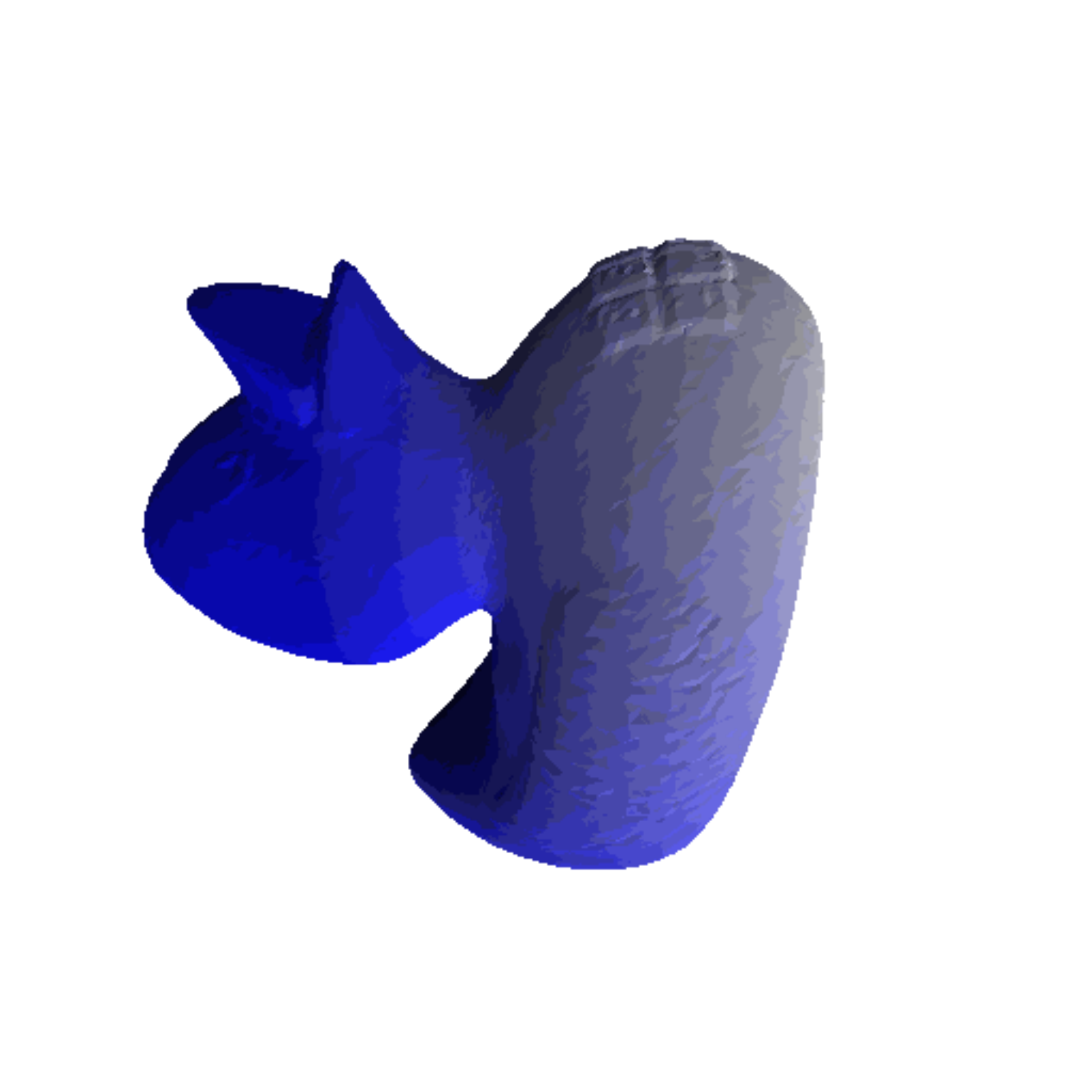}\vspace{-7mm}
         \caption{}\vspace{-1mm}
     \end{subfigure}
     \begin{subfigure}[b]{0.16\textwidth}
         \centering
         \includegraphics[trim={1cm 1cm 1cm 1cm},clip,width=\textwidth]{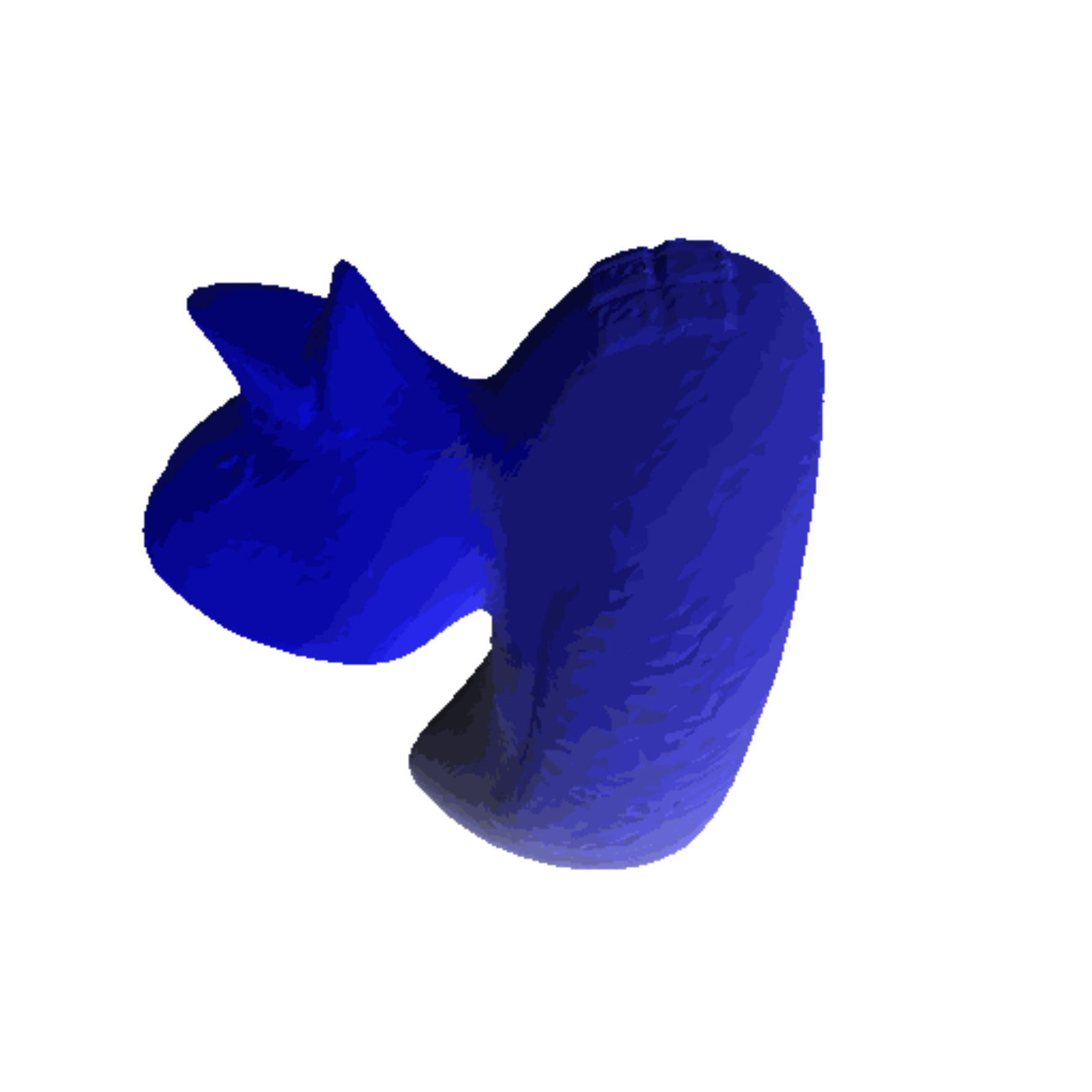}\vspace{-7mm}
         \caption{}\vspace{-1mm}
     \end{subfigure}
     \begin{subfigure}[b]{0.16\textwidth}
         \centering
         \includegraphics[trim={1cm 1cm 1cm 1cm},clip,width=\textwidth]{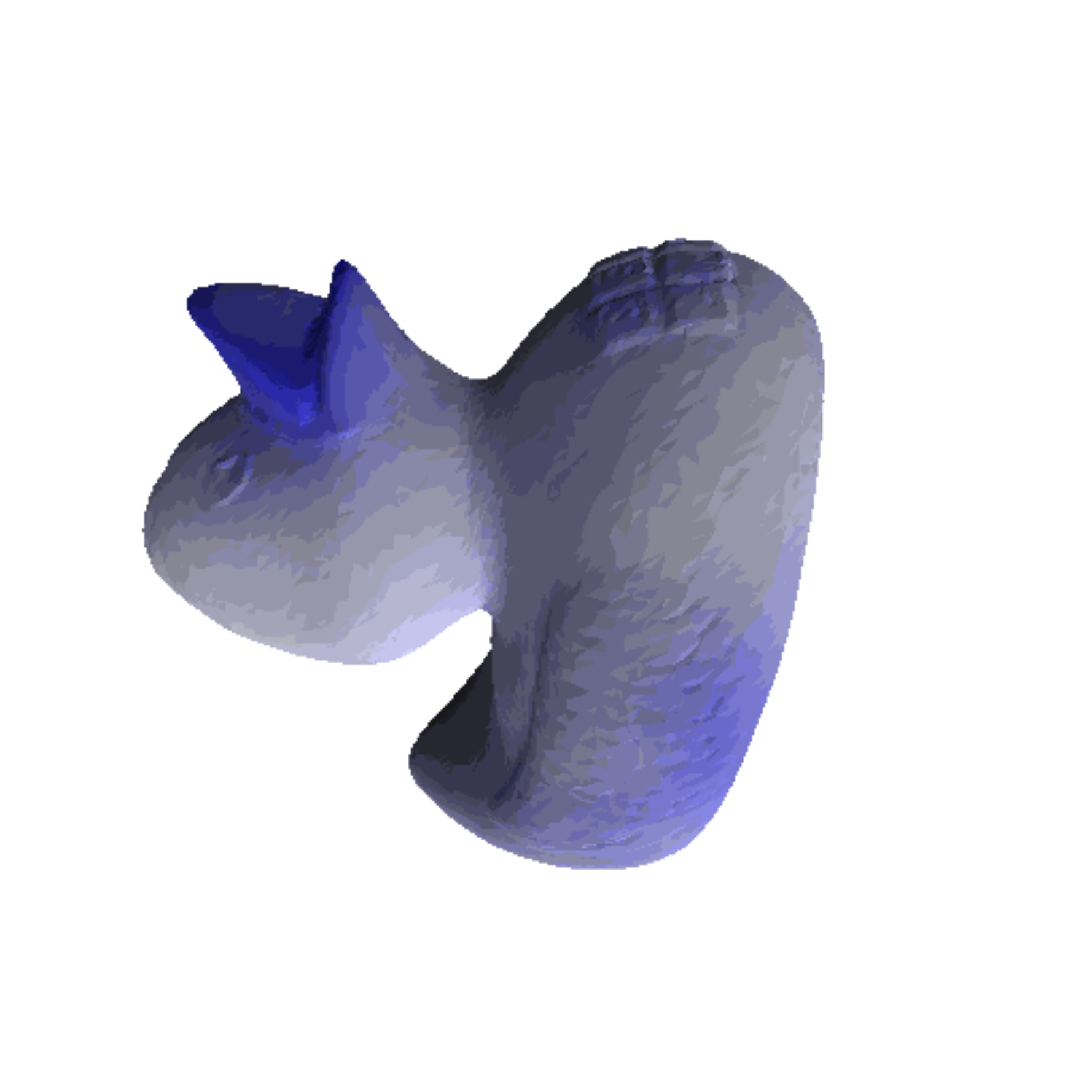}\vspace{-7mm}
         \caption{}\vspace{-1mm}
     \end{subfigure}
     \begin{subfigure}[b]{0.16\textwidth}
         \centering
         \includegraphics[trim={1cm 1cm 1cm 1cm},clip,width=\textwidth]{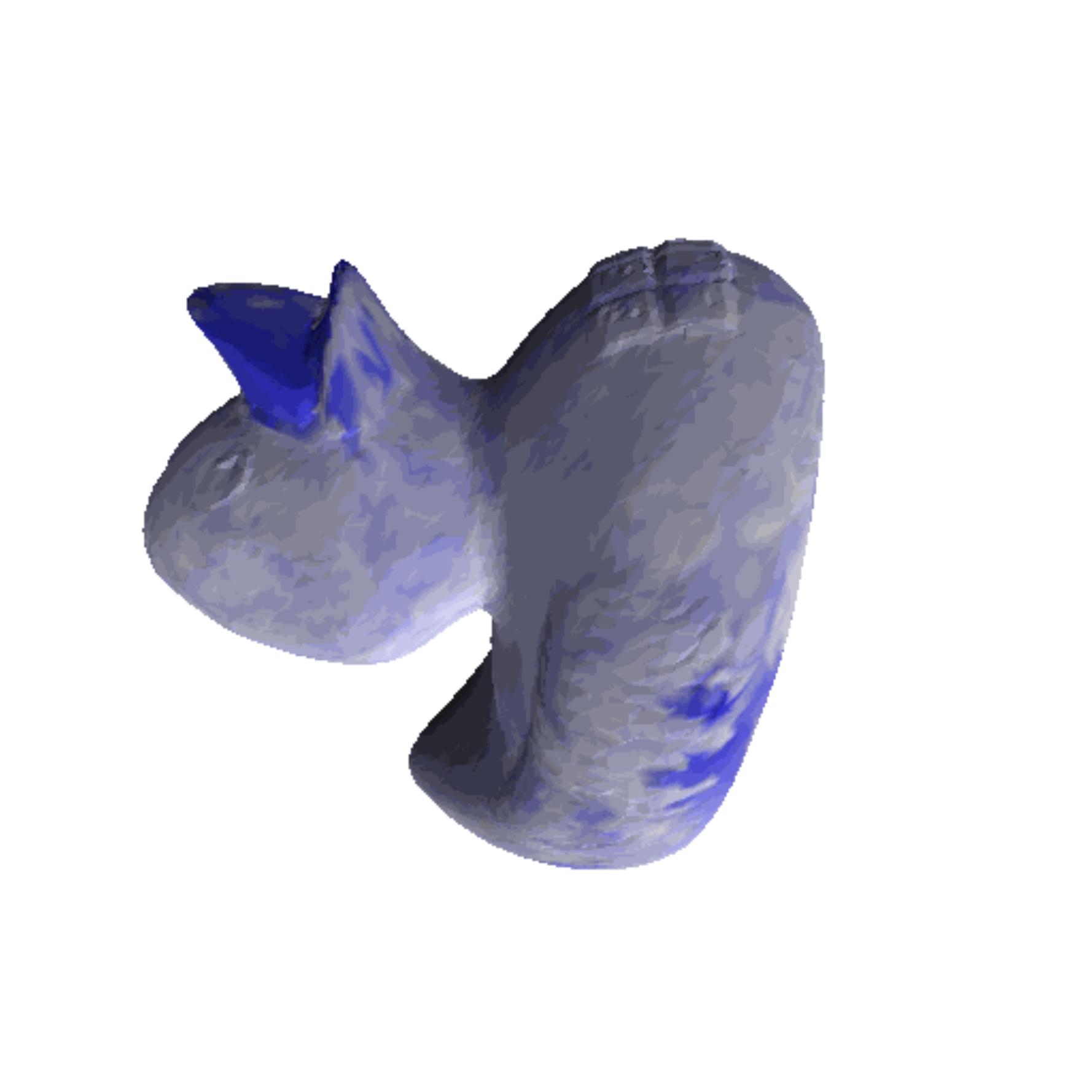}\vspace{-7mm}
         \caption{}\vspace{-1mm}
     \end{subfigure} 
     \\
     \begin{subfigure}[b]{0.16\textwidth}
         \centering
         \includegraphics[trim={2cm 2cm 2cm 2cm},clip,width=\textwidth]{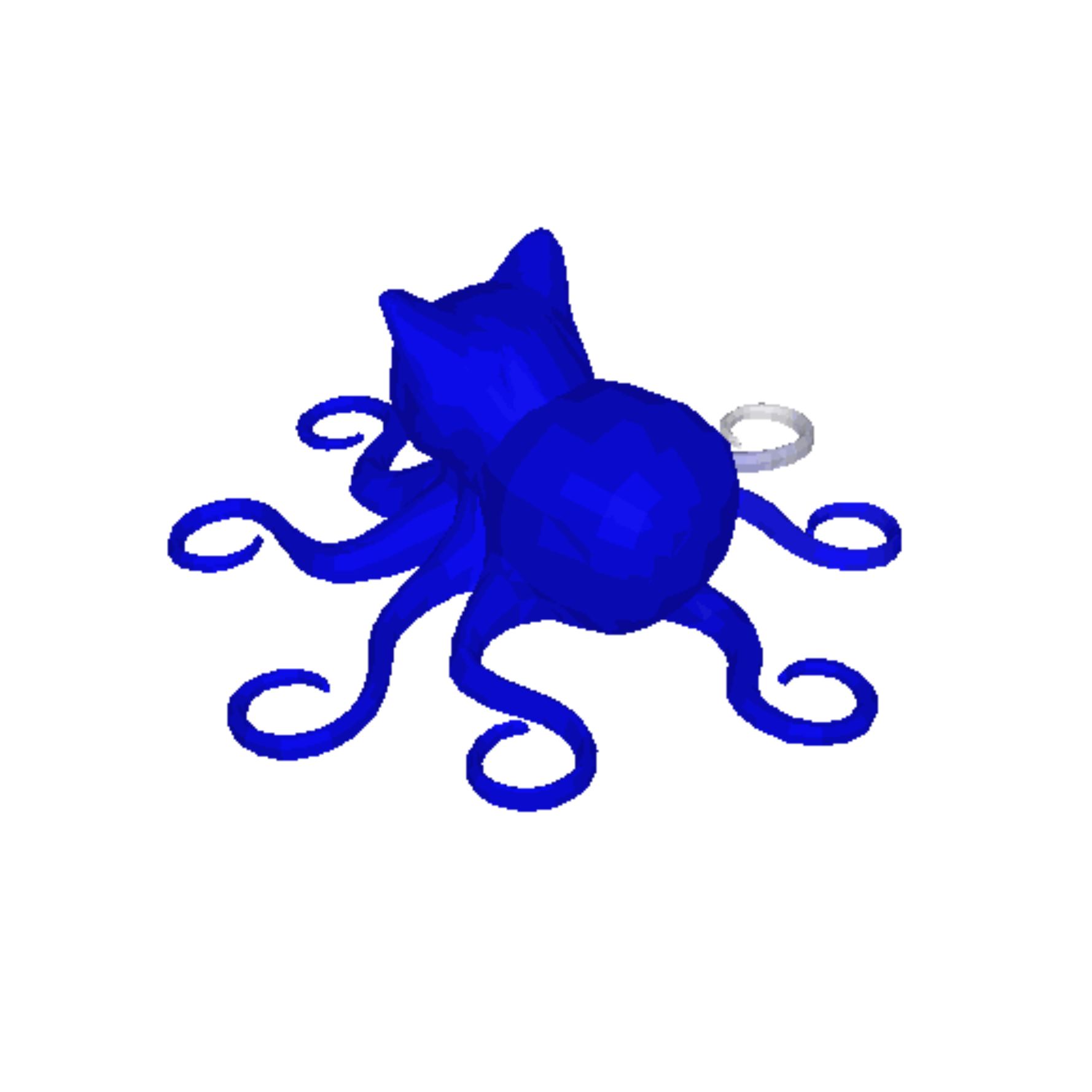}\vspace{-7mm}
         \caption{}\vspace{-5mm}
     \end{subfigure}
     \begin{subfigure}[b]{0.16\textwidth}
         \centering
         \includegraphics[trim={2cm 2cm 2cm 2cm},clip,width=\textwidth]{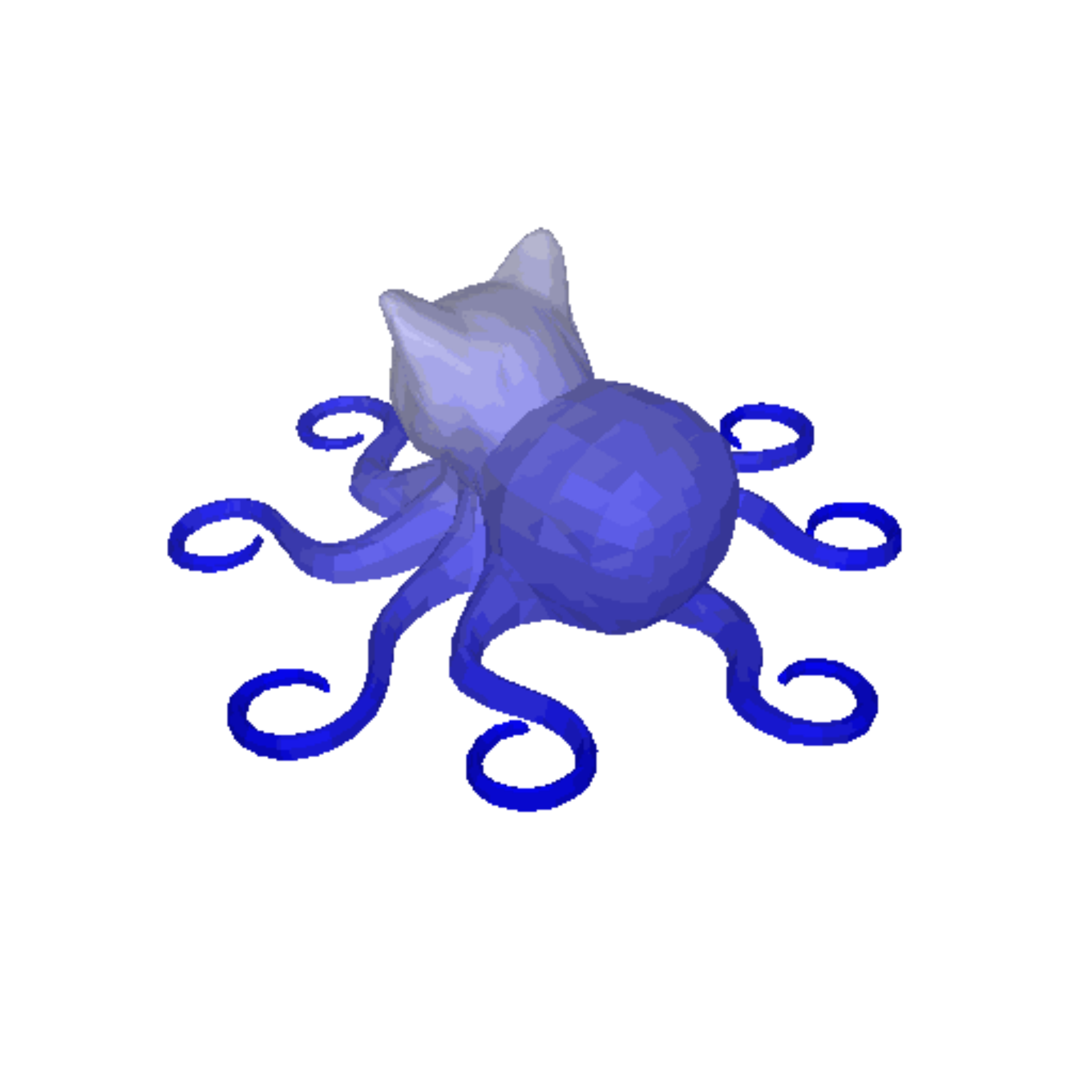}\vspace{-7mm}
         \caption{}\vspace{-5mm}
     \end{subfigure}
     \begin{subfigure}[b]{0.16\textwidth}
         \centering
         \includegraphics[trim={2cm 2cm 2cm 2cm},clip,width=\textwidth]{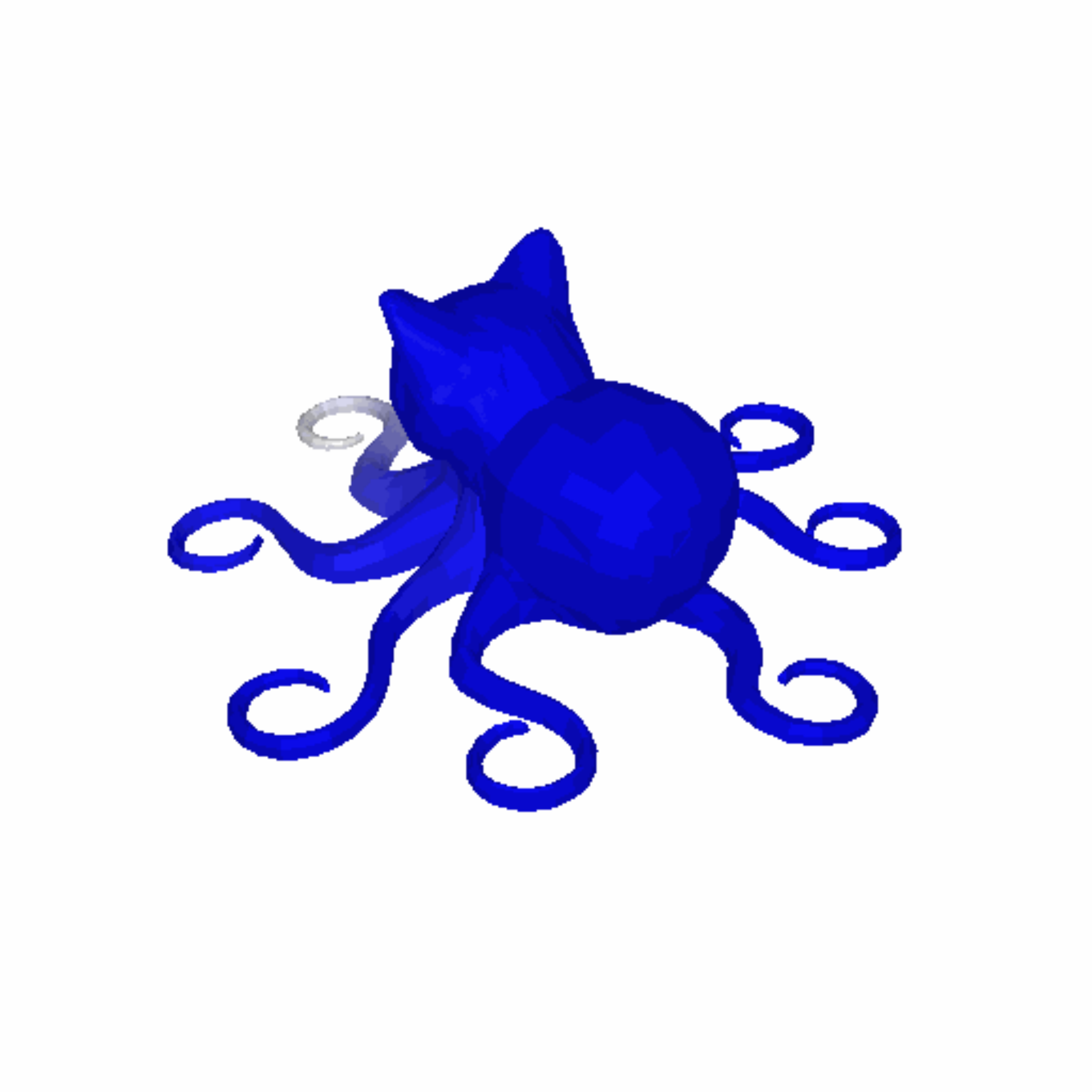}\vspace{-7mm}
         \caption{}\vspace{-5mm}
     \end{subfigure}
     \begin{subfigure}[b]{0.16\textwidth}
         \centering
         \includegraphics[trim={2cm 2cm 2cm 2cm},clip,width=\textwidth]{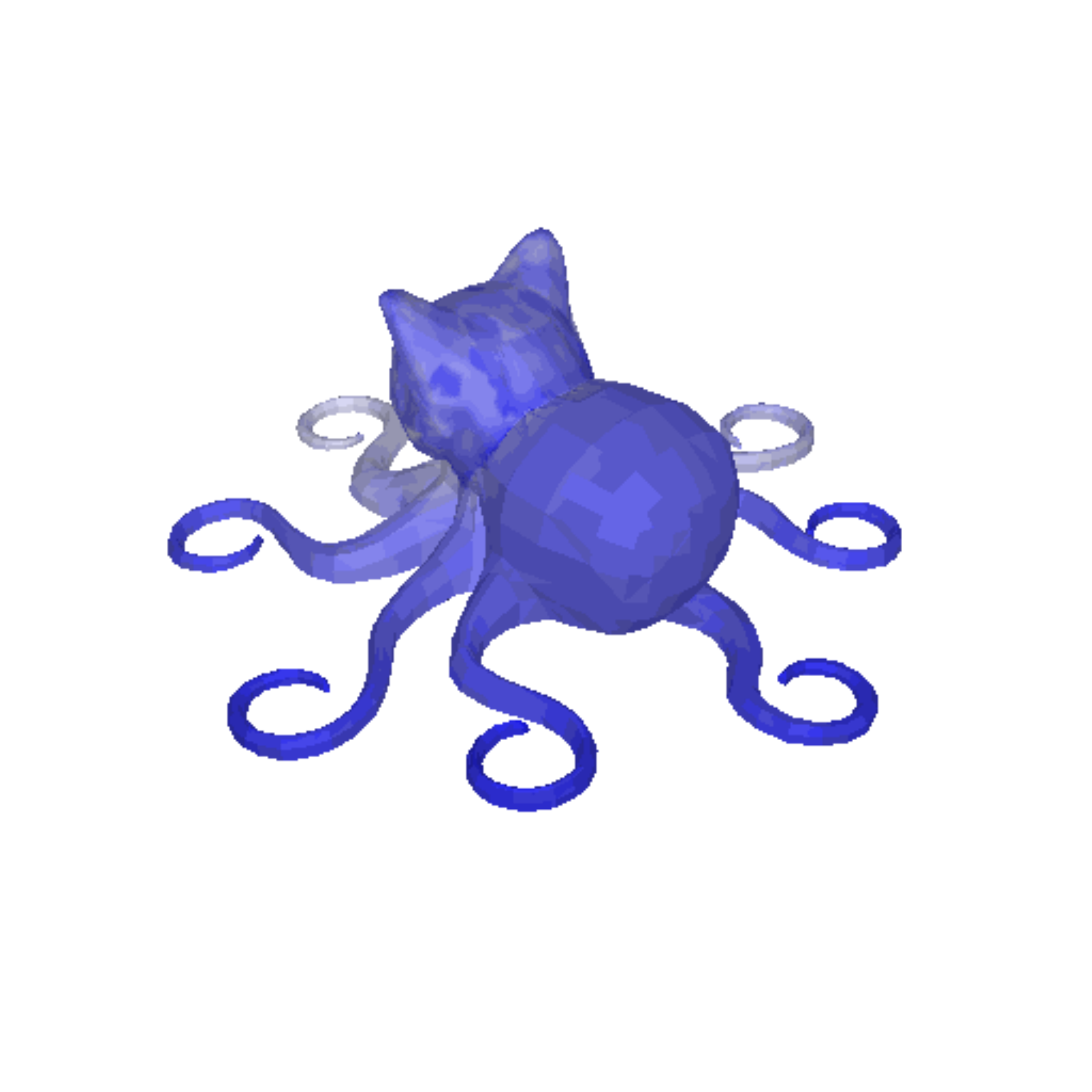}\vspace{-7mm}
         \caption{}\vspace{-5mm}
     \end{subfigure}
     \begin{subfigure}[b]{0.16\textwidth}
         \centering
         \includegraphics[trim={2cm 2cm 2cm 2cm},clip,width=\textwidth]{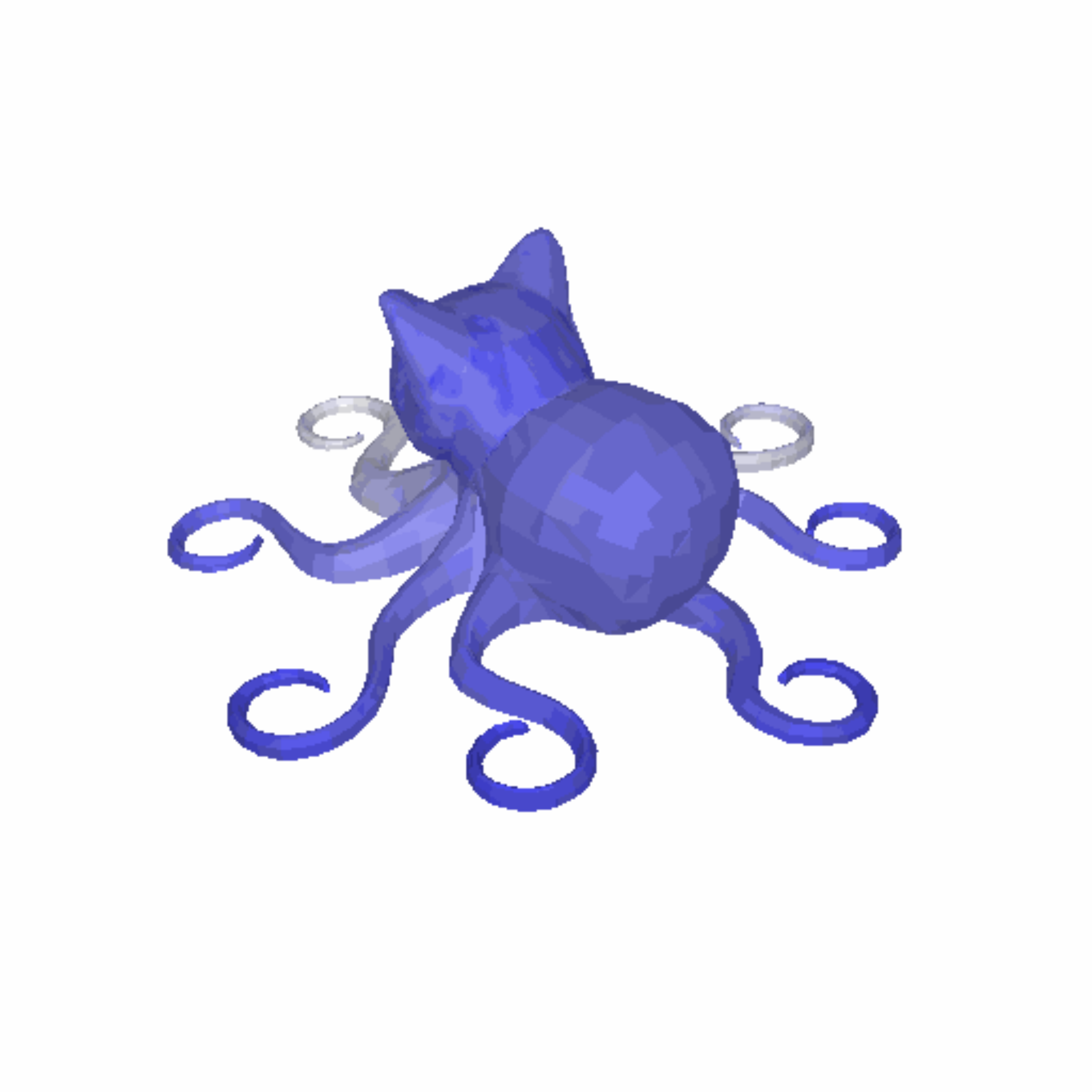}\vspace{-7mm}
         \caption{}
         \vspace{-5mm}
     \end{subfigure}
       \caption{\small{Comparison of the Wasserstein Barycenter output. \textbf{a-c, f-h}: three input distributions; \textbf{d,i}: Wasserstein Barycenter output with brute-force (BF); \textbf{e}: Wasserstein Barycenter output with \RFD}; \textbf{j}: Wasserstein Barycenter output with \SF. The top row is for the mesh \textit{duck} and the second row is for the mesh \textit{Octocat-v1}.}
\label{fig:conv_wass_separator}
\end{figure*}

\begin{figure*}[!htb] 
\centering
\includegraphics[width=1.5\columnwidth]{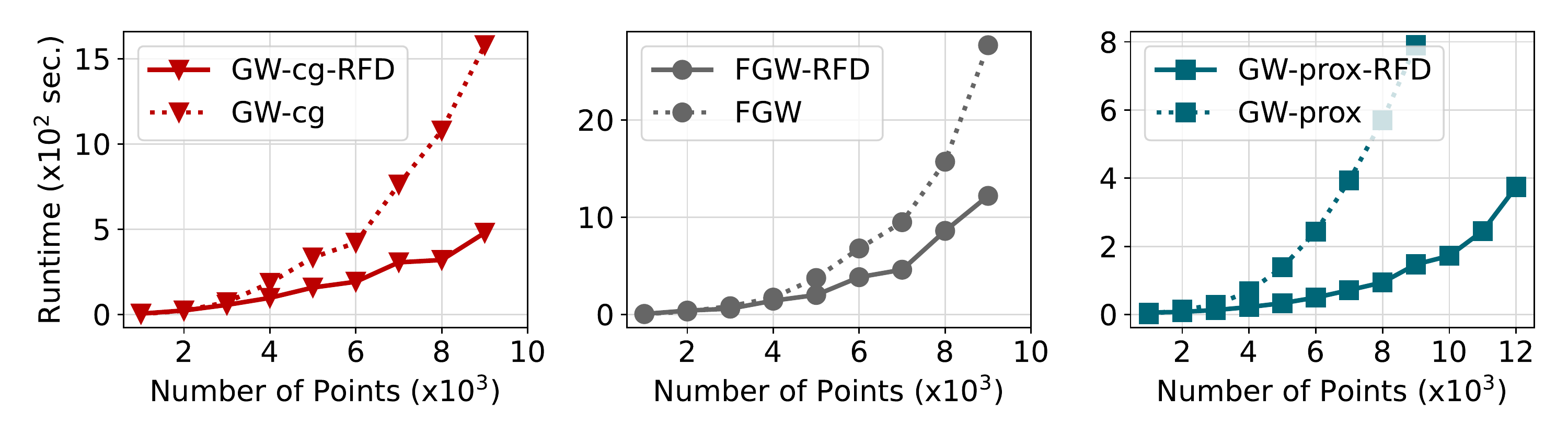}
\includegraphics[width=.5\columnwidth]{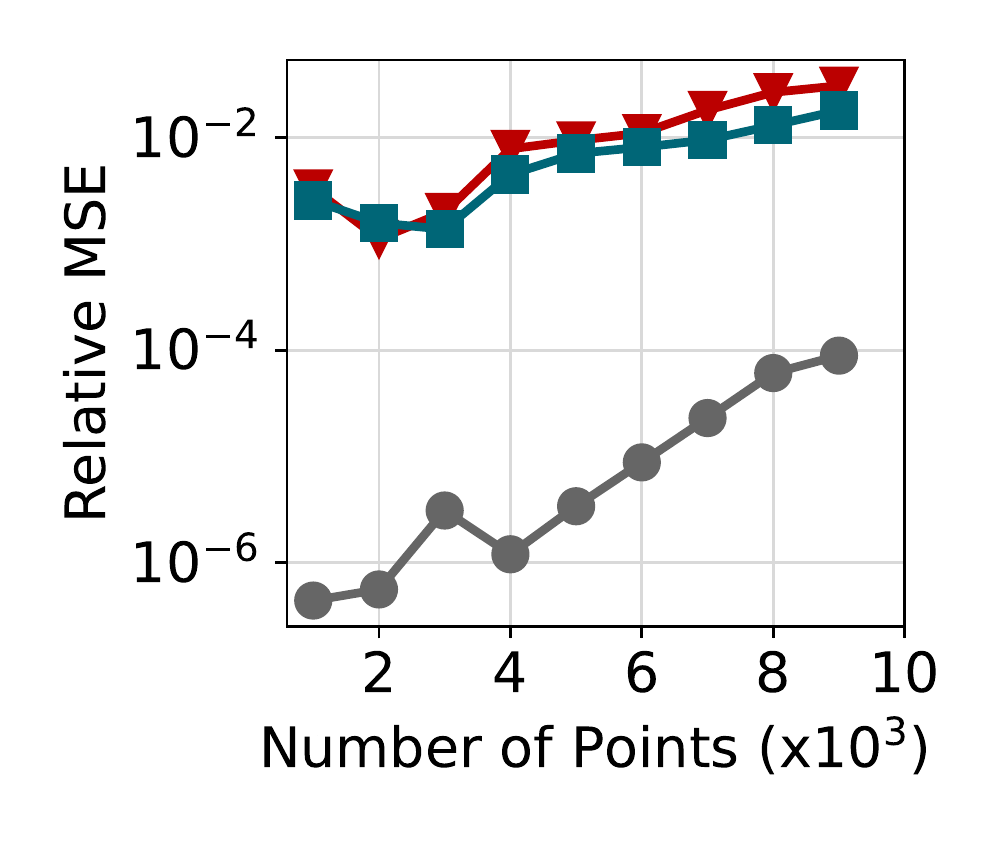}
\vspace{-5mm}
\caption{{\small \textbf{Left 3 figures}: Plots showing runtimes (in seconds) of \GW\ and \FGW\ vs our RFDiffusion injected counterparts. \textbf{Right figure}: The relative error of our RFDiffusion injected \textit{GW} variants. Except for the \GW-proximal-\RFD, all other variants are OOM after 9k points. For the \textit{FGW} experiments, random binary labels are generated for each node.}}\vspace{-5mm}
\label{fig:runtimes_gw}
\end{figure*}
\subsection{Wasserstein Distances and Barycenters}
\label{sec:wass_dist}
Optimal transport (OT) has found many applications in machine learning for its principled approach to compare distributions~\citep{cuturi2013}. There has been considerable work in extending OT problems to non-Euclidean domains like manifolds~\cite{solomon2015convolutional} and graph-structured data~\cite{memoli}. 
Our proposed methods can be easily integrated into several popular algorithms for computing Wasserstein distances. Here, we show the computational efficiency of our algorithms against well-known baselines.

\vspace{-2.8mm}
\paragraph{Wasserstein barycenter.}
In this section, we consider the OT problem of moving masses on a surface mesh, particularly the computation of Wasserstein barycenters. Since the geodesic distance on a surface is intractable, we use two approximations of this metric: 1) shortest-path distance (used in \SF~calculations), and 2) distance coming from an $\epsilon$-NN graph approximating the surface (\RFD~).  

Wasserstein barycenter is a weighted average of probability distributions. More formally,
given input distributions $\{ \boldsymbol{\mu}^i \}_{i=1}^k$ and a set of weights $\boldsymbol{\alpha}\in\mathbb{R}^k_+$, the Wasserstein barycenter is the solution to the following problem \cite{agueh2011barycenters}:
\vspace{-4mm}
\begin{equation*}
    \underset{{\boldsymbol{\mu} \in \operatorname{Prob}(\mathrm{V})}}{\operatorname{minimize}} \quad \sum_{i=1}^k \boldsymbol{\alpha}_i \mathcal{W}_2^2\left(\boldsymbol{\mu}, \boldsymbol{\mu}^i\right), \vspace{-2mm}
\end{equation*}
where $\mathcal{W}_2(\cdot,\cdot)$ denotes the 2-Wasserstein distance. 

Algorithm 2 outlined in~\cite{solomon2015convolutional} is used for the experiments in this subsection. We modify their algorithm to directly materialize and plug in our appropriate kernel matrix (which we refer to as the \BF~algorithm). More details about the baselines are provided in Appendix~\ref{sec:wass_bary_baselines}. We give a detailed description of the adaptation of our Fast Multiplication (FM) techniques to the entropic Wasserstein distances (see Appendix~\ref{sec:conv_wass_bary_appendix} and Algorithm~\ref{alg:wass_barycenter}). Our FM-infused variants significantly speed up the runtime of the baseline algorithm without losing accuracy (see Table~\ref{table:wass_bary_diffusion} and~\ref{table:wass_bary_separation}.)
Three different input probability distributions, encoded as vectors of length $N=|\mathrm{V}|$, are chosen for all our experiments. The barycenter output is also encoded as a vector of length $N$. 
Given an estimator $\hat{\boldsymbol{\mu}}$ and the ground truth $\boldsymbol{\mu}$, we measure the quality of the estimator using the mean squared error ($\mathrm{MSE}$) given by
$
\frac{1}{N}\sum_{j=1}^N \left(\hat{\boldsymbol{\mu}}_j - \boldsymbol{\mu}_j \right)^2.
$

We visualize the Wasserstein barycenter generated by \RFD\ and \SF\ with their corresponding ground truth (generated by the baseline method) in Fig.~\ref{fig:conv_wass_separator}.
Note that the barycenters generated by our integrators are similar to the ground truth.

Low-distortion trees do not scale to mesh sizes considered here. We provide additional comparisons of our method with~\citep{solomon2015convolutional} in Appendix (Table~\ref{table:wass_bary_diffusion_with_solomon}). However, we note that the results are not directly comparable as the kernels employed by the authors are different from ours.
\vspace{-3.5mm}
\paragraph{Gromov Wasserstein and Fused Gromov Wasserstein distances.}
Gromov Wasserstein (\GW~)~discrepancy (resp. Fused Gromov Wasserstein discrepancy (\FGW~)) is an extension of Wasserstein distances to graph-structured data (resp. labeled graph-structured data) 
with widespread applications in a range of tasks including clustering, classification and generation~\citep{peyre:hal-01322992, memoli, Demetci2020.SCOT,Memoli2006, fgw, pmlr-v97-titouan19a}. Despite their widespread use, \GW~and \FGW~discrepancies are very expensive to compute.

The \GW\ discrepancy can be calculated iteratively by the conditional gradient method~\citep{peyre:hal-01322992}, which we refer to as \GW-cg or the proximal point algorithm~\citep{pmlr-v97-xu19b}, resp. \GW-prox. A key component in solving this OT problem by either method involves the computation of a tensor product, which is expensive. Our fast GFI methods can be used to  estimate this tensor product efficiently (Appendix Algorithm~\ref{alg:tensor_prod}), thus speeding up the runtime of the entire algorithm. Moreover, we can also effectively estimate the step size of the \FGW\ iterations in a line search algorithm (Appendix Algorithm~\ref{alg:line_search}). More details are presented in Appendix~\ref{sec:gw_fgw}. 

Since this task is particularly challenging from the computational standpoint, we choose our fastest \RFD\ algorithm. 
 Our methods (\GW-\RFD, \FGW-\RFD\, and \GW-prox-\RFD) run consistently $2$-$4$x faster than the baseline methods, with only a small drop in accuracy in computing the associated costs (Fig.~\ref{fig:runtimes_gw}). The  plots shown are obtained by averaging over $10$ seeds and random $3$-D distributions. For all our experiments, $m=16$ random features, $\epsilon=0.3$, and the smoothing factor $\lambda=-0.2$ are chosen. For the baseline experiments, we use the implementation from the POT library~\citep{flamary2021pot} for the \GW-cg and \FGW~variants, and official implementation from~\citep{pmlr-v97-xu19b} for the \GW-prox variant. Additional experiments (with ablation studies on the hyperparameters) are in Sec.~\ref{sec:gw_bary} (resp.~\ref{sec:gw_ablates}). 

\subsection{Experiments on Point Cloud Classification}
In this subsection, we demonstrate the effectiveness of the \RFD\ kernel for various point cloud classification tasks.
\paragraph{Topological Transformers.} We present additional experiments with results on Point Cloud Transformers (PCT)~\citep{Guo_2021}. The entrance point for the $\mathrm{RFDiffusion}$ algorithm is the topologically-modulated performized version~\citep{choromanski} of the regular PCT. The topological modulation works by Hadamard-multiplying regular attention matrix with the mask-matrix encoding relative distances between the points in the $3D$ space to directly impose structural priors while training the attention model. Performized PCT provides computational gains for larger point clouds ($N=2048$ points are used in our experiments). Moreover, its topologically modulated version can be executed in the favorable sub-quadratic time only if the mask-matrix itself supports sub-quadratic matrix-vector multiplication~\citep{topmasking} without the explicit materialization of the attention and the mask matrices. $\mathrm{RFDiffusion}$ provides a low-rank decomposition via its novel RF-mechanism and the observation in Sec. 3.4 (~\citep{topmasking}), can be used for time-efficient training in our particular setting. We conduct our experiments on the ShapeNet dataset~\citep{Zhirong15CVPR}. Performized PCT with efficient $\mathrm{RFDiffusion}$-driven masking achieves $\mathbf{91.13\%}$ validation accuracy and linear time complexity (due to the efficient integration algorithm with the $\mathrm{RFDiffusion}$). The brute-force variant runs out of memory in training.

\paragraph{Point Cloud Classification.}~\label{sec:pc_expt} We have also conducted point cloud classification experiments on ModelNet10~\citep{Zhirong15CVPR} and Cubes~\citep{meshcnn}, using our $\mathrm{RFDiffusion}$ kernel method. The classification in this case is conducted using the eigenvectors of the kernel matrix. Note that, as described in~\citep{nakatsukasa2019lowrank}, low-rank decomposition of the kernel matrix (provided directly by the $\mathrm{RFDiffusion}$ method via the random feature map mechanism) can be used to compute efficiently eigenvectors and the corresponding eigenvalues. For each dataset, we compute the $k$ smallest eigenvalues of the kernel matrix ($k = 32$ for ModelNet10 and $k = 16$ for Cubes). We pass these $k$ eigenvalues to a random forest classifier for downstream classification.  For all the experiments we use:  $\epsilon = .1$, $\lambda = -.1$ and we sample $2048$ points randomly for each shape in ModelNet10. 

The brute-force baseline version for the ModelNet10 and Cubes explicitly constructs the epsilon-neighborhood graph, directly conducting the eigendecomposition of its adjacency matrix and exponentiating eigenvalues. Comparison with this variant is the most accurate apple-to-apple comparison. The baseline variant has time complexity $O(N^3)$ whereas our method for obtaining the eigenvectors is of time complexity $O(N)$. Our results are summarized in Table~\ref{tab:pc_classification}. Our method excels at these point cloud classification tasks, beating the brute-force  baseline method by almost $\mathbf{25}$ points on ModelNet10 and by $\mathbf{5}$ points on the Cubes. Our reported numbers are comparable to earlier methods on ModelNet (SPH and LFD achieving $79\%$~\citep{Zhirong15CVPR}). Cubes is a fairly challenging dataset and deep learning models like PointNet achieves only $55\%$ accuracy. 

\begin{table}[h]
\caption{Point cloud classification using \RFD\ Kernel}
\label{tab:pc_classification}
\resizebox{\columnwidth}{!}{%
\begin{tabular}{@{}lccccc@{}}
\toprule
\textbf{Dataset} & \textbf{\# Graphs}  & \textbf{\# Classes} & \textbf{Baseline} & \textbf{RFD} \\
\midrule
ModelNet10       & 3991/908                   & 10                  & 43.0                & $\mathbf{70.1}$       \\
Cubes            & 3759/659                   & 23                  & 39.3                & $\mathbf{44.6}$   \\
\bottomrule
\end{tabular}%
}
\end{table}
For additional experiments on graph classification, see Appendix~\ref{sec:graph_classification}.
\vspace{-1mm}
\section{Conclusion}~\label{sec:conclusion}
We have presented in this paper two algorithms, $\mathrm{SeparatorFactorization}$ and $\mathrm{RFDiffusion}$, for efficient graph field integration based on the theory of balanced separators and Fourier analysis. As a byproduct of the developed techniques, we have obtained new results in structural graph theory. Our extensive empirical studies support our theoretical findings (e.g., mesh dynamics modeling) involving interpolation on meshes for rigid and deformable objects and the computation of the Wasserstein distance between distributions defined on meshes. 


\vspace{-1mm}
\section{Acknowledgement}~\label{sec:acknowledgement}
AW acknowledges support from a Turing AI Fellowship under grant EP/V025279/1 and the Leverhulme Trust via CFI.

\bibliography{references}

\begin{thebibliography}{67}
\providecommand{\natexlab}[1]{#1}
\providecommand{\url}[1]{\texttt{#1}}
\expandafter\ifx\csname urlstyle\endcsname\relax
  \providecommand{\doi}[1]{doi: #1}\else
  \providecommand{\doi}{doi: \begingroup \urlstyle{rm}\Url}\fi

\bibitem[Abraham et~al.(2008)Abraham, Bartal, and Neiman]{ld-3}
Abraham, I., Bartal, Y., and Neiman, O.
\newblock Nearly tight low stretch spanning trees.
\newblock In \emph{49th Annual {IEEE} Symposium on Foundations of Computer
  Science, {FOCS} 2008, October 25-28, 2008, Philadelphia, PA, {USA}}, pp.\
  781--790. {IEEE} Computer Society, 2008.

\bibitem[Agueh \& Carlier(2011)Agueh and Carlier]{agueh2011barycenters}
Agueh, M. and Carlier, G.
\newblock Barycenters in the {W}asserstein space.
\newblock \emph{SIAM Journal on Mathematical Analysis}, 43\penalty0
  (2):\penalty0 904--924, 2011.

\bibitem[Al-Mohy \& Higham(2010)Al-Mohy and Higham]{al2010new}
Al-Mohy, A.~H. and Higham, N.~J.
\newblock A new scaling and squaring algorithm for the matrix exponential.
\newblock \emph{SIAM Journal on Matrix Analysis and Applications}, 31\penalty0
  (3):\penalty0 970--989, 2010.

\bibitem[Al{-}Mohy \& Higham(2011)Al{-}Mohy and Higham]{exp-1}
Al{-}Mohy, A.~H. and Higham, N.~J.
\newblock Computing the action of the matrix exponential, with an application
  to exponential integrators.
\newblock \emph{{SIAM} J. Sci. Comput.}, 33\penalty0 (2):\penalty0 488--511,
  2011.

\bibitem[Avron et~al.(2017)Avron, Kapralov, Musco, Musco, Velingker, and
  Zandieh]{rf-3}
Avron, H., Kapralov, M., Musco, C., Musco, C., Velingker, A., and Zandieh, A.
\newblock Random {F}ourier features for kernel ridge regression: Approximation
  bounds and statistical guarantees.
\newblock In \emph{Proceedings of the 34th International Conference on Machine
  Learning, {ICML}}, 2017.

\bibitem[Bader et~al.(2019)Bader, Blanes, and Casas]{bader2019computing}
Bader, P., Blanes, S., and Casas, F.
\newblock Computing the matrix exponential with an optimized taylor polynomial
  approximation.
\newblock \emph{Mathematics}, 7\penalty0 (12):\penalty0 1174, 2019.

\bibitem[Balcilar et~al.(2020)Balcilar, Renton, Heroux, Gauzere, Adam, and
  Honeine]{balcilar2020bridging}
Balcilar, M., Renton, G., Heroux, P., Gauzere, B., Adam, S., and Honeine, P.
\newblock Bridging the gap between spectral and spatial domains in graph neural
  networks, 2020.

\bibitem[Bartal(1996)]{bartal1996probabilistic}
Bartal, Y.
\newblock Probabilistic approximation of metric spaces and its algorithmic
  applications.
\newblock In \emph{Proceedings of 37th Conference on Foundations of Computer
  Science}, pp.\  184--193. IEEE, 1996.

\bibitem[Bartal et~al.(2022)Bartal, Fandina, and Neiman]{ld-2}
Bartal, Y., Fandina, O.~N., and Neiman, O.
\newblock Covering metric spaces by few trees.
\newblock \emph{J. Comput. Syst. Sci.}, 130:\penalty0 26--42, 2022.

\bibitem[Barth(1998)]{fluid}
Barth, T.~J.
\newblock Computational fluid dynamics, structural analysis and mesh
  partitioning techniques - introduction.
\newblock In Palma, J. M. L.~M., Dongarra, J.~J., and Hern{\'{a}}ndez, V.
  (eds.), \emph{Vector and Parallel Processing - {VECPAR} '98, Third
  International Conference, Porto, Portugal, June 21-23, 1998, Selected Papers
  and Invited Talks}, volume 1573 of \emph{Lecture Notes in Computer Science},
  pp.\  171--175. Springer, 1998.

\bibitem[Bartolucci et~al.(2020)Bartolucci, Caccioli, Caravelli, and
  Vivo]{leontief}
Bartolucci, S., Caccioli, F., Caravelli, F., and Vivo, P.
\newblock Inversion-free {L}eontief inverse: statistical regularities in
  input-output analysis from partial information, 2020.
\newblock URL \url{https://arxiv.org/abs/2009.06350}.

\bibitem[Bodlaender et~al.(2016)Bodlaender, Drange, Dregi, Fomin, Lokshtanov,
  and Pilipczuk]{pilipczuk-2}
Bodlaender, H.~L., Drange, P.~G., Dregi, M.~S., Fomin, F.~V., Lokshtanov, D.,
  and Pilipczuk, M.
\newblock A c\({}^{\mbox{k}}\) n 5-approximation algorithm for treewidth.
\newblock \emph{{SIAM} J. Comput.}, 45\penalty0 (2):\penalty0 317--378, 2016.
\newblock \doi{10.1137/130947374}.
\newblock URL \url{https://doi.org/10.1137/130947374}.

\bibitem[Charikar et~al.(1998)Charikar, Chekuri, Goel, Guha, and
  Plotkin]{charikar}
Charikar, M., Chekuri, C., Goel, A., Guha, S., and Plotkin, S.~A.
\newblock Approximating a finite metric by a small number of tree metrics.
\newblock In \emph{39th Annual Symposium on Foundations of Computer Science,
  {FOCS} '98, November 8-11, 1998, Palo Alto, California, {USA}}, pp.\
  379--388. {IEEE} Computer Society, 1998.

\bibitem[Cheng et~al.(2006)Cheng, Crutchfield, Gimbutas, Greengard, Ethridge,
  Huang, Rokhlin, Yarvin, and Zhao]{FMM-6}
Cheng, H., Crutchfield, W.~Y., Gimbutas, Z., Greengard, L., Ethridge, J.~F.,
  Huang, J., Rokhlin, V., Yarvin, N., and Zhao, J.
\newblock A wideband fast multipole method for the {H}elmholtz equation in
  three dimensions.
\newblock \emph{J. Comput. Phys.}, 216\penalty0 (1):\penalty0 300--325, 2006.

\bibitem[Choromanski et~al.(2022)Choromanski, Lin, Chen, Zhang, Sehanobish,
  Likhosherstov, Parker{-}Holder, Sarl{\'{o}}s, Weller, and
  Weingarten]{topmasking}
Choromanski, K., Lin, H., Chen, H., Zhang, T., Sehanobish, A., Likhosherstov,
  V., Parker{-}Holder, J., Sarl{\'{o}}s, T., Weller, A., and Weingarten, T.
\newblock From block-{T}oeplitz matrices to differential equations on graphs:
  towards a general theory for scalable masked transformers.
\newblock In \emph{International Conference on Machine Learning, {ICML} 2022,
  17-23 July 2022, Baltimore, Maryland, {USA}}, volume 162 of \emph{Proceedings
  of Machine Learning Research}, pp.\  3962--3983. {PMLR}, 2022.

\bibitem[Choromanski et~al.(2021)Choromanski, Likhosherstov, Dohan, Song, Gane,
  Sarl{\'{o}}s, Hawkins, Davis, Mohiuddin, Kaiser, Belanger, Colwell, and
  Weller]{choromanski}
Choromanski, K.~M., Likhosherstov, V., Dohan, D., Song, X., Gane, A.,
  Sarl{\'{o}}s, T., Hawkins, P., Davis, J.~Q., Mohiuddin, A., Kaiser, L.,
  Belanger, D.~B., Colwell, L.~J., and Weller, A.
\newblock Rethinking attention with performers.
\newblock In \emph{9th International Conference on Learning Representations,
  {ICLR} 2021, Virtual Event, Austria, May 3-7, 2021}. OpenReview.net, 2021.

\bibitem[Cuturi(2013)]{cuturi2013}
Cuturi, M.
\newblock Sinkhorn distances: Lightspeed computation of optimal transport.
\newblock In \emph{Proceedings of the 26th International Conference on Neural
  Information Processing Systems - Volume 2}, NIPS'13, pp.\  2292–2300, Red
  Hook, NY, USA, 2013. Curran Associates Inc.

\bibitem[Cygan et~al.(2015)Cygan, Fomin, Kowalik, Lokshtanov, Marx, Pilipczuk,
  Pilipczuk, and Saurabh]{pilipczuk}
Cygan, M., Fomin, F.~V., Kowalik, L., Lokshtanov, D., Marx, D., Pilipczuk, M.,
  Pilipczuk, M., and Saurabh, S.
\newblock \emph{Parameterized Algorithms}.
\newblock Springer, 2015.

\bibitem[Darve \& Have(2004)Darve and Have]{FMM-9}
Darve, E. and Have, P.
\newblock A fast multipole method for {M}axwell equations stable at all
  frequencies.
\newblock \emph{Philos Trans A Math Phys Eng Sci.}, 15:\penalty0 603--28, 2004.
\newblock \doi{10.1098/rsta.2003.1337}.

\bibitem[Dattoli et~al.(2004)Dattoli, Ricci, and Pacciani]{bessel}
Dattoli, G., Ricci, P.~E., and Pacciani, P.
\newblock Comments on the theory of {B}essel functions with more than one
  index.
\newblock \emph{Appl. Math. Comput.}, 150\penalty0 (3):\penalty0 603--610,
  2004.

\bibitem[de~Lara \& Pineau(2018)de~Lara and Pineau]{de2018simple}
de~Lara, N. and Pineau, E.
\newblock A simple baseline algorithm for graph classification.
\newblock \emph{Relational Representation Learning Workshop, NIPS 2018}, 2018.

\bibitem[Demetci et~al.(2020)Demetci, Santorella, Sandstede, Noble, and
  Singh]{Demetci2020.SCOT}
Demetci, P., Santorella, R., Sandstede, B., Noble, W.~S., and Singh, R.
\newblock Gromov-{W}asserstein optimal transport to align single-cell
  multi-omics data.
\newblock \emph{bioRxiv}, 2020.
\newblock \doi{10.1101/2020.04.28.066787}.

\bibitem[Diestel \& M{\"{u}}ller(2012)Diestel and M{\"{u}}ller]{diestel}
Diestel, R. and M{\"{u}}ller, M.
\newblock Connected tree-width.
\newblock \emph{CoRR}, abs/1211.7353, 2012.

\bibitem[Fakcharoenphol et~al.(2004)Fakcharoenphol, Rao, and
  Talwar]{fakcharoenphol2004tight}
Fakcharoenphol, J., Rao, S., and Talwar, K.
\newblock A tight bound on approximating arbitrary metrics by tree metrics.
\newblock \emph{Journal of Computer and System Sciences}, 69\penalty0
  (3):\penalty0 485--497, 2004.

\bibitem[Fellows et~al.(2008)Fellows, Fomin, Lokshtanov, Losievskaja, Rosamond,
  and Saurabh]{ld-tree}
Fellows, M.~R., Fomin, F.~V., Lokshtanov, D., Losievskaja, E., Rosamond, F.~A.,
  and Saurabh, S.
\newblock Parameterized low-distortion embeddings - graph metrics into lines
  and trees.
\newblock \emph{CoRR}, abs/0804.3028, 2008.

\bibitem[Flamary et~al.(2021)Flamary, Courty, Gramfort, Alaya, Boisbunon,
  Chambon, Chapel, Corenflos, Fatras, Fournier, Gautheron, Gayraud, Janati,
  Rakotomamonjy, Redko, Rolet, Schutz, Seguy, Sutherland, Tavenard, Tong, and
  Vayer]{flamary2021pot}
Flamary, R., Courty, N., Gramfort, A., Alaya, M.~Z., Boisbunon, A., Chambon,
  S., Chapel, L., Corenflos, A., Fatras, K., Fournier, N., Gautheron, L.,
  Gayraud, N.~T., Janati, H., Rakotomamonjy, A., Redko, I., Rolet, A., Schutz,
  A., Seguy, V., Sutherland, D.~J., Tavenard, R., Tong, A., and Vayer, T.
\newblock Pot: Python optimal transport.
\newblock \emph{Journal of Machine Learning Research}, 22\penalty0
  (78):\penalty0 1--8, 2021.

\bibitem[Gilbert et~al.(1984)Gilbert, Hutchinson, and Tarjan]{genus-2}
Gilbert, J.~R., Hutchinson, J.~P., and Tarjan, R.~E.
\newblock A separator theorem for graphs of bounded genus.
\newblock \emph{J. Algorithms}, 5\penalty0 (3):\penalty0 391--407, 1984.

\bibitem[Gimbutas \& Rokhlin(2003)Gimbutas and Rokhlin]{FMM-3}
Gimbutas, Z. and Rokhlin, V.
\newblock A generalized fast multipole method for nonoscillatory kernels.
\newblock \emph{{SIAM} J. Sci. Comput.}, 24\penalty0 (3):\penalty0 796--817,
  2003.

\bibitem[Greengard \& Rokhlin(1987)Greengard and Rokhlin]{FMM}
Greengard, L. and Rokhlin, V.
\newblock A fast algorithm for particle simulations.
\newblock \emph{Journal of Computational Physics}, 73:\penalty0 325--348, 1987.

\bibitem[Greengard et~al.(2021)Greengard, O'Neil, Rachh, and Vico]{FMM-5}
Greengard, L., O'Neil, M., Rachh, M., and Vico, F.
\newblock Fast multipole methods for the evaluation of layer potentials with
  locally-corrected quadratures.
\newblock \emph{J. Comput. Phys. {X}}, 10:\penalty0 100092, 2021.

\bibitem[Gumerov \& Duraiswami(2021)Gumerov and Duraiswami]{FMM-10}
Gumerov, N.~A. and Duraiswami, R.
\newblock Fast multipole accelerated boundary element methods for room
  acoustics.
\newblock \emph{CoRR}, abs/2103.16073, 2021.

\bibitem[Guo et~al.(2021)Guo, Cai, Liu, Mu, Martin, and Hu]{Guo_2021}
Guo, M.-H., Cai, J.-X., Liu, Z.-N., Mu, T.-J., Martin, R.~R., and Hu, S.-M.
\newblock Pct: Point cloud transformer.
\newblock \emph{Computational Visual Media}, 7\penalty0 (2):\penalty0
  187–199, Apr 2021.
\newblock ISSN 2096-0662.
\newblock \doi{10.1007/s41095-021-0229-5}.
\newblock URL \url{http://dx.doi.org/10.1007/s41095-021-0229-5}.

\bibitem[Han et~al.(2022)Han, Gao, Pfaff, Wang, and Liu]{mesh-net-v2}
Han, X., Gao, H., Pfaff, T., Wang, J., and Liu, L.
\newblock Predicting physics in mesh-reduced space with temporal attention.
\newblock In \emph{The Tenth International Conference on Learning
  Representations, {ICLR} 2022, Virtual Event, April 25-29, 2022}.
  OpenReview.net, 2022.

\bibitem[Hanocka et~al.(2019)Hanocka, Hertz, Fish, Giryes, Fleishman, and
  Cohen-Or]{meshcnn}
Hanocka, R., Hertz, A., Fish, N., Giryes, R., Fleishman, S., and Cohen-Or, D.
\newblock Meshcnn: A network with an edge.
\newblock \emph{ACM Trans. Graph.}, 38\penalty0 (4), jul 2019.
\newblock ISSN 0730-0301.
\newblock \doi{10.1145/3306346.3322959}.
\newblock URL \url{https://doi.org/10.1145/3306346.3322959}.

\bibitem[Ishiyama et~al.(2022)Ishiyama, Yoshikawa, and Tanikawa]{nbody}
Ishiyama, T., Yoshikawa, K., and Tanikawa, A.
\newblock High performance gravitational n-body simulations on supercomputer
  fugaku.
\newblock In \emph{{HPC} Asia 2022: International Conference on High
  Performance Computing in Asia-Pacific Region, Virtual Event, Japan, January
  12 - 14, 2022}, pp.\  10--17. {ACM}, 2022.

\bibitem[Janati et~al.(2020)Janati, Cuturi, and Gramfort]{janati:hal-03063875}
Janati, H., Cuturi, M., and Gramfort, A.
\newblock {Debiased Sinkhorn barycenters}.
\newblock In \emph{{ICML 2020 - 37th International Conference on Machine
  Learning}}, Vienna / Virtuel, Austria, July 2020.
\newblock URL \url{https://hal.science/hal-03063875}.

\bibitem[Li \& Chang(2014)Li and Chang]{dot-product-graphs}
Li, B. and Chang, G.~J.
\newblock Dot product dimensions of graphs.
\newblock \emph{Discret. Appl. Math.}, 166:\penalty0 159--163, 2014.

\bibitem[Liska \& Colonius(2014)Liska and Colonius]{FMM-7}
Liska, S. and Colonius, T.
\newblock A parallel fast multipole method for elliptic difference equations.
\newblock \emph{J. Comput. Phys.}, 278:\penalty0 76--91, 2014.

\bibitem[Lozano{-}Dur{\'{a}}n \& Borrell(2016)Lozano{-}Dur{\'{a}}n and
  Borrell]{genus}
Lozano{-}Dur{\'{a}}n, A. and Borrell, G.
\newblock Algorithm 964: An efficient algorithm to compute the genus of
  discrete surfaces and applications to turbulent flows.
\newblock \emph{{ACM} Trans. Math. Softw.}, 42\penalty0 (4):\penalty0
  34:1--34:19, 2016.

\bibitem[M\'{e}moli(2011)]{memoli}
M\'{e}moli, F.
\newblock Gromov---{W}asserstein distances and the metric approach to object
  matching.
\newblock \emph{Found. Comput. Math.}, 11\penalty0 (4):\penalty0 417–487, aug
  2011.
\newblock ISSN 1615-3375.

\bibitem[M{\'e}moli \& Sapiro(2006)M{\'e}moli and Sapiro]{Memoli2006}
M{\'e}moli, F. and Sapiro, G.
\newblock Computing with point cloud data.
\newblock In Krim, H. and Yezzi, A. (eds.), \emph{Statistics and Analysis of
  Shapes}, pp.\  201--229, Boston, MA, 2006. Birkh{\"a}user Boston.

\bibitem[M{\"{o}}ller et~al.(2019)M{\"{o}}ller, Petit, Carayol, Dinh, and
  Jalby]{FMM-2}
M{\"{o}}ller, N., Petit, E., Carayol, Q., Dinh, Q., and Jalby, W.
\newblock Scalable fast multipole method for electromagnetic simulations.
\newblock In \emph{Computational Science - {ICCS} 2019 - 19th International
  Conference, Faro, Portugal, June 12-14, 2019, Proceedings, Part {II}}, volume
  11537 of \emph{Lecture Notes in Computer Science}, pp.\  663--676. Springer,
  2019.

\bibitem[Morris et~al.(2020)Morris, Kriege, Bause, Kersting, Mutzel, and
  Neumann]{Morris2020}
Morris, C., Kriege, N.~M., Bause, F., Kersting, K., Mutzel, P., and Neumann, M.
\newblock Tudataset: A collection of benchmark datasets for learning with
  graphs.
\newblock In \emph{ICML 2020 Workshop on Graph Representation Learning and
  Beyond (GRL+ 2020)}, 2020.
\newblock URL \url{www.graphlearning.io}.

\bibitem[Mory et~al.(2009)Mory, Ardon, Yezzi, and Thiran]{mory}
Mory, B., Ardon, R., Yezzi, A.~J., and Thiran, J.
\newblock Non-euclidean image-adaptive radial basis functions for 3d
  interactive segmentation.
\newblock In \emph{{IEEE} 12th International Conference on Computer Vision,
  {ICCV} 2009, Kyoto, Japan, September 27 - October 4, 2009}, pp.\  787--794.
  {IEEE} Computer Society, 2009.

\bibitem[Musco et~al.(2018)Musco, Musco, and Sidford]{exp-2}
Musco, C., Musco, C., and Sidford, A.
\newblock Stability of the {L}anczos method for matrix function approximation.
\newblock In \emph{Proceedings of the Twenty-Ninth Annual {ACM-SIAM} Symposium
  on Discrete Algorithms, {SODA} 2018, New Orleans, LA, USA, January 7-10,
  2018}, pp.\  1605--1624. {SIAM}, 2018.

\bibitem[Nakatsukasa(2019)]{nakatsukasa2019lowrank}
Nakatsukasa, Y.
\newblock The low-rank eigenvalue problem, 2019.

\bibitem[Nikolentzos et~al.(2022)Nikolentzos, Siglidis, and
  Vazirgiannis]{graph_kernels}
Nikolentzos, G., Siglidis, G., and Vazirgiannis, M.
\newblock Graph kernels: A survey.
\newblock \emph{Journal of Artificial Intelligence Research}, 72:\penalty0
  943–1027, jan 2022.
\newblock ISSN 1076-9757.
\newblock \doi{10.1613/jair.1.13225}.
\newblock URL \url{https://doi.org/10.1613/jair.1.13225}.

\bibitem[Orecchia et~al.(2012)Orecchia, Sachdeva, and Vishnoi]{orecchia}
Orecchia, L., Sachdeva, S., and Vishnoi, N.~K.
\newblock Approximating the exponential, the {L}anczos method and an
  {\~{o}}(\emph{m})-time spectral algorithm for balanced separator.
\newblock In \emph{Proceedings of the 44th Symposium on Theory of Computing
  Conference, {STOC}}, pp.\  1141--1160. {ACM}, 2012.

\bibitem[Pettie \& Ramachandran(2002)Pettie and
  Ramachandran]{pettie2002optimal}
Pettie, S. and Ramachandran, V.
\newblock An optimal minimum spanning tree algorithm.
\newblock \emph{Journal of the ACM (JACM)}, 49\penalty0 (1):\penalty0 16--34,
  2002.

\bibitem[Peyr{\'e} et~al.(2016)Peyr{\'e}, Cuturi, and
  Solomon]{peyre:hal-01322992}
Peyr{\'e}, G., Cuturi, M., and Solomon, J.
\newblock {Gromov-Wasserstein Averaging of Kernel and Distance Matrices}.
\newblock In \emph{{ICML 2016}}, Proc. 33rd International Conference on Machine
  Learning, New-York, United States, June 2016.

\bibitem[Pfaff et~al.(2020)Pfaff, Fortunato, Sanchez-Gonzalez, and
  Battaglia]{pfaff2020learning}
Pfaff, T., Fortunato, M., Sanchez-Gonzalez, A., and Battaglia, P.~W.
\newblock Learning mesh-based simulation with graph networks.
\newblock \emph{arXiv preprint arXiv:2010.03409}, 2020.

\bibitem[Rahimi \& Recht(2007)Rahimi and Recht]{rf-1}
Rahimi, A. and Recht, B.
\newblock Random features for large-scale kernel machines.
\newblock In \emph{Advances in Neural Information Processing Systems 20,
  Proceedings of the Twenty-First Annual Conference on Neural Information
  Processing Systems, Vancouver, British Columbia, Canada, December 3-6, 2007},
  pp.\  1177--1184. Curran Associates, Inc., 2007.

\bibitem[Rahimi \& Recht(2008)Rahimi and Recht]{rf-2}
Rahimi, A. and Recht, B.
\newblock Weighted sums of random kitchen sinks: Replacing minimization with
  randomization in learning.
\newblock In \emph{Advances in Neural Information Processing Systems 21,
  Proceedings of the Twenty-Second Annual Conference on Neural Information
  Processing Systems, Vancouver, British Columbia, Canada, December 8-11,
  2008}, pp.\  1313--1320. Curran Associates, Inc., 2008.

\bibitem[Scetbon et~al.(2021)Scetbon, Peyré, and Cuturi]{lr_gw}
Scetbon, M., Peyré, G., and Cuturi, M.
\newblock Linear-time {G}romov {W}asserstein distances using low rank couplings
  and costs, 2021.

\bibitem[Seenappa et~al.(2019)Seenappa, Potika, and
  Potikas]{graph_classification_ieee_gcn}
Seenappa, M.~G., Potika, K., and Potikas, P.
\newblock Graph classification with kernels, embeddings and convolutional
  neural networks.
\newblock In \emph{2019 First International Conference on Graph Computing
  (GC)}, pp.\  88--93, 2019.
\newblock \doi{10.1109/GC46384.2019.00021}.

\bibitem[Shrivastava \& Li(2014)Shrivastava and Li]{lsh}
Shrivastava, A. and Li, P.
\newblock Asymmetric {LSH} {(ALSH)} for sublinear time maximum inner product
  search {(MIPS)}.
\newblock In Ghahramani, Z., Welling, M., Cortes, C., Lawrence, N.~D., and
  Weinberger, K.~Q. (eds.), \emph{Advances in Neural Information Processing
  Systems 27: Annual Conference on Neural Information Processing Systems}, pp.\
   2321--2329, 2014.

\bibitem[Smola \& Kondor(2003)Smola and Kondor]{smola_kondor}
Smola, A.~J. and Kondor, R.
\newblock Kernels and regularization on graphs.
\newblock In \emph{Computational Learning Theory and Kernel Machines, 16th
  Annual Conference on Computational Learning Theory and 7th Kernel Workshop,
  COLT/Kernel 2003, Washington, DC, USA, August 24-27, 2003, Proceedings},
  volume 2777 of \emph{Lecture Notes in Computer Science}, pp.\  144--158.
  Springer, 2003.

\bibitem[Solomon et~al.(2015)Solomon, De~Goes, Peyr{\'e}, Cuturi, Butscher,
  Nguyen, Du, and Guibas]{solomon2015convolutional}
Solomon, J., De~Goes, F., Peyr{\'e}, G., Cuturi, M., Butscher, A., Nguyen, A.,
  Du, T., and Guibas, L.
\newblock Convolutional {W}asserstein distances: Efficient optimal
  transportation on geometric domains.
\newblock \emph{ACM Transactions on Graphics (ToG)}, 34\penalty0 (4):\penalty0
  1--11, 2015.

\bibitem[Takahashi et~al.(2020)Takahashi, Chen, and Darve]{FMM-8}
Takahashi, T., Chen, C., and Darve, E.
\newblock Parallelization of the inverse fast multipole method with an
  application to boundary element method.
\newblock \emph{Comput. Phys. Commun.}, 247, 2020.

\bibitem[Thorup(2003)]{thorup}
Thorup, M.
\newblock Integer priority queues with decrease key in constant time and the
  single source shortest paths problem.
\newblock In Larmore, L.~L. and Goemans, M.~X. (eds.), \emph{Proceedings of the
  35th Annual {ACM} Symposium on Theory of Computing, June 9-11, 2003, San
  Diego, CA, {USA}}, pp.\  149--158. {ACM}, 2003.
\newblock \doi{10.1145/780542.780566}.
\newblock URL \url{https://doi.org/10.1145/780542.780566}.

\bibitem[Titouan et~al.(2019)Titouan, Courty, Tavenard, Laetitia, and
  Flamary]{pmlr-v97-titouan19a}
Titouan, V., Courty, N., Tavenard, R., Laetitia, C., and Flamary, R.
\newblock Optimal transport for structured data with application on graphs.
\newblock In Chaudhuri, K. and Salakhutdinov, R. (eds.), \emph{Proceedings of
  the 36th International Conference on Machine Learning}, volume~97 of
  \emph{Proceedings of Machine Learning Research}, pp.\  6275--6284. PMLR,
  09--15 Jun 2019.

\bibitem[Vayer et~al.(2018)Vayer, Chapel, Flamary, Tavenard, and Courty]{fgw}
Vayer, T., Chapel, L., Flamary, R., Tavenard, R., and Courty, N.
\newblock Fused {G}romov-{W}asserstein distance for structured objects:
  theoretical foundations and mathematical properties, 2018.

\bibitem[Villani(2003)]{villani2003}
Villani, C.
\newblock \emph{Topics in optimal transportation}.
\newblock Graduate studies in mathematics ; v. 58. American Mathematical
  Society, Providence, R.I, 2003.
\newblock ISBN 082183312X.

\bibitem[Wu et~al.(2015)Wu, Song, Khosla, Yu, Zhang, Tang, and
  Xiao]{Zhirong15CVPR}
Wu, Z., Song, S., Khosla, A., Yu, F., Zhang, L., Tang, X., and Xiao, J.
\newblock 3d shapenets: A deep representation for volumetric shapes.
\newblock In \emph{Computer Vision and Pattern Recognition}, 2015.

\bibitem[Xu et~al.(2019)Xu, Luo, Zha, and Duke]{pmlr-v97-xu19b}
Xu, H., Luo, D., Zha, H., and Duke, L.~C.
\newblock Gromov-{W}asserstein learning for graph matching and node embedding.
\newblock In Chaudhuri, K. and Salakhutdinov, R. (eds.), \emph{Proceedings of
  the 36th International Conference on Machine Learning}, volume~97 of
  \emph{Proceedings of Machine Learning Research}, pp.\  6932--6941. PMLR,
  09--15 Jun 2019.

\bibitem[Yokota et~al.(2016)Yokota, Ibeid, and Keyes]{FMM-4}
Yokota, R., Ibeid, H., and Keyes, D.~E.
\newblock Fast multipole method as a matrix-free hierarchical low-rank
  approximation.
\newblock \emph{CoRR}, abs/1602.02244, 2016.

\bibitem[Zhou \& Jacobson(2016)Zhou and Jacobson]{Thingi10K}
Zhou, Q. and Jacobson, A.
\newblock Thingi10k: A dataset of 10,000 3d-printing models.
\newblock \emph{arXiv preprint arXiv:1605.04797}, 2016.

\end{thebibliography}
\bibliographystyle{icml2023}

\newpage
\appendix
\onecolumn

{
\begin{center}
    \Large
    \textbf{{Appendix: Efficient Graph Field Integrators Meet Point Clouds}}
\end{center}
}

\section*{Broader Impacts}
Matrix-vector multiplication is a core component of all machine learning (ML) models. Thus there is a lot of interest in the ML community to discover ways or use cases where the above operation can be done in an efficient manner. This problem of fast matrix-vector multiplication also has tremendous applications in physical sciences, chemistry, and networking protocols. A vast body of literature has proposed scenarios where this problem is applicable. Our work makes an important contribution towards this direction by discovering new examples where such methods exist. We expect our work to benefit the ML community and the broader scientific community. Our work is mostly theoretical in nature, therefore we do not foresee any negative applications of our algorithms. 

\section{Theoretical Analysis}

\subsection{Warmup Results on $(\mathrm{G},f)$-tractability}
\label{sec:app_warmup}

We start with the following simple remark:

\begin{remark}
\label{remark:sim}
Let $\mathcal{G}$ be a family of graphs and let $\mathcal{F}=\{f_{1},\ldots,f_{|\mathcal{F}|}\}$ be a family of functions $\mathbb{R} \rightarrow \mathbb{C}$. If $(\mathcal{G},f_{i})$ is $T$-tractable for $i=1,\ldots,|\mathcal{F}|$ then for any $f:\mathbb{R} \rightarrow \mathbb{C}$ of the form: $f(z) = \sum_{i=1}^{|\mathcal{F}|}a_{i}f_{i}(z)$, where $a_{1},\ldots,a_{\mathcal{F}} \in \mathbb{C}$, $(\mathcal{G},f)$ is $(T \cdot \mathcal{F})$-tractable. 
\end{remark}

Furthermore, the following trivially holds:
\begin{remark}
\label{remark:im}
Let $\mathcal{G}$ be a family of graphs and let $f:\mathbb{R} \rightarrow \mathbb{C}$ be a function. If $(\mathcal{G},f)$ is $T$-tractable then $(\mathcal{G}, \mathrm{Re}(f))$ and $(\mathcal{G}, \mathrm{Im}(f))$ are $T$-tractable, where $\mathrm{Re}$ and $\mathrm{Im}$ stand for the real and imaginary part of $f$ respectively.
\end{remark}

The $|\mathrm{V}|$-tractability of $(\mathcal{T},f)$, where $\mathcal{T}$ is the family of trees and $f(z)=\exp(az+b)$, proven in \cite{topmasking}, combined with the above remarks implies several results for specific important classes of functions $f$. In particular, the following holds:

\begin{corollary}
If $\mathcal{T}$ is the family of trees and $f$ is given by a finite Fourier series of length $L$, then $(\mathcal{T},f)$ is $(\mathrm{V} \cdot L)$-tractable. Thus in particular: $(\mathcal{T},f)$ is $|\mathrm{V}|$-tractable for $f(z)=A \sin(\omega z + \phi)$ for $A,\omega,\phi \in \mathbb{R}$. This remain true if $f(z)=A \exp(-bz)\sin(\omega z + \phi)$, where $b \in \mathbb{R}$.
\end{corollary}

\subsection{Proof of Theorem \ref{theorem:bct_main}}
\label{sec:app_thmbct}

 Detailed analysis ragarding tree-decompositions with connected bags that we leverage in our theoretical analysis is illustrated in Sec. \ref{sec:connected_bags}.

\begin{proof}
Let $\mathrm{G} \in \mathcal{G}$ and denote: $N=|\mathrm{V}(\mathrm{G})|$. Without loss of generality, we can assume that $\mathrm{G}$ is connected.
We start with the following auxiliary lemma, where we introduce the key notion of graph \textit{separator}:
\begin{lemma}
\label{lemma:separator}
The set of vertices $\mathrm{V}(\mathrm{G})$ of $\mathrm{G}$ can be partitioned in time $O(N)$ into three pairwise disjoint sets: $\mathcal{A},\mathcal{B}, \mathcal{S}$ such that: $\delta N \leq |\mathcal{A}|, |\mathcal{B}|  \leq (1-\delta)N$ for some universal $0 < \delta < 1$ and $|\mathcal{S}| \leq t+1$, where $t=\mathrm{ctw}(\mathrm{G})$ stands for the connected treewidth of $\mathrm{G}$. Furthermore, no edges exist between $\mathcal{A}$ and $\mathcal{B}$
and: the induced sub-graphs $\mathrm{G}_{\mathcal{A}}$ and $\mathrm{G}_{\mathcal{B}}$ of $\mathrm{G}$ with sets of vertices $\mathcal{A}$ and $\mathcal{B}$ respectively both have connected treewidth at most $\mathrm{ctw}(\mathrm{G})$. Finally, the graph $\mathrm{G}_{\mathcal{S}}$ induced by $\mathcal{S}$ is connected.
We call set $\mathcal{S}$ a \textit{separator} in $\mathrm{G}$. Furthermore, the tree decomposition with connected bags and of treewidth $t$ can be found in time $O(N)$, with $\mathcal{S}$ being one of the bags.
\end{lemma}
\begin{proof}
This follows directly from the algorithmic proof of the following theorem:
\begin{theorem}[\citealp{pilipczuk-2}]
For a graph on $N$ vertices with treewidth $k$, there is an algorithm
that will return a tree decomposition with width $5k + 4$ in time $2^{O(k)}N$.
\end{theorem}
The bounded treewidth decomposition from the above theorem can be easily refined to the bounded connected treewidth deecomposition. Note that the $O(N)$ time-complexity algorithm for the bounded connected treewidth decomposition immediately implies that its representation is of size $O(N)$ (i.e. the corresponding tree has $O(N)$ edges/vertices).

Since our graph admits a tree decomposition with connected bags, the above tree decomposition can also be constructed to have this property. Now we can apply the algorithmic version of the proof of Lemma 7.19 from \cite{pilipczuk}, concluding that one of the bags of this tree decomposition is a balanced separator and can be found by searching the tree in time $O(N)$.

The base version of the lemma is provided below:
\begin{lemma}
\label{cyg-lemma}
Assume that $\mathrm{G}$ is a graph of treewidth at most $k$. Then there exists a  separator $X$ in $\mathrm{G}$ of size at most $k + 1$ and such that each connected component of the graph obtained from $\mathrm{G}$ by deleting vertices from $X$ and the incident edges has at most $\frac{1}{2}|V(\mathrm{G})|$ vertices.    
\end{lemma}

Let us explain know how the algorithmic version of the proof works. From the proof of the Lemma \ref{cyg-lemma} it is clear that as long as computing the sizes of the sets 
for all nodes of the tree (from the treewidth decomposition) can be done in $O(N)$
time, the balanced separator can be found in $O(N)$ time ($V_{t}$ in the Lemma refers to the union of all the bags in the nodes of the tree rooted in $t$). However this can be done via a standard recursion algorithm. Consider a node $t$ which is not a leaf and its children: $c_{1},c_{2},\ldots$
Note that in order to compute the size of the union of the sets: $V_{c_{1}},V_{c_{2}},\ldots$,  all that we need is: (a) the
individual sizes of the sets: $V_{c_{1}},V_{c_{2}},\ldots$ (that can be stored in the individual nodes as we progress with the recursion) (b) the number of children from the set $\{c_{1},c_{2},\ldots\}$ whose corresponding bags contain a given vertex $x$ from the bag $B_{t}$ associated with $t$ (for every $x$ in $B_{t}$). This is true since sets $V_{c_{1}},V_{c_{2}},\ldots$
are not necessarily disjoint, but by the definition of the treewidth, their intersections are subsets of $B_{t}$. Furthermore, if $V_{c_{i}}$ intersects with $B_{t}$, then (again, directly from the treewidth definition) this intersection is also a subset of the bag $B_{c_{i}}$ corresponding to $c_{i}$. All the computations from (b) can be clearly done in time $s \times O(k^{2})$, where s is the number of children of $t$
and $k$ is an upper bound on the bag size. Since we consider bounded connected treewidth graphs, time complexity reduces to $O(s)$. By unrolling this recursion, we obtain the algorithm of time complexity proportional to the number of edges of the tree from the treewidth decomposition which is $O(N)$ (see our discussion above). That completes the analysis.

\end{proof}
We are ready to prove the Theorem. We will identify the set of vertices $\mathrm{V}(\mathrm{G})$ with the set $\{1,\ldots,N\}$. Denote:
\begin{equation}
v_{i} = \sum_{j=1}^{N} f(\mathrm{dist}(i,j))x_{j}
\end{equation}
For a subset $\mathrm{S} \subseteq \{1,\ldots,N\}$, we will also use the following notation:
\begin{equation}
v_{i}^{\mathrm{S}} = \sum_{j \in \mathrm{S}} f(\mathrm{dist}(i,j))x_{j}
\end{equation}

Our goal is to compute $v_{i}$ for $i=1,\ldots,N$.
Equipped with Lemma \ref{lemma:separator}, we find the decomposition of $\mathrm{V}(\mathrm{G})$ into $\mathcal{A}$, $\mathcal{B}$ and $\mathcal{S}$ in time $O(N)$. We find the entire tree decomposition $\mathcal{T}$ from 
Lemma \ref{lemma:separator}.
Our strategy is first to compute $v_{i}$ for all $i \in \mathcal{S}$ and then to compute $v_{i}$ for all $i \in \mathcal{A} \cup \mathcal{B}$.

\textbf{1. Computation of $v_{i}$ for $i \in \mathcal{S}$.\\}
For every $i \in \mathcal{S}$, we can simply run Dijkstra's algorithm (or one of its improved variants mentioned in the main body) to find shortest path trees and, consequently, compute quantities $v_{i}$. This can clearly be done in time $O(|\mathcal{S}|(N+M)\log(M))$, where $M$ is the number of edges of $\mathrm{G}$. Since $\mathrm{G}$ is sparse, the total time complexity is $O(N\log(N))$ (or even $O(N\log\log(N))$ if the fastest algorithms for finding shortest paths in graphs with positive weights are applied, see: \cite{thorup}).

\textbf{2. Computation of $v_{i}$ for $i \in \mathcal{A} \cup \mathcal{B}$.\\}
We will show how to compute $v_{i}$ for all $i \in \mathcal{A}$. The calculations of $v_{i}$ for all $i \in \mathcal{B}$ will be completely analogous. Note first that:
\begin{equation}
v_{i} = v_{i}^{\mathcal{A}} + v_{i}^{\mathcal{S} \cup \mathcal{B}}
\end{equation}
\paragraph{2.1 Computation of $v_{i}^{\mathcal{A}}$ for all $i \in \mathcal{A}$.}
We first show how to compute $v_{i}^{\mathcal{A}}$ for all $i \in \mathcal{A}$.
Take a vertex $j \in \mathcal{A}$.
Denote by $P_{i,j}$ the shortest path from $i$ to $j$ in $\mathrm{G}$. If $P_{i,j}$ contains vertices from $\mathcal{B}$ then $P_{i,j}$ needs to use some vertices from $\mathcal{S}$ (since there are no edges between $\mathcal{A}$ and $\mathcal{B}$).
If that is the case, denote by $x$ the first vertex from $\mathcal{S}$ on the path $P_{i,j}$ as we go from $i$ to $j$ and by $y$ the last vertex from $\mathcal{S}$ on the path $P_{i,j}$ as we go from $i$ to $j$. Note that $x \neq y$. Denote by $P_{x,y}$ the sub-path of the path $P_{i,j}$ that starts at $x$ and ends at $y$.
Note that all the vertices of $\mathcal{B}$ that belong to $P_{i,j}$ also belong to $P_{x,y}$. 
Furthermore, since $P_{i,j}$ is the shortest path from $i$ to $j$, $P_{x,y}$ is the shortest path from $x$ to $y$. 

Note that $P_{x,y}$ is of length at most $|\mathcal{S}| - 1$, where $|\mathcal{S}|$ is the size of $\mathcal{S}$. This is the case since there exists a path from $x$ to $y$ using only vertices from $\mathcal{S}$ (since $\mathrm{G}_{\mathcal{S}}$ is connected). The number of edges $m_{x,y}$ of the path $P_{x,y}$ using at least one vertex from $\mathcal{B}$ is at most $t$.  

Denote: $\mathcal{P} = \{P_{x,y}:x,y \in \mathcal{S}, x \neq y\}$, where $P_{x,y}$ is the shortest path from $x$ to $y$ (if there are many such paths, choose an arbitrary one). 
Note that the size of $\mathcal{P}$ satisfies:
$|\mathcal{P}| \leq {\binom{|\mathcal{S}|} {2}}={\binom{t+1}{2}}$. Denote by $\mathcal{B}^{\prime}$ the subset of $\mathcal{B}$ consisting of the vertices of $\mathcal{B}$ that belong to these paths from $\mathcal{P}$ that have length at most $t$. Note that the size of $\mathcal{B}^{\prime}$ satisfies: $|\mathcal{B}^{\prime}| \leq t {\binom{t+1}{2}}$. Furthermore, set $\mathcal{B}^{\prime}$ can be found in time 
$O(|\mathcal{S}|(N+M)\log(M))=O(N\log(N))$, simply by running Dijkstra algorithm for every  vertex: $x \in \mathcal{S}$
(again, as before, this time can be improved to $N\log\log(N)$).

Note that since $\mathrm{G}$ has bounded connected treewidth, a subset $\mathcal{C} \subseteq \mathcal{B}$ of constant size can be found in $O(N)$ time (using $\mathcal{T}$) such that: $\mathrm{G}[\mathcal{A} \cup \mathcal{S} \cup \mathcal{B}^{\prime} \cup \mathcal{C}]$ has bounded connected treewidth $t$.
Denote: $\mathcal{A}_{\mathrm{ext}} = \mathcal{A} \cup \mathcal{S} \cup \mathcal{B}^{\prime} \cup \mathcal{C}$.
By the definition of $\mathcal{B}^{\prime}$, for every $i,j \in \mathcal{A}$ at least one of the shortest paths between $i$ and $j$ lies entirely in the sub-graph of $\mathrm{G}$ induced by $\mathcal{A}_{\mathrm{ext}}$. Furthermore, the sub-graph $\mathrm{G}_{\mathcal{A}_{\mathrm{ext}}}$ induced by $\mathcal{A}_{\mathrm{ext}}$ has connected treewidth at most $\mathrm{ctw}(\mathrm{G})$. 
We then recursively compute for each $i \in \mathcal{A}$ the following expression:
\begin{equation}
\tilde{v}_{i} = \sum_{j \in \mathrm{V}(\mathrm{G}_{\mathcal{A}_{\mathrm{ext}}})}f(\mathrm{dist}_{\mathrm{G}_{\mathcal{A}_{\mathrm{ext}}}}(i,j))x_{j},
\end{equation}
where $\mathrm{dist}_{\mathrm{G}_{\mathcal{A}_{\mathrm{ext}}}}$ is the shortest path distance in graph $\mathrm{G}_{\mathcal{A}_{\mathrm{ext}}}$.

Note that we have: $v_{i} = \tilde{v}_{i} - \delta_{i}$, where:
\begin{equation}
\delta_{i} = \sum_{j \in \mathcal{S} \cup \mathcal{B}^{\prime} \cup \mathcal{C}} f(\mathrm{dist}_{\mathrm{G}_{\mathcal{A}_{\mathrm{ext}}}}(j,i))x_{j}
\end{equation}

All $\delta_{i}$ can be computed in time 
$O(|\mathcal{S} \cup \mathcal{B}^{\prime} \cup \mathcal{C}|(N+M)\log(M))=O(N\log(N))$, simply by running Dijkstra's algorithms for every vertex $v \in \mathcal{S} \cup \mathcal{B}^{\prime} \cup \mathcal{C}$ (as before, this can be improved to $O(N\log\log(N))$ time).
That completes the computation of $v_{i}^{\mathcal{A}}$ for every $i \in \mathcal{A}$.

\paragraph{2.2 Computation of $v_{i}^{\mathcal{S} \cup \mathcal{B}}$ for all $i \in \mathcal{A}$.}
It remains to show how to compute $v_{i}^{\mathcal{S} \cup \mathcal{B}}$ for every $i \in \mathcal{A}$. To do that, we introduce for every  vertex $v \in \mathrm{V}(\mathrm{G})$ a vector $\chi_{v} \in \mathbb{R}^{|\mathcal{S}|}$ defined as follows:
\begin{equation}
\chi_{v}[k] = \mathrm{dist}(v,k)
\end{equation}
for $k=1,\ldots,|\mathcal{S}|$.
In the above, we identify the set $\mathcal{S}$ with $\{1,\ldots,|\mathcal{S}|\}$. Note that the following is true for any $i \in \mathcal{A}$, $j \in \mathcal{S} \cup \mathcal{B}$:
\begin{equation}
\mathrm{dist}(i,j) = \min_{k \in \mathcal{S}} \left(\chi_{i}[k] + \chi_{j}[k]\right),
\end{equation}
since every path from $i$ to $j$ needs to use a vertex from $\mathcal{S}$. The following is also true:
\begin{equation}
\chi_{v} = \tau_{v} + \rho_{v},
\end{equation}
where: $\tau_{v}[i]=\min_{k \in \mathcal{S}} \mathrm{dist}(v,k),  \forall i$ and furthermore $\rho_{v}$ is a vector satisfying for $k=1,\ldots,|\mathcal{S}|$:
\begin{equation}
0 \leq \rho_{v}[k] \leq |\mathcal{S}| - 1 = t    
\end{equation}
The latter is true since the lengths of any two shortest paths from $v$ to two vertices of $\mathcal{S}$ differ by at most $|\mathcal{S}| - 1$ (because $\mathrm{G}_{\mathcal{S}}$ is connected and thus there exists a path between any two vertices of $\mathrm{G}_{\mathcal{S}}$ of length at most $|\mathcal{S}| - 1$). 
We call $\rho_{v}$ the \textit{signature} of $v$ with respect to $\mathcal{S}$.
The following holds:
\begin{equation}
\label{eq:neat_formula_1}
\mathrm{dist}(i,j) = \tau_{i}[1] + \tau_{j}[1] + \min_{k \in \mathcal{S}} (\rho_{i}[k] + \rho_{j}[k])
\end{equation}
Note that all the vectors $\rho_{i}$ and $\tau_{i}$ for $i \in \mathrm{V}(\mathrm{G})$ can be computed in time $O(|\mathcal{S}|(N+M)\log(M))=O(N\log(N))$, by running Dijkstra's algorithm for every vertex $v \in \mathcal{S}$ (and again, we can improve this time complexity to $O(N\log\log(N))$).
We then partition the set $\mathcal{A}$ into subsets $\mathcal{A}_{\rho}$ (some of them potentially empty) indexed by the vectors $\rho \in \{0,1,\ldots,|\mathcal{S}|-1\}^{|\mathcal{S}|}$:
\begin{equation}
\mathcal{A}_{\rho} = \{i \in \mathcal{A}:\rho_{i}=\rho\}
\end{equation}
We define the partitioning of $\mathcal{C} = \mathcal{S} \cup \mathcal{B}$ into subsets $\mathcal{C}_{\rho}$ in the analogous way.

We will first compute $v_{i}^{\mathcal{C}_{\rho^{2}}}$ for every $i \in \mathcal{A}_{\rho^{1}}$ for given $\rho^{1},\rho^{2} \in \{0,1,\ldots,|\mathcal{S}|-1\}^{|\mathcal{S}|}$.
For fixed $\rho^{1},\rho^{2} \in \{0,1,\ldots,|\mathcal{S}|-1\}^{|\mathcal{S}|}$, we first show how to compute $v_{i}^{\mathcal{C}_{\rho^{2}}}$ for all $i \in \mathcal{A}_{\rho^{1}}$. We partition $\mathcal{A}_{\rho^{1}}$ into subsets (some of them potentially empty):
\begin{equation}
\mathcal{A}_{\rho^{1}}^{l} = \{i \in \mathcal{A}_{\rho^{1}}: \tau_{i}[1]=l\}
\end{equation}
for $l=0,1,\ldots,N-1$.

We define the partitioning of $\mathcal{C}_{\rho^{2}}$ into subsets $\mathcal{C}_{\rho^{2}}^{l}$ in the analogous way.

Note that for $i \in \mathcal{A}_{\rho^{1}}^{l_{1}}$
and $j \in \mathcal{C}_{\rho^{2}}^{l_{2}}$ the following is true: 
\begin{equation}
\mathrm{dist}(i,j) = l_{1} + l_{2} + g(\rho^{1},\rho^{2}),
\end{equation}
where $g(\rho^{1},\rho^{2}) \overset{\mathrm{def}}{=} \min_{k \in \mathcal{S}} (\rho^{1}[k] + \rho^{2}[k])$. 
We observe that the quantity $v_{i}^{\mathcal{C}_{\rho^{2}}}$ is the same for every $i \in \mathcal{A}_{\rho^{1}}^{l_{1}}$. Thus it suffices to compute $\{v_{i_{0}}^{\mathcal{C}_{\rho^{2}}},\ldots,v_{i_{N-1}}^{\mathcal{C}_{\rho^{2}}}\}$ for arbitrary representatives
$i_{u} \in \mathcal{A}_{\rho^{1}}^{u}$ (without loss of generality, we will assume that all sets $\mathcal{A}_{\rho^{1}}^{u}$ for $u=0,\ldots,N-1$ are nonempty).
If we define vector $\mathbf{w} \in \mathbb{R}^{N}$ as: $\mathbf{w}[u] = v_{i_{u}}^{\mathcal{C}_{\rho^{2}}}$ then we have: $\mathbf{w} = \mathbf{W}\mathbf{z}$, where vector $\mathbf{z} \in \mathbb{R}^{N}$ is given as follows:
\begin{equation}
\mathbf{z}[u] = \sum_{v \in \mathcal{C}_{\rho^{2}}^{u}} x_{v}
\end{equation}
and furthermore matrix $\mathbf{W} \in \mathbb{R}^{N \times N}$ is given as:
\begin{equation}
\mathbf{W}[l_{1},l_{2}] = f(l_{1}+l_{2}+g(\rho^{1},\rho^{2})) 
\end{equation}
Vector $\mathbf{z}$ can be easily computed in $O(N)$ time. The key observation is that multiplication $\mathbf{W}\mathbf{z}$ can be conducted in $O(N\log(N))$ time (the matrix $\mathbf{W}$ does not need to be explicitly materialized) with  the use of Fast Fourier Transform since $\mathbf{W}$ is a \textit{Hankel matrix} (constant along each anti-diagonal). Thus we conclude that we can compute $v_{i}^{\mathcal{C}_{\rho^{2}}}$ for all $i \in \mathcal{A}_{\rho^{1}}$ in $O(N\log(N))$ time. If a kernel is defined as $\mathrm{K}(i,j)=\exp(-\lambda\mathrm{dist(i,j)})$, this becomes a very special Hankel matrix, where each row is obtained from the previous one by multiplication with a fixed constant. It is easy to see that the multiplications with such matrices $\mathbf{W}$ can be conducted in time $O(N)$ (we thus save a $\log(N)$-factor).

By applying this method to all pairs $(\rho^{1}, \rho^{2})$, we conclude that we can compute $v_{i}^{\mathcal{S} \cup \mathcal{B}}$ for all $i \in \mathcal{A}$ in time 
$O(|\mathcal{S}|^{|\mathcal{S}|-1} \cdot |\mathcal{S}|^{|\mathcal{S}|-1} N\log(N))=O(N\log(N))$. This can be improved to $O(N)$ time if $\mathrm{K}(i,j)=\exp(-\lambda\mathrm{dist(i,j)})$ is being applied.

Combining step 1 with steps 2.1 and 2.2, we obtain a method for computing $v_{i}$ for every $i \in \mathcal{A}$. The computations of $v_{i}$ for every $i \in \mathcal{B}$ are conducted in a completely analogous way (where we borrow the notation from the above analysis but adapt to this case, e.g., we replace $\mathcal{B}^{\prime}$ with $\mathcal{A}^{\prime}$).

\textbf{3. Putting this all together -- time complexity analysis. \\}
To summarize, $v_{i}$ for all $i \in \mathrm{V}(\mathrm{G})$ can be computed in time:
\begin{equation}
T(N) = T(N_{1}) + T(N_{2}) + O(N\log(N)),
\end{equation}
where: $\rho N \leq N_{1},N_{2} \leq (1-\rho)N + C$ for constants $0<\rho<1, C>0$.
It is easy to see that the solution to this recursive equation satisfies the following:
\begin{equation}
T(N) = O(N\log^{2}(N))
\end{equation}
If the kernel being used is of the form: $\mathrm{K}(i,j)=\exp(-\lambda\mathrm{dist(i,j)})$, then the recursion for the total runtime is of the form:
\begin{equation}
T(N) = T(N_{1}) + T(N_{2}) + O(N\log\log(N)),
\end{equation}
which implies that: $T=O(N\log^{1.38}(N))$.
That completes the entire proof.
\end{proof}

\begin{remark}
Note that the proof of the above result but for the family $\mathcal{G}$ of trees from \cite{topmasking} is a special instantiation of the proof presented above.
\end{remark}

\vspace{-3mm}
\subsection{Tree-Decomposition with Connected Bags}
\label{sec:connected_bags}

Let $\mathcal{D} = (T, (X_v: v \in V(T))$ be a tree decomposition of a finite connected undirected graph $G$. We say that $\mathcal{D}$ is a connected tree-decomposition of for every $v \in V(T)$, the subgraph of $G$ induced on $X_v$ is connected. For a vertex $v \in V(T)$ and a path $P$ of $G$, we say that $P$ {\em internally avoids} $v$ if none of the vertices of $P$ are in $X_v$, except for possibly its end vertices. For two vertices $x,y \in V(G)$, denote by $d(x,y,\bar{v})$ the edge length of the shortest path between $x$ and $y$ avoiding $v$, or $\infty$ if no such path exists. 

\begin{theorem}
    Let $\mathcal{D} = (T, (X_v: v \in V(T))$ be a tree decomposition of a finite connected undirected graph $G$.
    Write $L = |V(T)|$ and let $k$ be the maximum size of a bag in $\mathcal{D}$. Then there is an algorithm of time complexity $O(Lk^3)$ calculating $d(x,y,\bar{v})$ for every $v \in V(T)$ and every $x,y \in X_u$ for every $T$-neighbor $u$ of $v$. 
\end{theorem}

\begin{proof}
    If $u$ is a leaf of $T$ and $v$ is its only $T$ neighbor, then we can use Dijkstra's Algorithm to calculate $d(x,y,\bar{v})$ for every $x,y \in X_u$ in time $O(k^3)$. 

    Now let $\{u,v\}$ be any edge of $T$ and suppose $d(x,y,\bar{v})$ was already calculated for every $x,y \in X_w$ whenever $w \neq u$ is a $T$-neighbor of $v$. Let $s,t \in X_v$. we note that every path $P$ from $s$ to $t$ internally avoiding $u$ can be cut into paths $Q_1, \ldots, Q_j$ internally avoiding $v$. If $Q_i$ has more than one edge then its two ends must belong to $X_w$ for some $T$-neighbor $w \neq u$ of $v$. Therefore we can calculate $d(s,t,\bar{u})$ for every $s,t \in X_v$ in time $O(k^3)$.
    We need to repeat this twice for every edge of $T$, i.e., $2L-2$ times, so altogether the algorithm runs in time $O(Lk^3)$.
\end{proof}

\begin{theorem}
    Let $\mathcal{D} = (T, (X_v: v \in V(T))$ be a tree decomposition of a finite connected undirected graph $G$.
    Write $L = |V(T)|$ and let $k$ be the maximum size of a bag in $\mathcal{D}$. Then there is an algorithm of time complexity $O(Lk^3)$ obtaining a new tree decomposition $\mathcal{D}' = (T', (X'_v: v \in V(T'))$ where for every $v \in V(T')$ and every $x,y \in X_u$ for every $T'$-neighbor $u$ of $v$, we have $d(x,y,\bar{v}) < \infty$.
\end{theorem}

\begin{proof}
    We first calculate $d(x,y,\bar{v})$ for every $v \in V(T)$ and every $x,y \in X_u$ for every $T$-neighbor $u$ of $v$. Then for every $u \in V(T)$ we define a relation $\sim_u$ on $X_u$, where $x \sim_u y$ if and only if $d(x,y,\bar{v}) < \infty$ for every $T$-neighbor $v$ of $u$. We note that this is an equivalence relation. Denote its classes by $C_1(u), \ldots, C_{j_u}(u)$. 

    We now define $V(T')$ to be the set of all pairs $(u,i)$ where $u \in V(T)$ and $i \in \{1, \ldots, j_u\}$. We define $X'_{(u,i)}=C_i(u)$ and $E(T') = \{\{(u_1,i_1), (u_2,i_2)\}: \{u_1,u_2\} \in E(T), \, X'_{(u_1,i_1)} \cap X'_{(u_2,i_2)} \neq \emptyset \}$. 
\end{proof}

\begin{theorem}
    Let $\mathcal{D} = (T, (X_v: v \in V(T))$ be a tree decomposition of a finite connected undirected graph $G$.
    Write $L = |V(T)|$ and let $k$ be the maximum size of a bag in $\mathcal{D}$. Let $\ell>3$ be an integer and suppose $G$ has no geodesic cycle of length more than  $\ell$. Then there is an algorithm of time complexity $O(Lk^2 (k+l))$ obtaining a new tree decomposition $\mathcal{D}'' = (T'', (X''_v: v \in V(T''))$ where $\mathcal{D}''$ is connected and every bag in it has size at most $k^2 \ell$ vertices.  
\end{theorem}

\begin{proof}
    We first calculate the values $d(x,y,\bar{v})$ and construct the tree decomposition $\mathcal{D}'$.  We can use the calculated values of $d(x,y,\bar{v})$ in order to construct a path $P_{st}$ of length $dist_G(s,t)$ between $s$ and $t$ whenever $\{s,t\}$ is contained in some bag of $\mathcal{D'}$ and $dist_G(s,t) < \ell$. 

    We now define $T'' = T'$ and for every $w \in V(T'')$ we define

    $$X''_w = \bigcup_{s,t \in X'_w, \, dist_G(s,t)<\ell} V(P_{st}) \; .$$

    This yields a tree decomposition, and if for some $w \in V(T'')$ the induce graph $G[X''_w]$ is not connected then $G$ has a geodesic cycle of length more than $\ell$. 

    In order to see this, assume for contradiction that $G[X''_w]$ is not connected, and let $P$ be a shortest path in $G$ between vertices of $X'_w$ which are in different connected components of $G[X''_w]$. Let $x,y$ be the ends of $P$. Let $z$ be some internal vertex of $P$, let $u$ be a vertex of $T'$ with $z \in X'_u$, and let $v$ be the second vertex in the $T'$ path from $w$ to $u$. Let $Q$ be a path of $G$ between $x$ and $y$ internally avoiding $v$. 

    Then the concatenation $PQ$ is a cycle in $G$. It might be non-geodesic, however, if we apply shortcuts to it until we obtain a geodesic cycle, then this cycle must still contain a path in $G$ between vertices of $X'_w$ which are in different connected components of $G[X''_w]$, and therefore it has more than $\ell$ edges. 
\end{proof}

\subsection{Proof of Lemma \ref{lemma:rfd}}
\label{sec:lemma_rfd}
\begin{proof}
Denote: $\tilde{f}(\mathbf{z})=\int_{\mathcal{B}(R)} \exp(2\pi \mathbf{i} \omega^{\top}\mathbf{z})\tau(\omega)d\omega$ for $\mathbf{z}=\mathbf{n}_{v}-\mathbf{n}_{w}$.
Note that in the efficient implementation of the RFD-estimator (that neglects the imaginary part of the dot-product), we have: 
\begin{equation}
\widehat{\mathbf{W}}(v,w)=\frac{1}{m}\sum_{i=1}^{m}X_{i},
\end{equation}
where: $X_{i} = \cos(2\pi\omega_{i}^{\top}\mathbf{z})\frac{\tau(\omega_{i})}{p(\omega_{i})}$,
$\tau$ is the Fourier Transform (FT) of the function $\mathbbm{1}[\|\mathbf{z}\|_{1} \leq \epsilon]$
and $p$ is the pdf of the $R$-truncated Gaussian distribution (e.g. Gaussian distribution truncated to the $L_{1}$-ball $\mathcal{B}(R)$ of radius $R$ and centered at $0$).
We clearly have: $\mathbb{E}[\widehat{\mathbf{W}}(v,w)]=\tilde{f}(\mathbf{z})$. Furthermore, the following holds:
\begin{equation}
\mathrm{Var}(\widehat{\mathbf{W}}(v,w)) = \frac{1}{m^{2}}\cdot m \cdot \mathrm{Var}(X_{1})
=\frac{1}{m}\mathrm{Var}(X_{1})
\end{equation}
Note that: $p(\omega) = \frac{1}{(2\pi)^{\frac{3}{2}}}\exp(-\frac{\|\omega\|^{2}}{2})\cdot C^{-1}$.
We have: $\mathrm{Var}(X_{1}) = \mathbb{E}[X_{1}^{2}]-(\tilde{f}(\mathbf{z}))^{2}$ and:
\begin{align}
\begin{split}
\mathbb{E}[X_{1}^{2}] = \int_{\mathcal{B}(R)} \cos^{2}(2\pi\omega^{\top}\mathbf{z})\tau^{2}(\omega)p^{-1}(\omega)d\omega \leq (2\pi)^{\frac{3}{2}}C \int_{\mathcal{B}(R)} 
\frac{\sin^{2}(2\epsilon \omega_{1})}{\omega^{2}_{1}}\cdot\ldots\cdot
\frac{\sin^{2}(2\epsilon \omega_{d})}{\omega^{2}_{d}} \cdot \\ \exp(\frac{\omega_{1}^{2}}{2}) \cdot \ldots \cdot \exp(\frac{\omega_{d}^{2}}{2}) d\omega
 = (2\pi)^{\frac{3}{2}}C \Gamma^{d}_{\epsilon}
\end{split}
\end{align}
We have leveraged the formula for the FT of the function $\mathbbm{1}[\|\mathbf{z}\|_{1} \leq \epsilon]$ from the main body of the paper. 
Thus we have:
\begin{equation}
\mathrm{Var}(\widehat{\mathbf{W}}(v,w)) \leq \frac{1}{m}\left((2\pi)^{\frac{3}{2}}C (\Gamma_{\epsilon}(R))^{d}-(\tilde{f}(\mathbf{z}))^{2}\right)
\end{equation}
We also have:
\begin{equation}
\mathrm{MSE}(\widehat{\mathbf{W}}(v,w)) \leq \mathrm{Var}(\widehat{\mathbf{W}}(v,w)) +
(\tilde{f}(\mathbf{z}))^{2} - (f(\mathbf{z}))^{2} = 
\mathrm{Var}(\widehat{\mathbf{W}}(v,w)) + (\tilde{f}(\mathbf{z}) - f(\mathbf{z}))(\tilde{f}(\mathbf{z}) + f(\mathbf{z}))  
\end{equation}
Plugging in the formula for $\tilde{f}(\mathbf{z})$ and the variance, we complete the proof.
\end{proof}

\section{Graph Metric Approximation with Trees}~\label{sec-graph-metric-approx}
Define $\operatorname{dist}_{\mathrm{G}}(\cdot, \cdot)$ to be a shortest path distance function on a graph $\mathrm G $.
Recall that when $\mathcal{T}$ is the family of trees and $f(z)=\exp(az+b)$, then we can compute the GFI in time $O(|\mathrm{V}|)$ using dynamic programming (via single bottom-up and single top-down traversal of the tree). In other words, $(\mathcal{T},f)$ is $|\mathrm{V}|$-tractable. Therefore, it is advantageous to consider representing a graph by trees that preserve/approximate its metric. 
\paragraph{Spanning tree.}
A naive tree approximating the weighted graph metric is the graph's minimum spanning tree. The optimal running time for finding a minimum spanning tree is $O(|\mathrm{E}|\cdot \alpha (|\mathrm{V}|, |\mathrm{E}|))$ \cite{pettie2002optimal}, where $\alpha$ is incredibly slowly growing function.

Note that while it is cheap to build a spanning tree, the distortion of the shortest path distance of the original graph can be considerable. For example, let $\mathrm{G}$ be an unweighted $n$-cycle and $\mathrm{T}$ be its minimum spanning tree. Then the distortion between leaf nodes in the spanning tree is $\operatorname{dist}_{\mathrm{T}}(u,v)/\operatorname{dist}_{\mathrm{G}}(u,v)=n-1$. However, if instead of a single tree, we can embed  our graph into a distribution over trees, then we can hope to get better \emph{expected} distortion. In this specific example, if we take a uniform distribution over $n$ different spanning trees (each obtained by deleting an edge), then the expected distance distortion becomes 
\[
\mathbb E_{\mathrm{T}}\left[\operatorname{dist}_{\mathrm{T}}(u,v)/\operatorname{dist}_{\mathrm{G}}(u,v)\right] =2(1-1/n).
\]  
This brings us to another type of method based on embedding the arbitrary weighted graph metric into the distribution of trees.
\paragraph{Low-distortion trees.} Building on the low diameter randomized decomposition, \citet{bartal1996probabilistic} introduced an algorithm for sampling random hierarchically well-separated trees with expected distortion factor $O(\log^2 |\mathrm{V}|)$. Assume that the diameter of an input graph is $O(\operatorname{poly}(|\mathrm{V}|))$. Then for all  $u, v\in \mathrm{V}$,
the expectation over random tree $\mathrm T$ of the distance distortion is
\[
\mathbb E_{\mathrm{T}}\left[\operatorname{dist}_{\mathrm{T}}(u,v)/\operatorname{dist}_{\mathrm{G}}(u,v)\right] \in \left[1,~ O(\log^2 |\mathrm{V}|)\right].
\]
Time complexity to sample a single Bartal tree is $O(\log |\mathrm{V}| (|\mathrm{E}| + |\mathrm{V}|\log |\mathrm{V}| ))$. Further, \citet{fakcharoenphol2004tight} improved an arbitrary metric space embedding into random trees by providing a constructive algorithm with optimal distortion factor of $\Theta(\log |\mathrm{V}|)$, i.e.,
\[
\mathbb E_{\mathrm{T}}\left[\operatorname{dist}_{\mathrm{T}}(u,v)/\operatorname{dist}_{\mathrm{G}}(u,v)\right] \in \left[1,~ O(\log |\mathrm{V}|)\right].
\]
Note that this improvement comes at higher time complexity for sampling a tree. Unlike FRT trees~\cite{fakcharoenphol2004tight}, during the low-diameter decomposition in the Bartal algorithm, the cluster center is always included in the cluster itself.  
As a result, in the Bartal algorithm, we can consolidate the sub-trees recursively without introducing new vertices. In contrast, the FRT algorithm defines a laminar family, which corresponds to a rooted tree with graph nodes at the leaves.

In our applications for fast graph field integration during preprocessing, we sample $\mathrm T_1, \ldots, \mathrm T_k$ trees independently from one of the above distributions. Note that sampling can be done in parallel. For the inference, we compute
\[
i(v) = \frac{1}{k}\sum_{i=1}^k\sum_{w \in \mathrm{V}}f(\mathrm{dist}_{\mathrm T_i}(w,v))\mathcal{F}(w),
\]
which takes $O(k|\mathrm V|)$. The integration on each  tree can be carried out in parallel, reducing the inference time by a factor of $k$.

In Fig. \ref{fig:vertex_normal}, we set the number of trees in T-FRT (FRT trees) to 3 and implemented two variants of T-Bart (Bartal trees) with 3 and 20 trees, respectively.
\section{Interpolation on Meshes}~\label{sec:mesh_interpolation} 
In this section, we present implementation details for Sec. \ref{sec:interpolation_meshes}. All experiments are run on a single computer with an i9-12900k CPU and 96GB memory. 

\subsection{Vertex Normal Prediction.} In the vertex normal prediction task in Sec. \ref{sec:vertex_normal_prediction}, we choose 120 meshes for 3D-printed objects with a wide range of size from the Thingi10k \cite{Thingi10K} dataset with the following File IDs:

\small
\texttt{[60246, 85580, 40179, 964933, 1624039, 91657, 79183, 82407, 91658, 40172, 65414, 90431, 74449, 73464, 230349, 40171, 61193, 77938, 375276, 39463, 110793, 368622, 37326, 42435, 1514901, 65282, 116878, 550964, 409624, 101902, 73410, 87602, 255172, 98480, 57140, 285606, 96123, 203289, 87601, 409629, 37384, 57084, 136024, 202267, 101619, 72896, 103538, 90064, 53159, 127243, 293452, 78671, 75667, 285610, 80597, 90736, 75651, 1220293, 126660, 75654, 75657, 111240, 75665, 75652, 68706, 123472, 88855, 470464, 444375, 208741, 80908, 73877, 495918, 1215157, 85758, 80516, 101582, 75496, 441708, 796150, 257881, 68381, 294160, 265473, 762595, 461110, 461111, 38554, 762594, 79353, 81589, 95444, 762586, 762610, 762607, 1335002, 274379, 437375, 59333, 551074, 550810, 93130, 372053, 372059, 133078, 178340, 133079, 133568, 331105, 80650, 47984, 551021, 308214, 372057, 59197, 1717685, 439142, 372058, 376252, 372114]}
\normalsize 

For each method listed in Fig. \ref{fig:acceleration_prediction}, we do a grid-search over its hyper-parameter(s) for each mesh and report the pre-processing time and interpolation time associated with the hyper-parameter(s) that give(s) us the best cosine similarity. 




\section{Wasserstein Distances and Barycenters}
The Wasserstein metric has received a lot of attention in the machine learning
community, especially for its principled way of comparing distributions on a metric space $\mathcal{X}$~\citep{villani2003}. It is a distance function between probability measures defined on $\mathcal{X}$, while strongly reflecting the metric of the underlying space. In spite of their broad use, Wasserstein distances are computationally expensive.

\subsection{Wasserstein Barycenters on Meshes}
To alleviate this computational bottleneck,
convolutional Wasserstein distance is proposed as an entropic regularized Wasserstein distance over geometric domains by leveraging the heat kernel as the proxy for the geodesic distance over the manifold. 

The heat kernel matrix $\mathbf{H}$ can be seen as a generalization of a Gaussian kernel on a manifold by Varadhan’s formula~\cite{solomon2015convolutional}. Moreover, the geodesic distance function $\mathrm{dist}(i,j)$ on a surface mesh can be approximated by the Euclidean distance on an $\epsilon$-nearest-neighbor graph (when $\epsilon$ is suitably chosen). Note that, unlike the Gaussian kernel, the heat kernel can be efficiently computed. 

In our work, we approximate the geodesic distance on a surface mesh by (1) shortest-path distance (used in SF calculations), and 2) distance coming from an $\epsilon$-NN graph approximating the surface (RFD).

\subsubsection{Efficient computation of Wasserstein barycenter}~\label{sec:conv_wass_bary_appendix}
One of the key steps for the computation of Wasserstein distance is the action $\mathbf{H}$ on a given vector $\mathbf{x}$.~\citet{solomon2015convolutional} use a pre-factorized decomposition of $\mathbf{H}$ to do the above matrix-vector multiplication efficiently without actually materializing $\mathbf{H}$. 

Similar to their method, we never materialize our kernel matrices $\mathbf{K}$ explicitly, i.e. we only need to know how to apply $\mathbf{K}$ to vectors. Here FM can either be SeparationFactorization (SF) or the RFDiffusion algorithm (RFD) and for clarity, we use the subscript for the matrix $\mathbf{K}$ to specify that we are approximating the (right) action of the matrix $\mathbf{K}$. 

We define $\otimes$ as the Hadamard product (also known as the element-wise product) and $\oslash$ as the element-wise division. 

\begin{algorithm}[!htb]
   \caption{Fast Computation of Wasserstein Barycenter}
   \label{alg:wass_barycenter}
\begin{algorithmic}
   \STATE {\bfseries Inputs:} probability distributions $\left\{\boldsymbol{\mu}^i\right\}_{i=1}^k$, area weights $\vec{a}\in\mathbb{R}^N_+, \text{maxIter}\in\mathbb{N}$, $\boldsymbol{\alpha}\in\mathbb{R}^k_+$
   \STATE {\bfseries Output:} Wasserstein barycenter $\boldsymbol{\mu}\in \operatorname{Prob}(\mathrm{V})$
   \STATE {\bfseries Initialize:} $\mathbf{v}^1, \ldots, \mathbf{v}^k \leftarrow \vec{1}, 
 \mathbf{w}^1, \ldots, \mathbf{w}^k \leftarrow \vec{1}, \boldsymbol{\mu} \leftarrow \vec{1}$.
   \STATE for $j\le \text{maxIter}$ 
   \STATE\hskip2.5em for $i=1, \ldots, k$ 
   \STATE\hskip5em 1. $\mathbf{w}^i \leftarrow \boldsymbol{\mu}^i \oslash \left(\mathrm{FM}_{\mathbf{K}} (\vec{a} \otimes \mathbf{v}^i) \right)$
   \STATE\hskip5em 2. $\mathbf{d}^i \leftarrow \mathbf{v}^i \otimes \left(\mathrm{FM}_{\mathbf{K}} (\vec{a} \otimes \mathbf{w}^i)\right)$
   \STATE\hskip5em 3. $\boldsymbol{\mu} \leftarrow \boldsymbol{\mu} \otimes (\mathbf{d}^i)^{\boldsymbol{\alpha_i}}$
   \STATE\hskip2.5em for $i=1, \ldots, k$
   \STATE\hskip5em 4. $\mathbf{v}^i \leftarrow \mathbf{v}^i \otimes \boldsymbol{\mu} \oslash \mathbf{d}^i$  
   \STATE {\bfseries return} $\boldsymbol{\mu}$
\end{algorithmic}
\end{algorithm}

\subsubsection{Details on Baselines}~\label{sec:wass_bary_baselines}
For the \BF\ in separation integration, we first compute the pairwise \emph{shortest path distances} for all vertices, using vertices and edges in the input mesh. We then compute the \emph{element-wise exponential} $\mathbf{K}$ with $\mathbf{K}_{ij}:=\exp(-\lambda\mathrm{dist}(i,j))$ for all $i,j$. 

For the \BF\ in diffusion integration, we use the vertex embeddings of the input mesh and create a graph $\mathrm{G}$ with edges between nodes  $i$ and $j$ if $\|\mathbf{n}_i-\mathbf{n}_j\|_{1} \leq\epsilon$. After creating the set of edges, we compute the \emph{matrix exponential} $\mathbf{K}=\exp(\lambda  \mathbf{W}_{\mathrm{G}})$ with $(\mathbf{W}_{\mathrm{G}})_{ij}:=\|\mathbf{n}_i-\mathbf{n}_j\|_{1}\cdot \mathbbm{1}[\|\mathbf{n}_i-\mathbf{n}_j\|_{1} \le \epsilon]$. 

In steps $1$ and $2$ in the algorithm~\ref{alg:wass_barycenter}, both the baseline variants explicitly perform $\mathbf{Kx}$, the matrix-vector multiplication. 

\subsubsection{Details on Hyper-parameters}
For diffusion integration, we fix parameters $\epsilon=0.01$ and $\lambda=0.5$, and the number of random features is 30. For separation integration, we choose $\lambda=0.2$, $\text{unit-size}=0.1$, $\text{threshold}=2000$ (the maximum size of the graph, measured in the number of vertices, for which the integrator is conducted in a brute-force manner).

For computations of the Wasserstein barycenters, the input vector $\vec{a}$ contains area weights for vertices. The area weights are proportional to the sum of triangle areas adjacent to each vertex in a triangle mesh \citep{solomon2015convolutional}. We set the number of input distributions $k=3$ and $\boldsymbol{\alpha}=\vec{\frac{1}{k}}=\vec{\frac{1}{3}}$. For each mesh, we generate three different input distributions $\boldsymbol{\mu}^i$, each with mass concentrated in vertices surrounding a distinct center vertex.   

The experiments performed for Wasserstein barycenter are conducted on an 8-CPU core Ubuntu virtual machine on Google Cloud Compute. 

\subsubsection{Additional Experiments}~\label{sec:wass_sol}
In Table~\ref{table:wass_bary_diffusion_with_solomon}, we include comparisons of \RFD\ with an additional baseline \citep{solomon2015convolutional}. Even though there are similarities between our work and the above mentioned autors, we would like to point the following key differences : (1) \citet{solomon2015convolutional} does not consider an $\epsilon$-graph and (2) uses the heat kernel which is constructed using mesh Laplacian instead of our matrix exponential of the weighted adjacency matrix. 

Moreover, we note that the kernel employed in our \SF\ experiments can be seen as a generalization of the Laplace kernel on the manifold, and thus not directly comparable to the heat kernel. 


\begin{table}[h]
    \begin{center}
            \caption{Comparison of the total runtime and mean-squared error (MSE) across several meshes for diffusion-based integration. $\mathrm{Slmn}$ is the integrator from \citep{solomon2015convolutional}. Runtimes are reported in seconds. The lowest time for each mesh is shown in bold. MSE is calculated w.r.t. the output of brute force (BF).}
        \begin{tabular}{@{}p{8mm}crrr c rr@{}}
        \toprule
        \multirow{2}{*}{\textbf{Mesh}} & \multirow{2}{*}{\textbf{$|\mathrm{V}|$}} & \multicolumn{3}{c}{\textbf{Total Runtime}} & & \multicolumn{2}{c}{\textbf{MSE}} \\
        \cmidrule{3-5} \cmidrule{7-8}
        & & \BF & Slmn & \RFD & & Slmn & \RFD \\
        \midrule
        {Alien} & $5212$ & 8.06 & 0.57 & \textbf{0.39} & & 0.042 & 0.041 \\
        {Duck} & $9862$ & 45.36 & 1.94 & \textbf{1.10} & & 0.002 & 0.002\\
        {Land} & $14738$ & 147.64 & 4.17 & \textbf{2.17} & & 0.023 & 0.017\\
        {Octocat} & $18944$ & 302.84 & 6.74 & \textbf{3.36} & & 0.022 & 0.027\\
        \bottomrule
        \end{tabular}
\label{table:wass_bary_diffusion_with_solomon}
    \end{center}
\end{table}

\subsection{Gromov Wasserstein Distance}~\label{sec:gw_fgw}
The optimal transport (OT) problem associated with Gromov-Wasserstein (\GW) discrepancy~\citep{peyre:hal-01322992},
which extends the Gromov-Wasserstein distance~\citep{memoli}, has emerged as an effective transportation distance between structured data, alleviating the incomparability issue between different structures by aligning the intra-relational
geometries. The \GW~discrepancy problem can be solved iteratively by conditional gradient method~\citep{peyre:hal-01322992} and the proximal point algorithm~\citep{pmlr-v97-xu19b}.  \GW~distance is isometric, meaning the unchanged similarity under rotation, translation, and permutation and is thus related to graph matching problem, encoding structural information to compare graphs, and has also been successfully adopted in image recognition~\citep{peyre:hal-01322992}, alignment of large single-cell datasets~\citep{Demetci2020.SCOT}, and point-cloud data alignment~\citep{Memoli2006}. However, despite its broad
use, the Gromov Wasserstein distance is computationally expensive as it scales as $O(n^2m^2)$ where $n, m$ are the numbers of source and target nodes respectively.

\subsubsection{(Fused) Gromov Wasserstein Discrepancy}~\label{fgw_comp}
Formally, the Gromov-Wasserstein discrepancy between two measured similarity
matrices $(\mathbf{C}, \mathbf{p}) \in \mathbb{R}^{n\times n} \times \Sigma_n$ and $ (\mathbf{D}, \mathbf{q}) \in \mathbb{R}^{m \times m} \times \Sigma_m $ is defined as :
\begin{equation}~\label{gwa}
\begin{split}
\operatorname{GW}(\mathbf{C}, \mathbf{D}, \mathbf{p}, \mathbf{q}) & = \min_{\mathbf{T} \in \mathcal{C}_{\mathbf{p},\mathbf{q}}} \sum_{i,j,k,l}\ell(\mathbf{C}_{i,k}, \mathbf{D}_{j,l})\mathbf{T}_{i,j}\mathbf{T}_{k,l}   \\
    & = \min_{\mathbf{T} \in \mathcal{C}_{\mathbf{p},\mathbf{q}}} \langle L(\mathbf{C}, \mathbf{D} , \mathbf{T} ), \mathbf{T} \rangle
\end{split}
\end{equation}
where $\mathbf{C}$ and $\mathbf{D}$ are matrices representing either similarities
or distances between nodes within the graph, $\mathbf{A}_{i,j}$ is the $ij$th entry of the matrix $\mathbf{A}$,  $\ell$ is the loss function applied elementwise on the matrices. The common choices of the loss function are  Euclidean distance, i.e. $\ell(x,y):= (x-y)^2$ or KL-divergence, i.e. $\ell(x,y):=x \log \frac{x}{
y} - x + y$, $\mathbf{p} \in \mathbb{R}^{n}_{+}$ (resp. $\mathbf{q} \in \mathbb{R}^{m}_{+}$),  $\sum p_i = 1 $ (resp. $\sum q_i = 1 $),
is the probability simplex of histograms
with $n$ (resp. $m$) bins, and $\mathbf{T}$ is the coupling matrix between the two spaces on
which the similarity matrices are defined, i.e.
\begin{equation}~\label{eqn:coupling}
    \mathcal{C}_{\mathbf{p},\mathbf{q}} = \{\mathbf{T} \in \mathbb{R}^{n \times m}_{+}  \mid \mathbf{T}\ones_{m} = \mathbf{p}, \mathbf{T}^{\top}\ones_{n} = \mathbf{q} \}
\end{equation}
Define, $L(\mathbf{C}, \mathbf{D} , \mathbf{T} ) := [\mathbf{L}_{k,l}] \in \mathbb{R}^{n \times m}$ and $\mathbf{L}_{k,l} := \sum_{i,j}\ell(\mathbf{C}_{i,k}, \mathbf{D}_{j,l})\mathbf{T}_{i,j}$
and $\langle \cdot, \cdot \rangle$ is the inner
product of matrices. For the rest of the section, we use the Euclidean distance $\ell$ as our loss function. 

Following the work of~\citep{fgw}, the concept of \GW~discrepancy can be extended to an OT discrepancy on graphs called Fused Gromov Wasserstein (FGW) that take into account both the node features of the graphs as well as their structure matrices. \FGW~can be written as follows: 
\begin{equation}
\begin{split}
\operatorname{FGW}_{\alpha}(\mathbf{C}, \mathbf{D}, \mathbf{M},\mathbf{p}, \mathbf{q}) & = \min_{\mathbf{T} \in \mathcal{C}_{\mathbf{p},\mathbf{q}}} \sum_{i,j,k,l}((1-\alpha)\mathbf{M}+ \\
& \alpha\ell(\mathbf{C}_{i,k}, \mathbf{D}_{j,l}))\mathbf{T}_{i,j}\mathbf{T}_{k,l}
\end{split}
\end{equation}
where $\mathbf{M}$ is the distance matrix encoding differences between the nodes of the $2$ graphs and $\alpha$ is the convex combination between the distance matrices.

\subsubsection{Estimating the Action of Hadamard Square of Matrices on Vectors}
To compute (Fused) Gromov-Wasserstein discrepancies, one needs to compute $\mathbf{C}^{\odot 2} \mathbf{p}$, where $\mathbf{C}^{\odot 2}$ is the element-wise square (Hadamard square) of a cost matrix $\mathbf{C}$ and a vector $\mathbf{p}$. This is given by the following formula :
\begin{equation}
\mathbf{C}^{\odot 2}\mathbf{p} =\operatorname {diag} \left(\mathbf{CD}_{\mathbf {p} }\mathbf{C}^{\top}\right)
\end{equation}
where diag($\mathbf{M}$) is the diagonal of $\mathbf{M}$ and $\mathbf{D_p}$ is the matrix formed by $\mathbf{p}$ as its diagonal. 

However, for all our fast variants, we never materialize the matrix $\mathbf{C}$ explicitly. Thus to estimate the above action, we can make $2$ calls to our Fast  Multiplication method (FM)  via the following :
\begin{equation}~\label{eqn:sq}
    \mathbf{C}^{\odot 2}\mathbf{p} \sim \operatorname{diag}(FM_{\mathbf{C}}(FM_{\mathbf{C}}(\mathbf{D_p})^{\top}))
\end{equation}
Here FM can either be SeparationFactorization (SF) or the RFDiffusion algorithm (RFD), and for clarity, we use the subscript for the matrix $\mathbf{C}$ to specify that we are approximating the (right) action of the matrix $\mathbf{C}$.

\subsubsection{Algorithm to Put It All Together}
To calculate the OT (for Gromov-Wasserstein and Fused Gromov Wasserstein), the loss matrix $L(\mathbf{C}, \mathbf{D}, \mathbf{T} )$ needs to be computed, which is one of the most expensive steps, as it involves a tensor-matrix
multiplication. Indeed if the source graph has $n$ nodes and the target graph has $m$ nodes, this operation has a time complexity of $O(n^2m^2)$. However, when the loss function $\ell$ can be written as $\ell(a, b) = f_1(a) + f_2(b) - h_1(a)h_2(b)$ for functions $(f_1, f_2, h_1, h_2)$, the loss matrix can be calculated as~\citep{peyre:hal-01322992}
\begin{equation}~\label{eqn:tensor_prod}
L(\mathbf{C}, \mathbf{D}, \mathbf{T} ) = f_1(\mathbf{C})\mathbf{p} \ones_{m}^{\top} +
\ones_{n}\mathbf{q}^{\top}f_2(\mathbf{D}) - h_1(\mathbf{C})\mathbf{T} h_2(\mathbf{D})^{\top}
\end{equation}
where the functions $(f_1, f_2, h_1, h_2)$ are applied elementwise. In this case, the time complexity reduces to $O(n^2m + m^2n)$. Moreover if $\ell$ is the Euclidean loss function, then $f_1(x) = f_2(x) = x^2$ and $h_1(x) = x, h_2(x) = 2x$. 

Our Fast Multiplication methods (FM) can be used to efficiently estimate the above tensor product given by equation~\ref{eqn:tensor_prod} via the algorithm~\ref{alg:tensor_prod} thus leading to computation gains in computing \GW~(resp. \FGW~) discrepancies. In the above algorithm, the implicit representation of a matrix $\mathbf{M}$ can be given as an array of $3$-D coordinates and hyperparameters that are specific to the chosen FM algorithm. 

\begin{algorithm}[tb]
   \caption{Fast Computation of Tensor Products}
   \label{alg:tensor_prod}
\begin{algorithmic}
   \STATE {\bfseries Input:} $\mathbf{T}$ and $I_\mathbf{C}, I_\mathbf{D}$ \COMMENT{$I_\mathbf{M} :=$ implicit representation of the cost matrix $\mathbf{M}$}
   \STATE {\bfseries Output:} $L(I_\mathbf{C},I_\mathbf{D},\mathbf{T})$
   \STATE 1. Estimate $\mathbf{v}_1 := f_1(\mathbf{C})\mathbf{p}$ by Equation~\ref{eqn:sq}
   \STATE \hskip2.5em Compute $\mathbf{w}_1 = \mathbf{v}_1 \ones^{\top}$
   \STATE 2. Estimate $\mathbf{v}_2 := f_2(\mathbf{D})\mathbf{q}$ by Equation~\ref{eqn:sq} \COMMENT{Using the fact that $\mathbf{D}$ is symmetric}
   \STATE \hskip2.5em Compute $\mathbf{w}_2 = \ones \mathbf{v}_2^{\top}$
   \STATE 3. Estimate $h_1(\mathbf{C})\mathbf{T} h_2(\mathbf{D})^{\top}$ by \\ 
   \hskip2.5em $\mathbf{w}_3 := (FM_{\mathbf{D}}(FM_{\mathbf{C}}(\mathbf{T})^{\top}))^{\top}$ \\
   \STATE {\bfseries return} $\mathbf{w}_1 + \mathbf{w}_2 - 2\mathbf{w}_3$
\end{algorithmic}
\end{algorithm}

Our contributions go further than providing fast accurate computation of the tensor products but also a fast computation of the line search algorithm (Algorithm 2 as presented in~\citep{pmlr-v97-titouan19a}). The line search algorithm provides an optimal step size for the \FGW~iterations. 

We now provide a brief description of how our novel FM methods can be injected into the line search algorithm for the conjugate gradient. The line search algorithm at a \FGW~iteration takes in the structure matrices of the source and target graphs (which in our case will be implicit representations of such matrices), transport cost $\mathbf{G}$, $d\mathbf{G}$ which is the difference between the optimal map found by linearization in the \FGW~algorithm and $\mathbf{G}$, and $\mathbf{M}$, a matrix measuring the differences between nodes features of source and target graphs. Define $c_{\mathbf{C},\mathbf{D}} := f_1(\mathbf{C})\mathbf{p} \ones_{m}^{\top} +
\ones_{n}\mathbf{q}^{\top}f_2(\mathbf{D})$. Finally, the algorithm needs a cost function that combines the transportation cost coming from the node features and the graph structures which is applied to $\mathbf{G}$. This cost function crucially relies on the tensor product computation (Equation~\ref{eqn:tensor_prod}) and our Algorithm~\ref{alg:tensor_prod} provides a fast efficient computation of this cost function as well. 

\begin{algorithm}[tb]
   \caption{Fast Computation of Line-search for CG }
   \label{alg:line_search}
\begin{algorithmic}[1]
\STATE {\bfseries Input:} $I_{\mathbf{C}}, I_{\mathbf{D}}, \alpha, \mathbf{G}, 
d\mathbf{G}, \mathbf{M}, $
\STATE {\bfseries Output:} Optimal Step Size $\tau$
   \STATE  Estimate $c_{\mathbf{C},\mathbf{D}}$ by Step $1$ and $2$ of algorithm~\ref{alg:tensor_prod}.
   \STATE Estimate $a_1 := \mathbf{C}d\mathbf{G}\mathbf{D}$ by 
 $FM_{\mathbf{D}}(FM_{\mathbf{C}}(d\mathbf{G})^{\top})^{\top}$ \COMMENT{since $\mathbf{D}$ is symmetric}.
   \STATE Compute $a := -2\alpha\langle a_1, d\mathbf{G} \rangle $.
   \STATE Estimate $b_1 := \mathbf{C}\mathbf{G}\mathbf{D}$ by 
   $FM_{\mathbf{D}}(FM_{\mathbf{C}}(\mathbf{G})^{\top})^{\top}$
   \STATE Compute $b := \langle (1-\alpha)\mathbf{M} + \alpha c_{\mathbf{C}, \mathbf{D}}, d\mathbf{G}\rangle - 2\alpha(\langle a_1, \mathbf{G} \rangle + \langle b_1, d\mathbf{G}\rangle)$
   \STATE Compute $c := \text{cost}(\mathbf{G})$
\IF {$a > 0$} 
  \STATE  $\tau \gets \text{min}(1, \text{max}(0, \frac{-b}{2a}))$
\ELSE
    \IF {$a + b < 0$}
        \STATE $\tau \gets 1$  
    \ELSE 
    \STATE $\tau \gets 0$
    \ENDIF
\ENDIF 
\end{algorithmic}
\end{algorithm}
Note that employing a low-rank decomposition of the cost matrices to speed up the computation of \GW~ has also been studied in~\citep{lr_gw}. However, our work differs from their work in certain key aspects. The choice of our kernel matrices and the method of factorization of the cost matrix differs from the above work. Moreover, we do not design our methods with \GW~computations in mind but a flexible mechanism that can be injected into various \GW~computations including entropic-\GW~(similar to Algorithm 2 proposed in~\citep{lr_gw}).

\subsubsection{Gromov Wasserstein Barycenters}~\label{sec:gw_bary}
Recall, that given graphs $G_1, \cdots , G_N$, where $G_i := \{\mathbf{C}_i, \mathbf{p}_i\}$ comes equipped with a cost matrix $\mathbf{C}_i$ between its nodes and a probability simplex defined on its nodes, the Wasserstein barycenter can be defined as the minimizer of the functional
\begin{equation}
F[\nu] = \sum_{i=1}^{n}w_i \operatorname{GW}((\bar{\mathbf{C}}, \mathbf{C}_i,  \bar{\mathbf{p}}, \bar{\mathbf{p}_i}) 
\end{equation}
where $w_i$ are some fixed positive weights and $\sum w_i = 1$, $\bar{G} := \{\bar{\mathbf{C}}, \bar{\mathbf{p}} \}$ is the predefined barycenter graph with a fixed number of nodes. One can similarly define a Fused \GW~barycenter as well. 

As an application of our methods, we interpolate between a bunny ($1887$ vertices) and a torus ($1949$ vertices). We center the meshes around $(0,0,0)$ and scale the coordinates such that $|x|,|y|,|z| \leq 1$. We then run a fast Sinkhorn barycenter algorithm~\citep{janati:hal-03063875} to get a configuration of intermediate shapes in $3$D space. A sampling algorithm (Voxel Grid Filter) is then used to reduce the density of the generated point clouds to $1445$, $1450$, and $1425$ points respectively. We then try to solve for the edges of these intermediate point clouds.
\begin{figure}[h!]
\centering
\includegraphics[scale=.4]{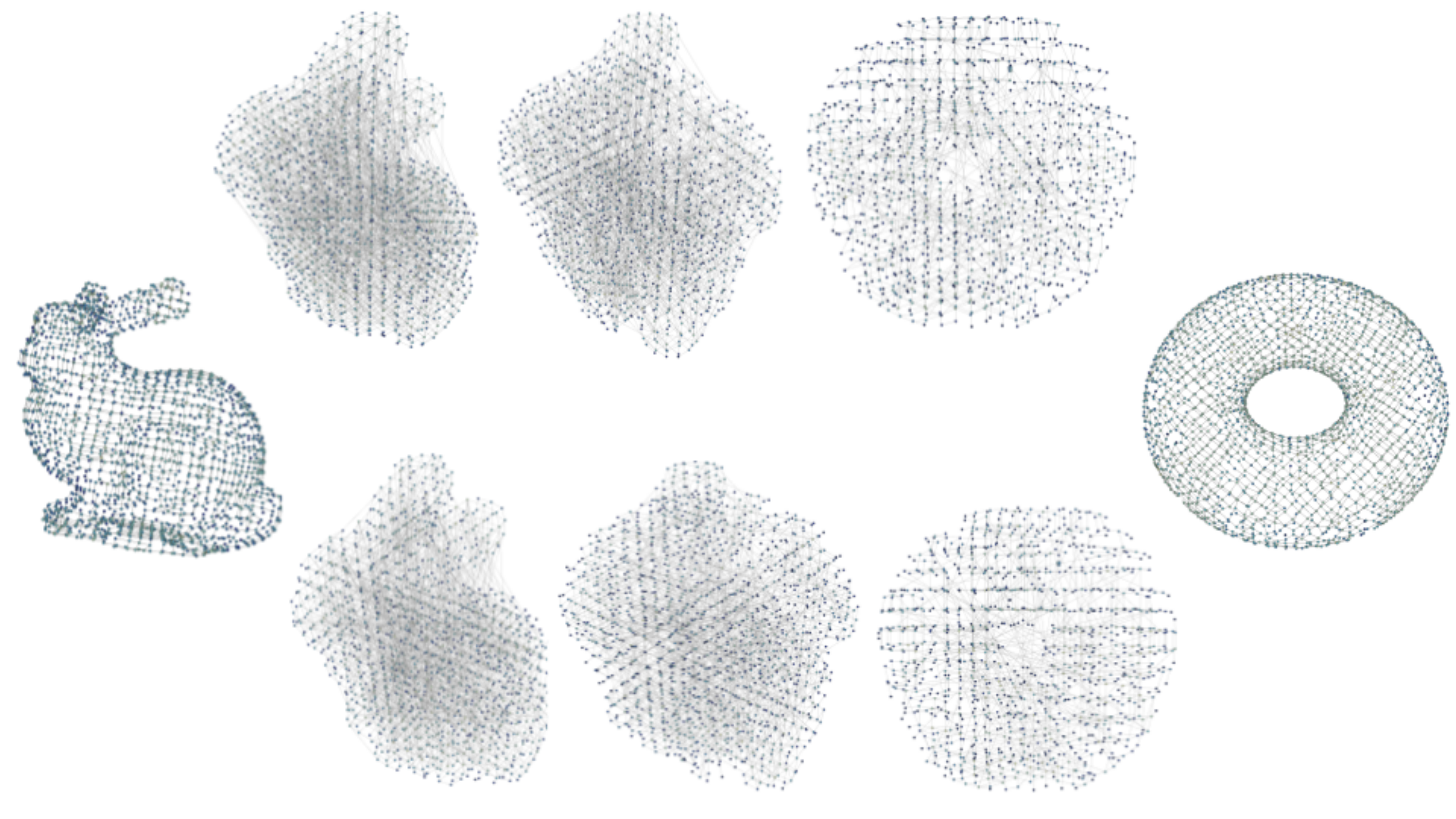}
\caption{Interpolation between a bunny and a torus. Barycenters are computed using, Top row: \GW-cg, Bottom row : \GW-cg-RFD }
\label{fig:gw_bary}
\end{figure}
Our method as well as the baseline \GW-cg algorithm produce decent meshes and tries to preserve the consistency of the manifold mesh throughout the interpolation (Figure~\ref{fig:gw_bary}). For the barycenter experiment, we use $m=16$ random features, $\lambda=-.15$, and $\epsilon=.13$.

All experiments on \GW~and its variants are conducted on a Google Colab.

\section{Ablation Studies}~\label{sec:ablates}
In this section, we present detailed ablation studies for our experiments. 
\subsection{Ablation Studies for Vertex Normal Prediction Experiments}~\label{sec:ablates_vertex_normal_prediction}
\vspace{-3mm}

\textbf{RFDiffusion.} There are three hyper-parameters in our RFDiffusion algorithm, which are the number of random features $m$, epsilon $\epsilon$, and lambda $\lambda$. The number of random features determines the accuracy of our approximation of the weighted adjacency matrix $\mathbf{W}_\mathrm{G}$. The epsilon hyper-parameter controls the sparsity level of the $\epsilon$-NN graph. The lambda hyper-parameter controls the "steepness" of our kernel. We can make the following conclusions based on the observations from Fig. \ref{fig:ablation_interpolation_RFD}. Firstly, increasing the number of random features usually gives us a better estimation of the weighted adjacency matrix $\mathbf{W}_\mathrm{G}$, which leads to higher cosine similarity. Secondly, a densely connected graph (large epsilon) coupled with a steeper kernel function (lambda with large absolute value) leads to better performance. 
\begin{figure}[h!]
\centering
\includegraphics[width=.99\textwidth]{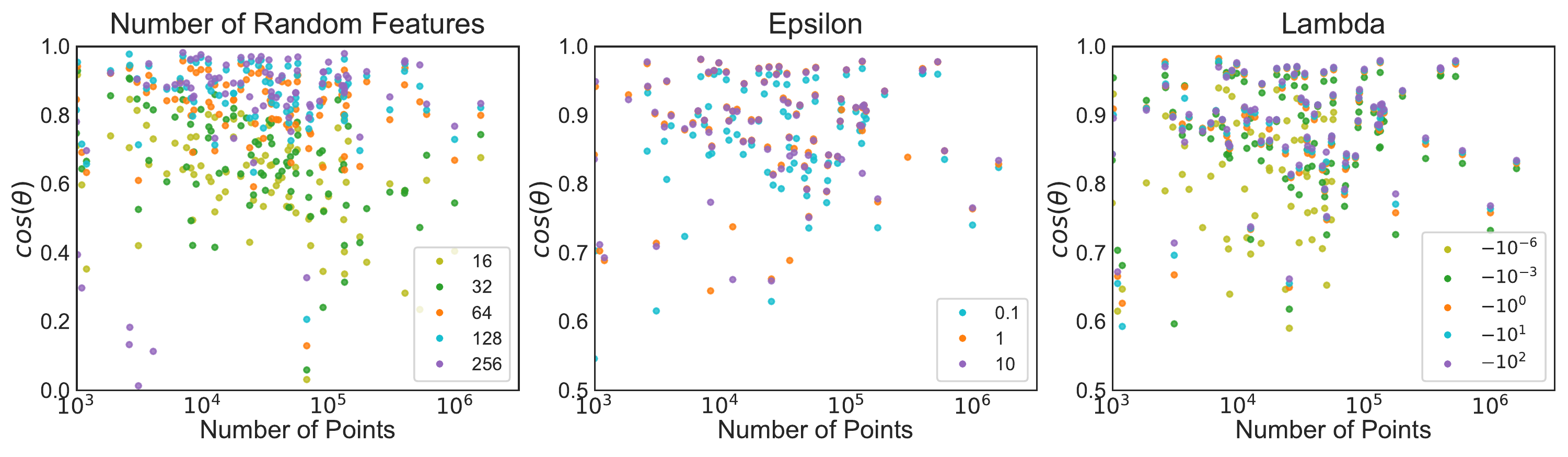}
\vspace{-4mm}
\caption{Ablations study for RFDs on the vertex normal prediction task.}
\label{fig:ablation_interpolation_RFD}
\vspace{-5mm}
\end{figure}

\textbf{SF.} There are two hyper-parameters in our $\mathrm{SF}$ algorithm, which are: $\mathrm{unit}$-$\mathrm{size}$ (determining the quantization mechanism: all the shortest path lengths are considered modulo $\mathrm{unit}$-$\mathrm{size}$) and $\mathrm{threshold}$ (specifying the maximum size of the graph, measured in the number of vertices, for which the GFI is conducted in a brute-force manner). Fig. \ref{fig:ablation_interpolation_SF_unitsize} shows pre-processing time, interpolation time, and cosine similarity under different values of the unit-size hyper-parameter. The results are reported with the threshold set as half of the number of vertices in the mesh. We can observe from the plots that a small value for unit-size provides a better estimation of the shortest-path distance without incurring significant changes in pre-processing and interpolation time. Fig. \ref{fig:ablation_interpolation_SF_threshold} shows the ablation of different thresholds while keeping the unit-size hyper-parameter the same (0.01). There is a trade-off between accuracy (measured by cosine similarity) and interpolation time. In the main body of our paper, we set the unit-size to 0.01 and the threshold to 0.5.
\begin{figure}[h!]
\centering
\includegraphics[width=.99\columnwidth]{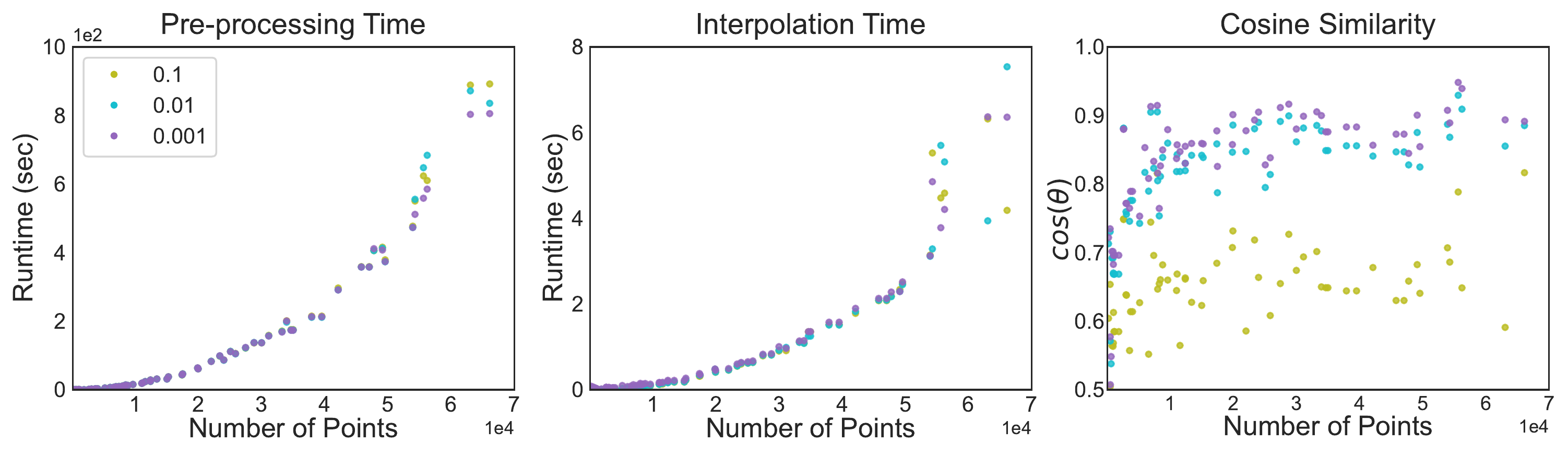}
\vspace{-4mm}
\caption{Ablation study for the \textbf{unit-size} hyper-parameter in $\mathrm{SF}$ algorithm for vertex normal prediction task.}
\label{fig:ablation_interpolation_SF_unitsize}
\end{figure}

\begin{figure}[h!]
\centering
\includegraphics[width=.99\columnwidth]{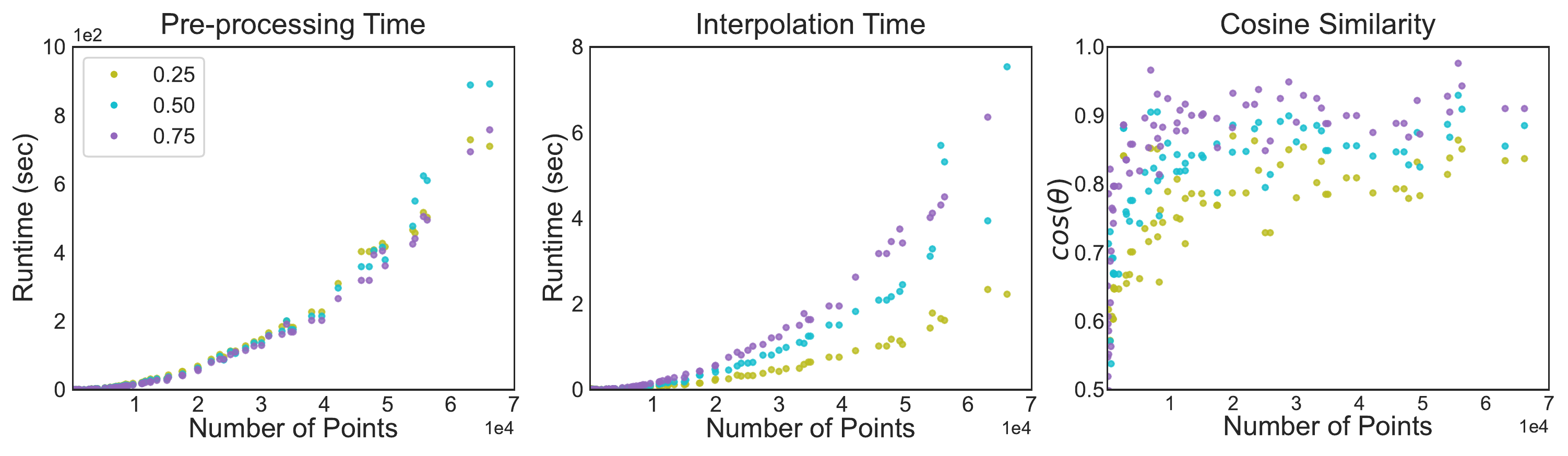}
\vspace{-4mm}
\caption{Ablation study for the \textbf{threshold} hyper-parameter in $\mathrm{SF}$ algorithm for vertex normal prediction task.}
\label{fig:ablation_interpolation_SF_threshold}
\end{figure}

\subsection{Ablation Studies for Gromov Wasserstein experiments}~\label{sec:gw_ablates}
The $\epsilon$ parameter in our RFDiffusion effectively controls the sparsity of our source and the target graphs. We find that our runtimes for \GW~distance via the conditional gradient algorithm remain mostly stable while that of the baseline algorithm grows with the density of the graph. However, runtime for \GW-distance computed via the proximal point algorithm is fairly stable~\citep{pmlr-v97-xu19b}. Surprisingly, the runtime for the \FGW~distance is also stable. We hypothesize that it is because even though the source or target graphs are sparse, we need to materialize a dense cross-feature distance matrix between the node features of the source and target nodes. In these cases, the runtimes for our RFDiffusion-integrated \GW~(with conditional gradient algorithm) and \FGW~are also stable and inconsistently lower than the baseline methods. The middle figure~\ref{fig:epsilon_ablates} shows the relative error as the function of $\epsilon$. As $\epsilon$ increases, for a fixed value of $\lambda$, the action of the matrix that we are trying to estimate will have a larger norm, and thus the relative error grows in accordance to Lemma~\ref{lemma:rfd}. However, for meshes rescaled in a unit box, the $\epsilon$ tends to be smaller in practice. 

We also see similar behavior with $\lambda$, i.e., smaller values of $|\lambda|$ tend to produce better results. This phenomenon is predicted by Lemma~\ref{lemma:rfd}. However, if $\lambda$ gets too close to $0$, the structure matrices approach an identity matrix, leading to information loss. This causes instabilities in the convergence of the algorithm.

All the experiments are run on random $3$D distributions with $3000$ points, and the results are averaged over $10$ runs.
\begin{figure}[h!]
\centering
\includegraphics[width=.3\columnwidth]{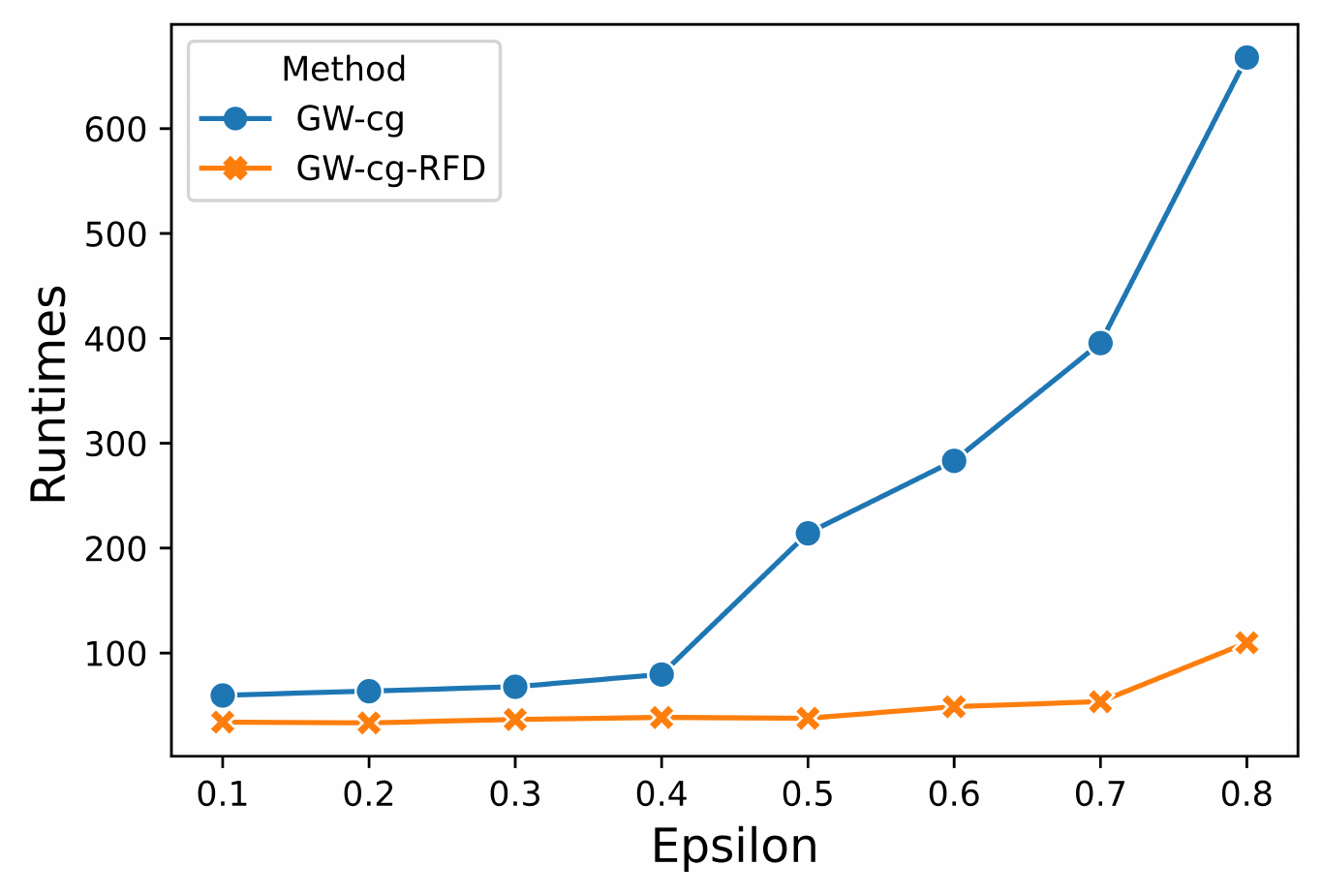}
\includegraphics[width=.3\columnwidth]{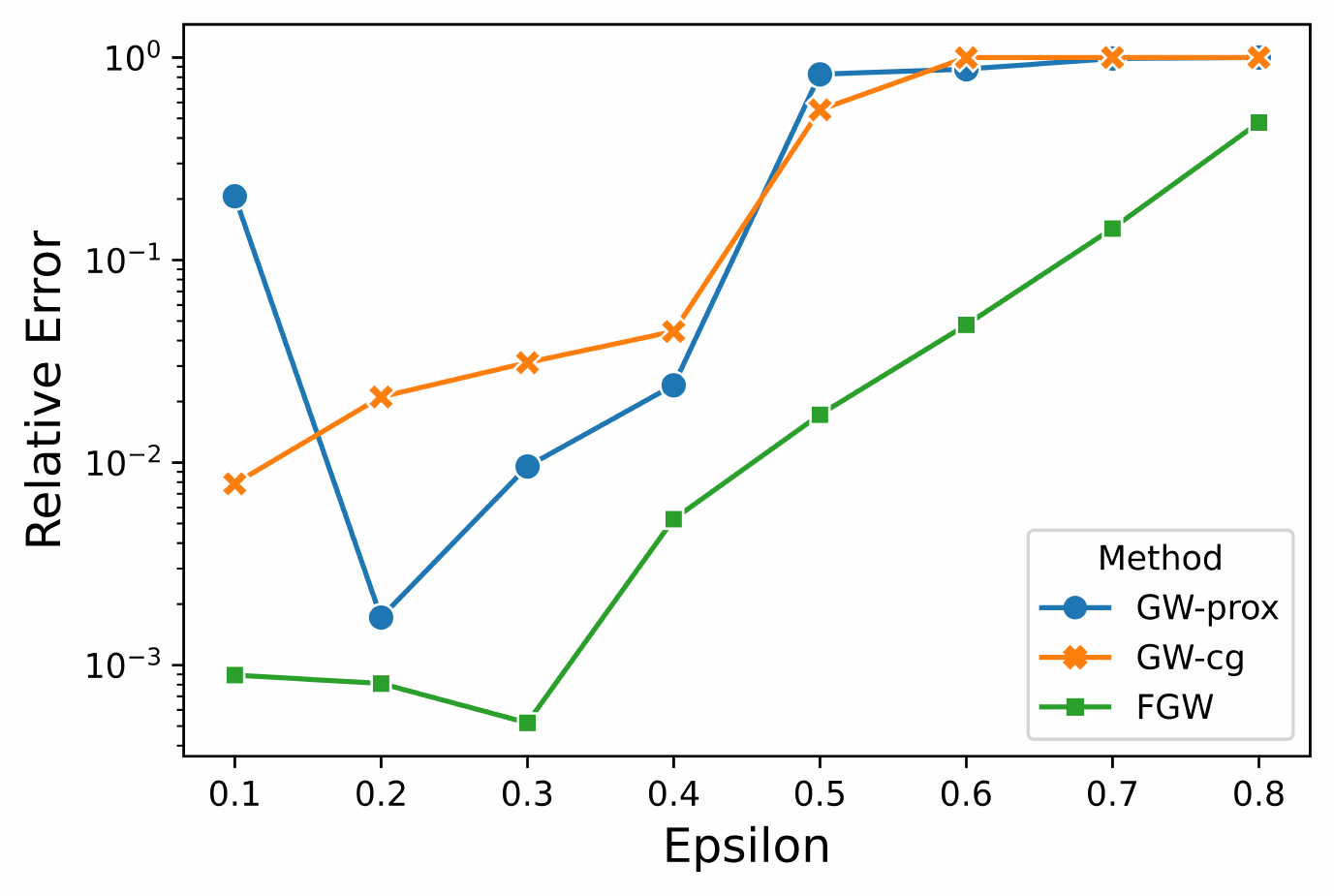}
\includegraphics[width=.3\columnwidth]{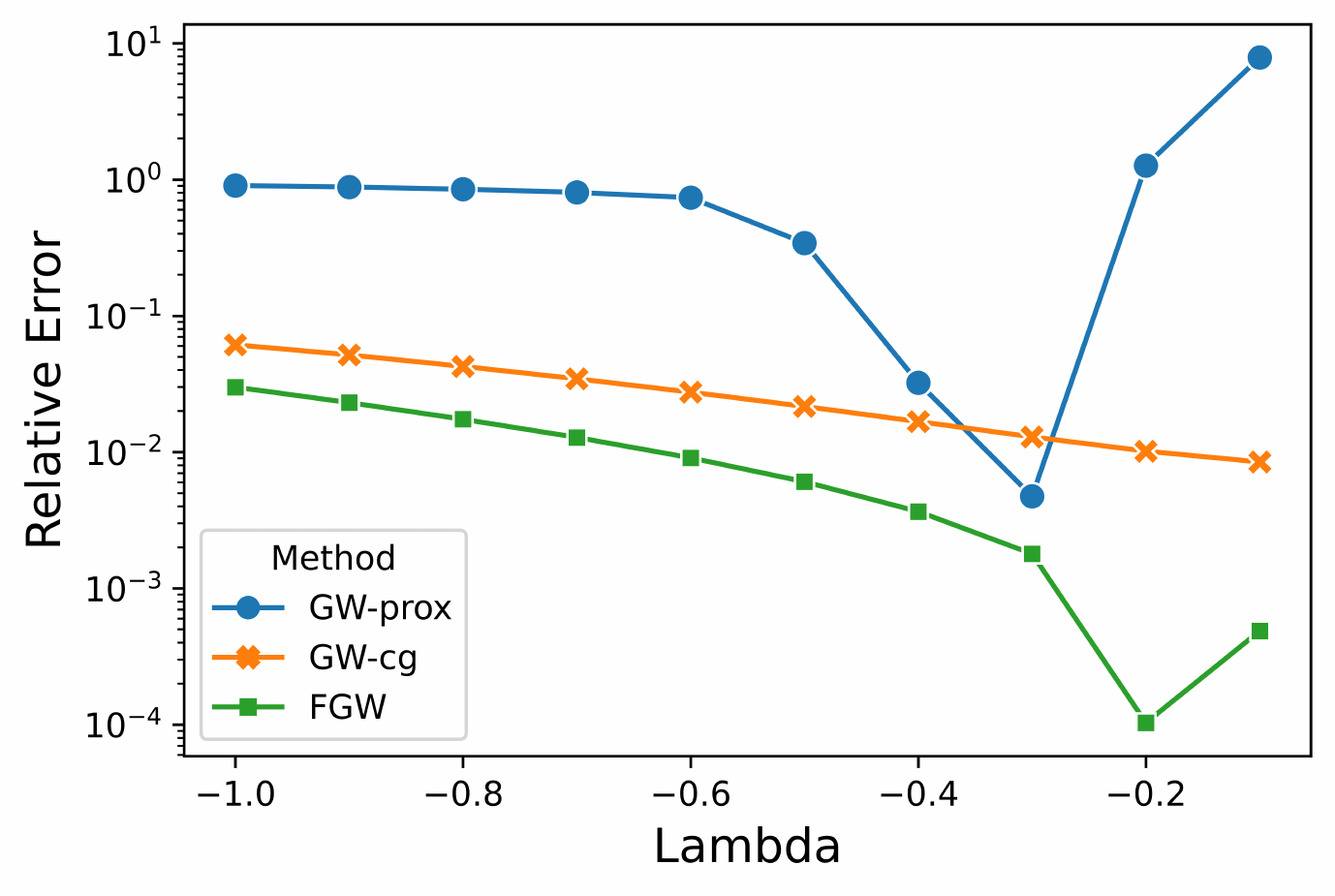}
\caption{Ablation study over the $\lambda$ and $\epsilon$ parameters for the \GW~variants. \textbf{Left:} The runtime for the baseline \GW-cg method increases as the input graph gets denser while our runtimes remain mostly constant. \textbf{Middle and right:} Plots show relative error as a function of $\epsilon$ and $\lambda$ respectively. The relative error increases if the graphs get ``dense" or the kernel becomes too ``steep."}
\label{fig:epsilon_ablates}
\end{figure}

\subsection{Ablation Studies on Wasserstein Barycenter Experiments}
In Table~\ref{table:wass_bary_diffusion_ablation}, we provide ablation results for the $\lambda$ hyper-parameter in $\mathrm{RFD}$ algorithm for the Wasserstein barycenter task. Experiments are conducted on the mesh \texttt{duck}. We show that the MSE increases with $\lambda$, which is in line with the observation in Section~\ref{sec:gw_ablates}. The runtime is nearly unchanged for different values of $\lambda$. 
We normalize the coordinates of the vertices and choose the epsilon parameter to be 0.01, making the computation meaningful. Larger epsilon values will cause the graph to be too dense, and smaller epsilon values will create an epsilon graph with almost no edges.

In Table~\ref{table:wass_bary_separation_ablation}, we provide ablation results for the \textbf{unit-size} hyperparameter in the $\mathrm{SF}$ algorithm. We show that the MSE slowly increases with unit-size, and the runtime is nearly unchanged for different values of \textbf{unit-size}.

\begin{table}[ht]
\begin{minipage}[b]{0.56\linewidth}
\centering
    \caption{Ablation study for the \textbf{unit-size} parameter in \SF\ for Wasserstein barycenter task.}
    \begin{tabular}{@{}l c c c@{}}
        \toprule
        \textbf{unit-size} & \textbf{MSE} & \textbf{Total time (secs)} \\
        \midrule
        $0.1$ & $2.1 \times 10^{-3}$ & $19.4$\\
        $0.5$ & $2.1 \times 10^{-3}$ & $19.1$\\
        $1.0$ & $2.1 \times 10^{-3}$ & $18.9$\\
        $5.0$ & $2.7 \times 10^{-3}$ & $18.8$\\
        $10.0$ & $3.1 \times 10^{-3}$ & $19.1$\\
        \bottomrule
    \end{tabular}
    \label{table:wass_bary_separation_ablation}
\end{minipage}\hfill
\begin{minipage}[b]{0.4\linewidth}
\centering
    \caption{Ablation study for the $\lambda$ parameter in \RFD\  for Wasserstein barycenter task.}
    \begin{tabular}{@{}c c c c@{}}
        \toprule
        \textbf{$\lambda$} & \textbf{MSE} & \textbf{Total time (secs)} \\
        \midrule
        $0.1$ & $2 \times 10^{-4}$ & $1.1$\\
        $0.3$ & $1.1 \times 10^{-3}$ & $1.1$\\
        $0.5$ & $2.1 \times 10^{-3}$ & $1.0$\\
        $0.7$ & $2.7 \times 10^{-3}$ & $1.1$\\
        $0.9$ & $3.3 \times 10^{-3}$ & $1.1$\\
        \bottomrule
    \end{tabular}
    \label{table:wass_bary_diffusion_ablation}
\end{minipage}
\end{table}
\section{Graph Classification Experiments using the \RFD\ Kernel}~\label{sec:graph_classification}
We extract our \RFD\ kernel and use it for various graph classification tasks. More specifically, we compute the top $k$ eigenvalues for the approximated kenel matrix and pass it to a random forest classifier for classification. Note that, as described in~\citep{nakatsukasa2019lowrank}, low-rank decomposition of the kernel matrix (provided directly by the $\mathrm{RFDiffusion}$ method via the random feature map mechanism) can be used to compute efficiently eigenvectors and the corresponding eigenvalues. 

However, most of the benchmark datasets for graph classification are molecular datasets~\citep{Morris2020}. Our methods are originally developed for meshes and point clouds where we excel (see section~\ref{sec:pc_expt}) hence we do not consider molecular graphs in the main paper. The node features of these molecular graphs are extremely coarse and thus the epsilon-neighborhood graph constructed using these features performs poorly in downstream graph classification tasks. We apply our $\mathrm{RFDiffusion}$ kernel on the sets of points, considering the node features as vectors in a $d$-dimensional space. The \RFD\ kernel produces a smoothened version of the epsilon-neighborhood graph, giving good results even when the baseline applying explicitly the epsilon-neighborhood graph does not. This is the case since random features replace the combinatorial object (a graph with edges and no-edges) with its “fuzzy” version, where all the nodes are connected by edges (that are not explicitly reconstructed though) but the weights corresponding to non-edges in the original graph are close to zero with high probability.

We compare our algorithm with four baselines : Vertex Histogram (VH), Random Walk (RW), Weisfeiler-Lehman shortest path kernel (WL-SP)~\citep{graph_kernels} and Feature based method (FB)~\citep{de2018simple}. Our method compares favorably with these methods and is also competitive with various kernel methods reported in~\citep{de2018simple, balcilar2020bridging, graph_kernels, graph_classification_ieee_gcn}. 
The results along with statistics about the datasets are summarized in Table~\ref{tab:mol_graph_classification}.
\begin{table}
\centering
\caption{Graph Classification using \RFD\ Kernel}
\begin{tabular}{@{}c c c c c c c c  c@{}}
        \toprule
    Dataset & \# Graphs & Avg. \# Nodes & Avg. \# Edges & VH & RW & WL-SP & FB & \RFD\ (ours) \\
        \midrule
        MUTAG &	$188$ &	$17.93$	& $19.79$	& $69.1$ & $81.4$ & $81.4$ & $\mathbf{84.7}$ & $71.0$ \\			
        ENZYMES &	$600$	& $32.63$ &	$62.14$ & 	$20.0$ & $16.7$ & $27.3$ & $\mathbf{29.0}$ & $27.0$ \\
        PROTEINS &	$1113$	& $39.06$ &	$72.82$ & 	$71.1$ & $69.5$ & $72.1$ & $70.0$ & $\mathbf{75.0}$ \\
        NCI1	& $4110$ &	$29.87$ &	$32.3$	 & $55.7$ & TIMEOUT & $60.8$ & $\mathbf{62.9}$ & $61.0$ \\		DD	& $1178$ &  $284.32$ &  $715.66$ & $74.8$ & OOM & $\mathbf{76.0}$ & - & $73.0$		\\				PTC-MR & $344$	& $14.29$	& $14.69$	& $57.1$ & $54.4$ & $54.5$ & $55.6$ & $\mathbf{61.0}$ \\																				
        \bottomrule
    \end{tabular}
    \label{tab:mol_graph_classification}
\end{table}


\end{document}